\documentclass{article}

\PassOptionsToPackage{numbers, compress}{natbib}

\usepackage[final]{neurips_2023}

\usepackage[utf8]{inputenc} %
\usepackage[T1]{fontenc}    %
\usepackage{hyperref}       %
\usepackage{url}            %
\usepackage{booktabs}       %
\usepackage{amsfonts}       %
\usepackage{nicefrac}       %
\usepackage{microtype}      %
\usepackage[dvipsnames]{xcolor}         %
\usepackage{siunitx}

\usepackage{dsfont}
\usepackage{notation}
\usepackage{researchpack}
\usepackage[capitalise,nameinlink]{cleveref}
\usepackage{xspace}
\usepackage{subcaption}
\usepackage{natbib}
\usepackage{multirow}
\usepackage{multicol}
\usepackage{tikz}
\usepackage{pgfplots}
\usetikzlibrary{shapes,calc,positioning,arrows.meta}
\pgfplotsset{compat=1.18}

\crefname{thm}{Thm.}{Thm.}
\crefname{cor}{Cor.}{Cor.}
\crefname{lem}{Lem.}{Lem.}
\crefname{prop}{Prop.}{Prop.}
\crefname{defn}{Def.}{Def.}

\makeatletter
\theoremstyle{definition}
\newtheorem{rethm@}{Theorem}
\newenvironment{rethm}[2][]{%
    \def\therethm@{\ref{#2}}
    \rethm@[#1]
}
{\endrethm@}
\newtheorem{reprop@}{Proposition}
\newenvironment{reprop}[2][]{%
    \def\thereprop@{\ref{#2}}
    \reprop@[#1]
}
{\endreprop@}
\makeatother

\theoremstyle{definition}
\newtheorem{alem}{Lemma}[section]
\newtheorem{aprop}{Proposition}[section]
\newtheorem{adefn}{Definition}[section]
\crefname{alem}{Lem.}{Lem.}
\crefname{aprop}{Prop.}{Prop.}
\crefname{adefn}{Def.}{Def.}

\definecolor{petroil}{HTML}{24A5AF}
\definecolor{sky}{HTML}{2C91D4}
\definecolor{gold}{HTML}{FF8200}
\definecolor{tomato}{HTML}{E53935}
\definecolor{peas}{HTML}{7CB342}
\definecolor{lightgold}{HTML}{FFE000}
\definecolor{darktomato}{HTML}{B71C1C}
\definecolor{violet}{HTML}{776FB2}
\definecolor{darkviolet}{HTML}{342F58}
\definecolor{bgrey0} {HTML} {78909C}
\definecolor{bgrey1} {HTML} {607D8B}
\definecolor{bgrey2} {HTML} {546E7A}
\definecolor{bgrey3} {HTML} {455A64}
\definecolor{bgrey4} {HTML} {37474F}
\definecolor{bgrey5} {HTML} {263238}
\definecolor{petroil2} {RGB} {36, 165, 175}
\definecolor{petroil4} {RGB} {30, 132, 149}
\definecolor{petroil6} {RGB} {23, 101, 115}

\newcommand{\CP}{\textsc{CP}\xspace}
\newcommand{\DistMult}{\textsc{DistMult}\xspace}
\newcommand{\RESCAL}{\textsc{RESCAL}\xspace}
\newcommand{\ComplEx}{\textsc{ComplEx}\xspace}
\newcommand{\TuckER}{\textsc{TuckER}\xspace}
\newcommand{\OurModelPlain}{GeKC}
\newcommand{\OurModel}{\OurModelPlain\xspace}
\newcommand{\OurModels}{\OurModelPlain s\xspace}
\newcommand{\NNegCP}{\textsc{CP\textsuperscript{+}}\xspace}

\newcommand{\NNegRESCAL}{\textsc{RESCAL\textsuperscript{+}}\xspace}
\newcommand{\NNegComplEx}{\textsc{ComplEx\textsuperscript{+}}\xspace}
\newcommand{\NNegTuckER}{\textsc{TuckER\textsuperscript{+}}\xspace}
\newcommand{\SquaredCP}{\textsc{CP\textsuperscript{2}}\xspace}

\newcommand{\SquaredRESCAL}{\textsc{RESCAL\textsuperscript{2}}\xspace}
\newcommand{\SquaredComplEx}{\textsc{ComplEx\textsuperscript{2}}\xspace}
\newcommand{\SquaredTuckER}{\textsc{TuckER\textsuperscript{2}}\xspace}
\newcommand{\ConstrSquaredComplEx}{\textsc{d-ComplEx\textsuperscript{2}}\xspace}

\renewcommand{\textsf}[1]{{\fontsize{9pt}{9pt}\sffamily #1}}

\newcommand{\tablesize}{\fontsize{8pt}{8pt}\selectfont}

\newcommand{\shortparagraph}[1]{\vspace{-5pt}\paragraph{#1}}
\newcommand{\shortsubsection}[1]{\subsection{#1}}

\newcommand{\shortsection}[1]{\section{#1}}

\title{How to Turn Your Knowledge Graph \\Embeddings 
into 
Generative Models}

\author{
  Lorenzo Loconte\\
  University of Edinburgh, UK\\
  \texttt{l.loconte@sms.ed.ac.uk}\\
  \And
  Nicola Di Mauro\\
  University of Bari, Italy\\
  \texttt{nicola.dimauro@uniba.it}\\
  \AND
  Robert Peharz\\
  TU Graz, Austria\\
  \texttt{robert.peharz@tugraz.at}\\
  \And
  Antonio Vergari\\
  University of Edinburgh, UK\\
  \texttt{avergari@ed.ac.uk}\\
}

\hypersetup{
colorlinks=true,linkcolor=teal,citecolor=gold
}

\begin{document}

\maketitle

\begin{abstract}
   Some of the most successful knowledge graph embedding (KGE) models for link prediction -- \CP, \RESCAL, \TuckER, \ComplEx\ -- can be interpreted as energy-based models.
   Under this perspective
   they are not amenable for
   exact maximum-likelihood estimation (MLE),
   sampling and struggle to integrate logical constraints.
   This work re-interprets
   the score functions of these KGEs as \textit{circuits} -- constrained computational graphs allowing efficient marginalisation.
   Then, we design two recipes to obtain efficient generative circuit models %
   by either restricting their activations to be non-negative or squaring their outputs. 
   Our interpretation comes
   with little or no loss of performance for link prediction,
   while the circuits framework unlocks exact learning by MLE,
   efficient sampling of new triples,
   and guarantee that logical constraints are satisfied by design.
   Furthermore,
   our models
   scale more gracefully than the original KGEs on
   graphs with millions of entities.
\end{abstract}

\shortsection{Introduction}
\label{sec:introduction}

Knowledge graphs (KGs) are a popular way to represent
structured domain information as directed graphs encoded as collections of triples (\textsf{subject}, \textsf{predicate}, \textsf{object}), where subjects and objects (entities) are nodes in the graph, and predicates their edge labels.
For example, the information that the drug ``\textsf{loxoprofen}'' interacts with the protein ``\textsf{COX2}'' is represented as a triple (\textsf{loxoprofen}, \textsf{interacts}, \textsf{COX2}) in the biomedical KG \textsf{ogbl-biokg} \citep{hu2020ogb}.
As real-world KGs are often incomplete, we are interested in performing reasoning tasks over them while dealing with missing information.
The simplest reasoning task is \textit{link prediction}~\citep{nickel2016kgreview}, i.e.,
querying for the entities that are connected in a KG by an edge labelled with a certain predicate.
For instance, we can retrieve all proteins that the drug ``loxoprofen'' interacts with by asking the query (\textsf{loxoprofen}, \textsf{interacts}, \textsf{?}).

Knowledge graph embedding (KGE) models are state-of-the-art (SOTA) models for link prediction \citep{lacroix2018cp-kbc,ruffinelli2020teach-dog-tricks,chen2021relational-auxiliary-objective} that map entities and predicates to sub-symbolic representations,
which are used to assign a real-valued degree of existence to triples in order to rank them.
For example, the SOTA KGE model \ComplEx~\citep{trouillon2016complex} assigns
(\textsf{loxoprofen}, \textsf{interacts}, \textsf{phosphoric-acid}) and (\textsf{loxoprofen}, \textsf{interacts}, \textsf{COX2}) scores 2.3 and 1.3, 
hence ranking the first higher than the second in our link prediction example.

This simple example, however, also highlights some opportunities that are missed by KGE models.
First, triple scores cannot be directly compared across different queries and across different KGE models, as they can be seen as \textit{negative energies} and not \textit{normalised probabilities over triples}~\citep{lecun2006tutorial-ebm,bordes2013transe,bordes2011se,glorot2013semantic,minervini2016efficient}.
To establish a sound probabilistic interpretation and therefore have \emph{probabilities instead of scores} that can be easily interpreted and compared \citep{zhu2023closer-look-calibration}, we would need to compute the normalisation constant (or partition function), which is impractical for real-world KGs due to their considerable size (see \cref{sec:background}).
Therefore,
learning KGE models by maximum-likelihood estimation (MLE) would be computationally infeasible, which is  the canonical probabilistic way to learn \textit{a generative model over triples}.
A generative model would enable us to sample new triples efficiently, 
e.g., to generate a surrogate KG whose statistics are consistent with the original one or to do data augmentation \citep{chauhan2021probabilistic}. %
Furthermore, traditional KGE models do not provide a principled way to \textit{guarantee the satisfaction of} 
\textit{hard constraints}, which are crucial to ensure trustworthy predictions in safety-critical contexts such as biomedical applications. %
The result is that predictions of these models can blatantly violate simple constraints such as KG schema definitions. 
For instance, the triple that the SOTA \ComplEx ranks higher in our example above violates the semantics of ``\textsf{interacts}'', i.e., such predicate can only hold between drugs (e.g., \textsf{loxoprofen}) and proteins (e.g., \textsf{COX2}) but \textsf{phosphoric-acid} is not a protein.

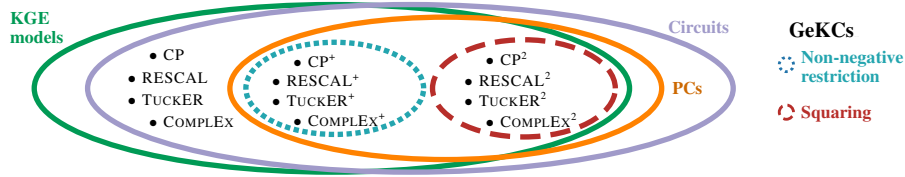
\begin{figure}[!t]
    \centering
    \resizebox{0.9\linewidth}{!}{
    \begin{tikzpicture}[line width=3pt, align=left, anchor=center]
        \node[ellipse, draw, ForestGreen, text width=230pt, text height=60, anchor=north west, xshift=-15pt] (kge-venn) {};
        \node[ellipse, draw, violet!70!white, text width=250pt, text height=60,  anchor=north west, right=30pt of kge-venn.west] (circuits-venn) {};
        \node[ellipse, draw, gold, text width=165pt, text height=50pt, anchor=west, right=80pt of circuits-venn.west] (pcs-venn) {};

        \node[inner sep=2pt, ForestGreen, text width=48, anchor=east, left=10pt of kge-venn.north west] (kge-label)
        {\textbf{KGE}\\\textbf{models}};
        \node[inner sep=2pt, violet!70!white, text width=48, anchor=west, right=12pt of circuits-venn.north east] (circuits-label)
        {\textbf{Circuits}};
        \node[inner sep=2pt, gold!80!black, text width=16pt, anchor=east, right=1pt of pcs-venn.east] (pcs-label)
        {\textbf{PCs}};

        \node[ellipse, draw, dotted, petroil2, text width=64pt, text height=30pt, anchor=west, right=9pt of pcs-venn.west] (pcs-plus-venn) {};
        \node[ellipse, draw, dash pattern=on 4mm off 2mm, tomato!80!black, text width=66pt, text height=32pt, anchor=east, left=8pt of kge-venn.east] (pcs-squared-venn) {};

        \node[inner sep=0pt, text width=60pt, align=left, anchor=north west, below=10pt of pcs-venn.north] (cp) at ($(circuits-venn.north west)!0.2!(pcs-venn.north west)$)
        {\small$\bullet$ \CP};
        \node[inner sep=0pt, text width=60pt, below left=8pt and 12pt of cp.west, anchor=north west] (rescal)
        {\small $\bullet$ \RESCAL};
        \node[inner sep=0pt, text width=60pt, below left=8pt and 0pt of rescal.west, anchor=north west] (tucker)
        {\small$\bullet$ \TuckER};
        \node[inner sep=0pt, text width=60pt, below right=8pt and 12pt of tucker.west, anchor=north west] (complex)
        {\small$\bullet$ \ComplEx};

        \node[inner sep=0pt, text width=60pt, align=left, anchor=north west, below=10pt of pcs-venn.north] (cp-plus) at ($(pcs-venn.north west)!0.22!(kge-venn.north east)$)
        {\small$\bullet$ \NNegCP};
        \node[inner sep=0pt, text width=60pt, below left=6pt and 12pt of cp-plus.west, anchor=north west] (rescal-plus)
        {\small$\bullet$ \NNegRESCAL};
        \node[inner sep=0pt, text width=60pt, below left=6pt and 0pt of rescal-plus.west, anchor=north west] (tucker-plus)
        {\small$\bullet$ \NNegTuckER};
        \node[inner sep=0pt, text width=60pt, below right=6pt and 12pt of tucker-plus.west, anchor=north west] (complex-plus)
        {\small$\bullet$ \NNegComplEx};

        \node[inner sep=0pt, text width=60pt, align=left, anchor=north west, below=12pt of pcs-venn.north] (cp-squared) at ($(pcs-venn.north west)!0.98!(kge-venn.north east)$)
        {\small$\bullet$ \SquaredCP};
        \node[inner sep=0pt, text width=60pt, align=left, below left=6pt and 12pt of cp-squared.west, anchor=north west] (rescal-squared)
        {\small$\bullet$ \SquaredRESCAL};
        \node[inner sep=0pt, text width=60pt, align=left, below left=6pt and 0pt of rescal-squared.west, anchor=north west] (tucker-squared)
        {\small$\bullet$ \SquaredTuckER};
        \node[inner sep=0pt, text width=60pt, align=left, below right=6pt and 12pt of tucker-squared.west, anchor=north west] (complex-squared)
        {\small$\bullet$ \SquaredComplEx};

        \node[inner sep=0pt, ultra thick, circle, dotted, minimum size=8pt, draw=sky!80!black, align=center, above right=10pt and 25pt of circuits-venn.east] (plus-label-bullet) {};
        \node[inner sep=0pt, ultra thick, circle, dash pattern=on 2mm off 1mm, minimum size=8pt, draw=tomato!80!black, align=center, below right=10pt and 25pt of circuits-venn.east] (squared-label-bullet) {};

        \node[inner sep=0pt, text width=20pt, text height=9pt, align=left, above=15pt of plus-label-bullet, anchor=west] (gecks)
        {\large\textbf{\OurModels}{\scalebox{0.015}{feels so clean like a money machine}}};

        \node[petroil2, inner sep=0pt, text width=70pt, text height=9pt, align=left, right=2pt of plus-label-bullet, anchor=west] (plus-label)
        {\textbf{Non-negative restriction}};
        \node[tomato!80!black, inner sep=0pt, text width=80pt, text height=9pt, align=left, right=2pt of squared-label-bullet, anchor=west] (squared-label)
        {\textbf{Squaring}};
    \end{tikzpicture}
    }
    \caption{\textbf{Which KGE models can be used as efficient generative models of triples?} The score functions of popular KGE models such as \ComplEx, \CP, \RESCAL and \TuckER can be easily represented as circuits (lilac). However, to retrieve a valid probabilistic circuit (PC, in orange) that encodes a probability distribution over triples (\OurModels) we need to either \emph{restrict its activations to be non-negative} (in blue, see \cref{sec:non-negative-restriction}) or \emph{square it} (in red, see \cref{sec:non-monotonic-squaring}).
    }
    \label{fig:kge-pcs-venn}
\end{figure}

\shortparagraph{Contributions.}
We show that KGE models that have become a de facto standard, such as \ComplEx and alternatives based on multilinear score functions~\citep{nickel2011rescal,lacroix2018cp-kbc,tucker64extension}, can be represented as structured computational graphs, named \textit{circuits}~\citep{choi2020pc}.
Under this light, \textbf{i)} we propose a different interpretation of these computational graphs and their parameters, to retrieve efficient and yet expressive probabilistic models over triples in a KG, which we name \emph{generative KGE circuits} (\OurModels) (\cref{sec:casting-kge-models-to-pcs}).
We show that \textbf{ii)} not only \OurModels can be efficiently trained by exact MLE, but learning them with widely-used discriminative objectives \citep{joulin2017fastlm,lacroix2018cp-kbc,ruffinelli2020teach-dog-tricks,chen2021relational-auxiliary-objective} also scales far better over large KGs with millions of entities.
In addition, \textbf{iii)} we are able to sample new triples exactly and efficiently from \OurModels, and propose a novel metric to measure their quality (\cref{sec:empirical-evaluation-sampling}).
Furthermore, by leveraging recent theoretical advances in representing circuits \citep{vergari2021compositional}, \textbf{iv)} we guarantee that predictions at test time will never violate logical constraints such as domain schema definitions by design (\cref{sec:injection-logical-constraints}).
Finally, our experimental results show that these advantages come with no or minimal loss in terms of link prediction accuracy. %

\shortsection{From KGs and embedding models\ldots}
\label{sec:background}

\textbf{KGs and embedding models.}
A KG $\calG$ is a directed multi-graph where nodes are entities and edges are labelled with predicates, i.e., elements of two sets $\calE$ and $\calR$, respectively.
We define $\calG$ as a collection of triples $(s,r,o)\subseteq\calE\times\calR\times\calE$, where $s$, $r$, $o$ denote the subject, predicate and object, respectively.
A \emph{KG embedding} (KGE) model maps a triple $(s,r,o)$ to a real scalar via a \emph{score function} $\phi\colon\calE\times\calR\times\calE \to \bbR$.
A common recipe to construct differentiable score functions for many state-of-the-art KGE models \citep{ruffinelli2020teach-dog-tricks} is to (i) map entities and predicates to \textit{embeddings} of rank $d$, i.e., elements of normed vector spaces (e.g., $\bbR^d$), and (ii) combine the embeddings of subject, predicate and object via multilinear maps.
This is the case for KGE models such as \CP~\citep{lacroix2018cp-kbc}, \RESCAL~\citep{nickel2011rescal}, \TuckER~\citep{balazevic2019tucker}, and \ComplEx~\citep{trouillon2016complex} (see \cref{fig:kge-circuits}).
For instance, the score function of \ComplEx\citep{trouillon2016complex}
is defined as $\phi_\ComplEx(s,r,o) = \Re(\langle \ve_s,\vw_r,\overline{\ve_o} \rangle)$
where $\ve_s,\vw_r,\ve_o\in\bbC^d$ are the complex embeddings of subject, predicate and object, $\langle \cdot,\  \cdot,\  \cdot \rangle$ denotes a trilinear product,  $\overline{\ \cdot\ }$ denotes the complex conjugate operator and $\Re(\cdot)$ the real part of complex numbers.

\shortparagraph{Probabilistic loss-derived interpretation.}
KGE models have been traditionally interpreted as \emph{energy-based models} (EBMs) \citep{lecun2006tutorial-ebm,bordes2013transe,bordes2011se,glorot2013semantic,minervini2016efficient}: their score function is assumed to compute the negative energy of a triple.
This interpretation induces a distribution over possible KGs by associating a Bernoulli variable, whose parameter is determined by the score function, to every triple \citep{nickel2016kgreview}.
Learning EBMs under this perspective requires using contrastive objectives \citep{bordes2013transe,nickel2016kgreview,ruffinelli2020teach-dog-tricks},
but several recent works
observed that to achieve SOTA link prediction results one needs only   to predict subjects, objects~\citep{joulin2017fastlm,lacroix2018cp-kbc,ruffinelli2020teach-dog-tricks} or more recently also predicates of triples~\citep{chen2021relational-auxiliary-objective},
i.e., to treat KGEs as \textit{discriminative} classifiers.
Specifically, they are learned by minimising a categorical cross-entropy, e.g., by maximising $\log p(o\mid s,r) = \phi(s,r,o) - \log \sum_{o'\in\calE} \exp \phi(s,r,o')$ for object prediction.
From this perspective, we observe that we can recover an energy-based interpretation if we assume there exist
a joint probability distribution $p$ over three variables $S,R,O$, denoting respectively subjects, predicates and objects.
Written as a Boltzmann distribution, we have that
$p(S=s,R=r,O=o) = (\exp \phi(s,r,o)) / Z$, where $Z=\sum_{s'\in\calE} \sum_{r'\in\calR} \sum_{o'\in\calE} \exp \phi(s',r',o')$ denotes the partition function.
This interpretation is apparently in contrast with the traditional one over possible KGs \citep{nickel2016kgreview}. We reconcile it with ours in \cref{app:interpreting-density-triples}.
Under this view, we can reinterpret and generalise the recently introduced discriminative objectives~\citep{chen2021relational-auxiliary-objective} as a weighted \emph{pseudo-log-likelihood} (PLL)~\citep{varin2011composite}
\begin{equation}
    \label{eq:pll-objective}
    \calL_\text{PLL} := \sum\nolimits_{(s,r,o)\in\calG} \omega_s \log p(s\mid r,o) + \omega_o \log p(o\mid s,r) + \omega_r \log p(r\mid s,o)
\end{equation}
where $\omega_s,\omega_o,\omega_r\in\bbR_+$ can differently weigh each term,
which is a conditional log-probability that can be computed by summing out either $s$, $r$ or $o$,
e.g., to compute $\log p(o\mid s,r)$ above.
Optimisation is usually carried out by mini-batch gradient ascent \citep{ruffinelli2020teach-dog-tricks} and, given a batch of triples $B\subset \calG$, we
have
that exactly computing the PLL objective requires time $\calO(|\calE|\cdot |B|\cdot \cost(\phi))$ and space $\calO(|\calE|\cdot |B|)$ to exploit GPU parallelism \citep{jain2020knowledge},\footnote{For large real-world KGs it is reasonable to assume that $|\calE|\gg|\calR|$.} where $\cost(\phi)$ denotes the complexity of evaluating the $\phi$ once. %

Note that the PLL objective (\cref{eq:pll-objective}) is a traditional proxy for learning \textit{generative} models for which it is infeasible to evaluate the \emph{maximum-likelihood estimation} (MLE) objective\footnote{In general, the PLL is recognised as a proxy for MLE because, under certain assumptions, it is possible to show it can retrieve the MLE solution asymptotically with enough data \citep{hyvrinen2006consistency}.}
\begin{equation}
    \label{eq:mle-objective}
    \calL_\text{MLE} := \sum\nolimits_{(s,r,o)\in\calG} \log p(s,r,o) = - |\calG|\log Z + \sum\nolimits_{(s,r,o)\in\calG} \phi(s,r,o).
\end{equation}
In theory, evaluating $\log p(s,r,o)$ exactly can be done in polynomial time under our three-variable interpretation, as computing $Z$ requires $\calO(|\calE|^2\cdot |\calR|\cdot \cost(\phi))$ time, but in practice
this cost is still prohibitive for real-world KGs.
In fact, it would require summing over $|\calE\times\calR\times\calE|$ evaluations of the score function $\phi$, which for \textsf{FB15k-237} \citep{toutanova2015observed},
 the small fragment  of Freebase \citep{nickel2016kgreview}, translates to \textasciitilde $10^{11}$ evaluations of $\phi$.
This practical bottleneck hinders the generative capabilities of these models and their ability to yield normalised
and interpretable probabilities.
Next, we show how we can reinterpret KGE score functions as to retrieve a generative model over triples
for which computing $Z$ exactly can be done in time $\calO((|\calE|+|\calR|)\cdot \cost(\phi))$, making renormalisation feasible.

\shortsection{\ldots to Circuits\ldots}
\label{sec:from-kge-models-to-circuits}

In this section, we show that popular and successful KGE models 
such as \CP, \RESCAL, \TuckER and \ComplEx (see \cref{fig:kge-circuits} and \cref{sec:background}), can be viewed as structured computational graphs that can, in principle,
enable summing over all possible triples efficiently.
Later, to exploit this efficient summation for marginalisation over triple probabilities, we reinterpret the semantics of these computational graphs as to yield circuits that output valid probabilities.
We start with the needed background about circuits and show that some score functions can be readily represented as circuits.

\begin{defn}[Circuit \citep{choi2020pc,vergari2021compositional}]
    \label{defn:circuit}
    A \emph{circuit} $\phi$ is a parametrized computational graph over variables $\vX$ encoding a function $\phi(\vX)$ and comprising three kinds of computational units: \emph{input}, \emph{product}, and \emph{sum}.
    Each product or sum unit $n$ receives
    as inputs
    the outputs of
    other units, denoted with the set $\inscope(n)$.
    Each unit $n$ encodes a function $\phi_n$ defined as: (i) $l_n(\scope(n))$ if $n$ is an input unit, where $l_n$ is a function over variables $\scope(n)\subseteq\vX$, called its \emph{scope}, (ii) $\prod_{i\in\inscope(n)} \phi_i(\scope(\phi_i))$ if $n$ is a product unit, and (iii) $\sum_{i\in\inscope(n)} \theta_i\phi_i(\scope(\phi_i))$ if $n$ is a sum unit, with $\theta_i\in\bbR$ denoting the weighted sum parameters.
    The scope of a product or sum unit $n$ is the union of the scopes of its inputs, i.e., $\scope(n) = \bigcup_{i\in\inscope(n)} \scope(i)$.
\end{defn}

\cref{fig:kge-circuits} and \cref{fig:tucker-eval} show examples of circuits.
Next, we introduce the two structural properties that  enable efficient summation and \cref{prop:kge-models-as-circuits} certifies that the aforementioned KGEs have these properties.

\begin{defn}[Smoothness and Decomposability]
    \label{defn:smoothness-decomposability}
    A circuit is \emph{smooth} if for every sum unit $n$, its input units depend all on the same variables, i.e, $\forall i,j\in\inscope(n)\colon \scope(i) = \scope(j)$.
    A circuit is \emph{decomposable} if the inputs of every product unit $n$ depend on disjoint sets of variables, i.e, $\forall i,j\in\inscope(n)\ i\neq j\colon \scope(i)\cap\scope(j) = \emptyset$.
\end{defn}

\begin{prop}[Score functions of KGE models as circuits]
    \label{prop:kge-models-as-circuits}
    The computational graphs of the score functions $\phi$ of \CP, \RESCAL, \TuckER and \ComplEx are smooth and decomposable circuits over $\vX=\{S,R,O\}$, 
    whose evaluation cost is $\cost(\phi)\in\Theta(|\phi|)$, where $|\phi|$ denotes the number of edges in the circuit, also called its size.
    For example, the size of the circuit for \CP is  $|\phi_{\CP}|\in\calO(d)$.
\end{prop}

\begin{figure}[!t]
    \centering
    \begin{subfigure}[t]{0.19\linewidth}
        \centering
        \includegraphics[scale=0.26]{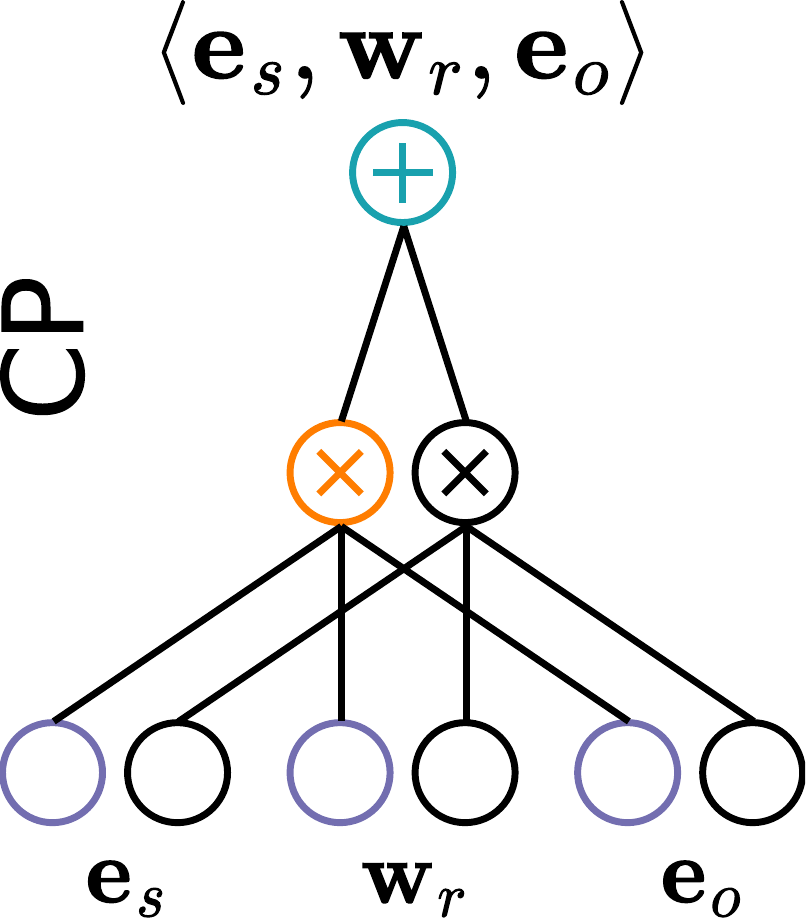}
    \end{subfigure}\hspace{1pt}
    \begin{subfigure}[t]{0.225\linewidth}
        \centering
        \includegraphics[scale=0.26]{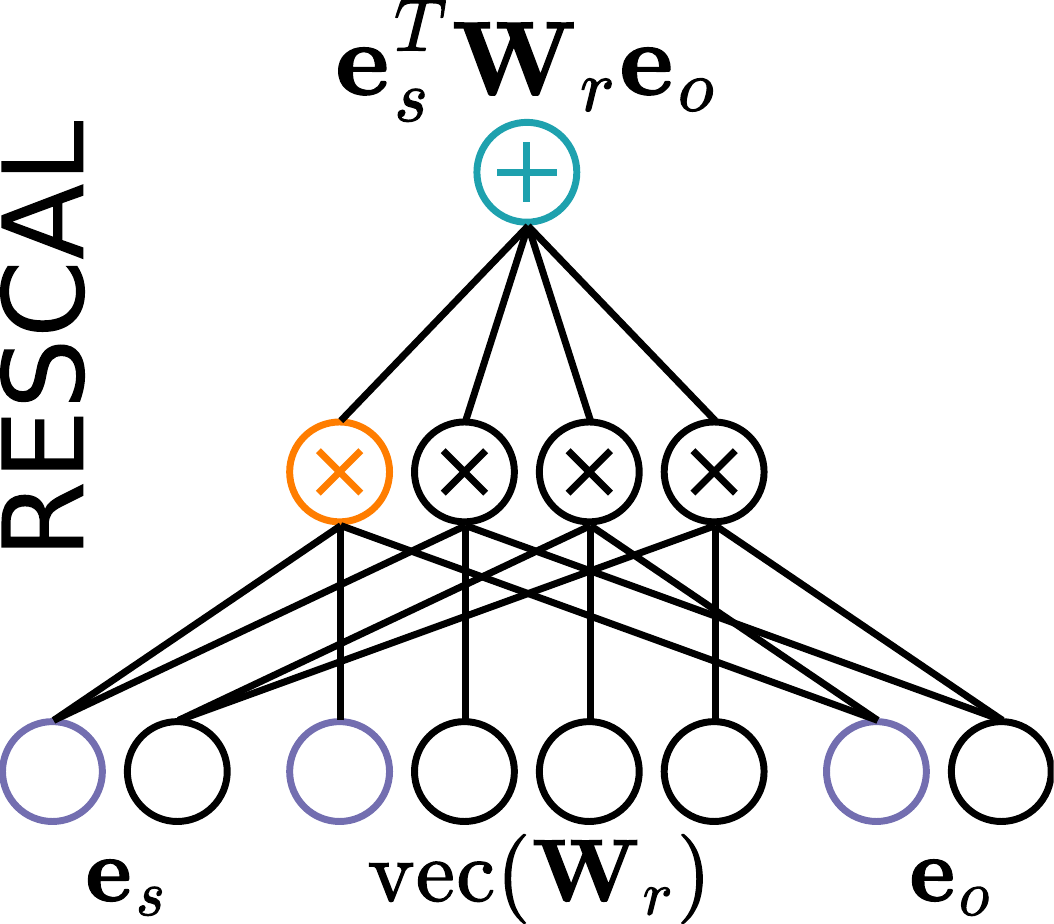}
    \end{subfigure}%
    \begin{subfigure}[t]{0.245\linewidth}
        \centering
        \includegraphics[scale=0.26]{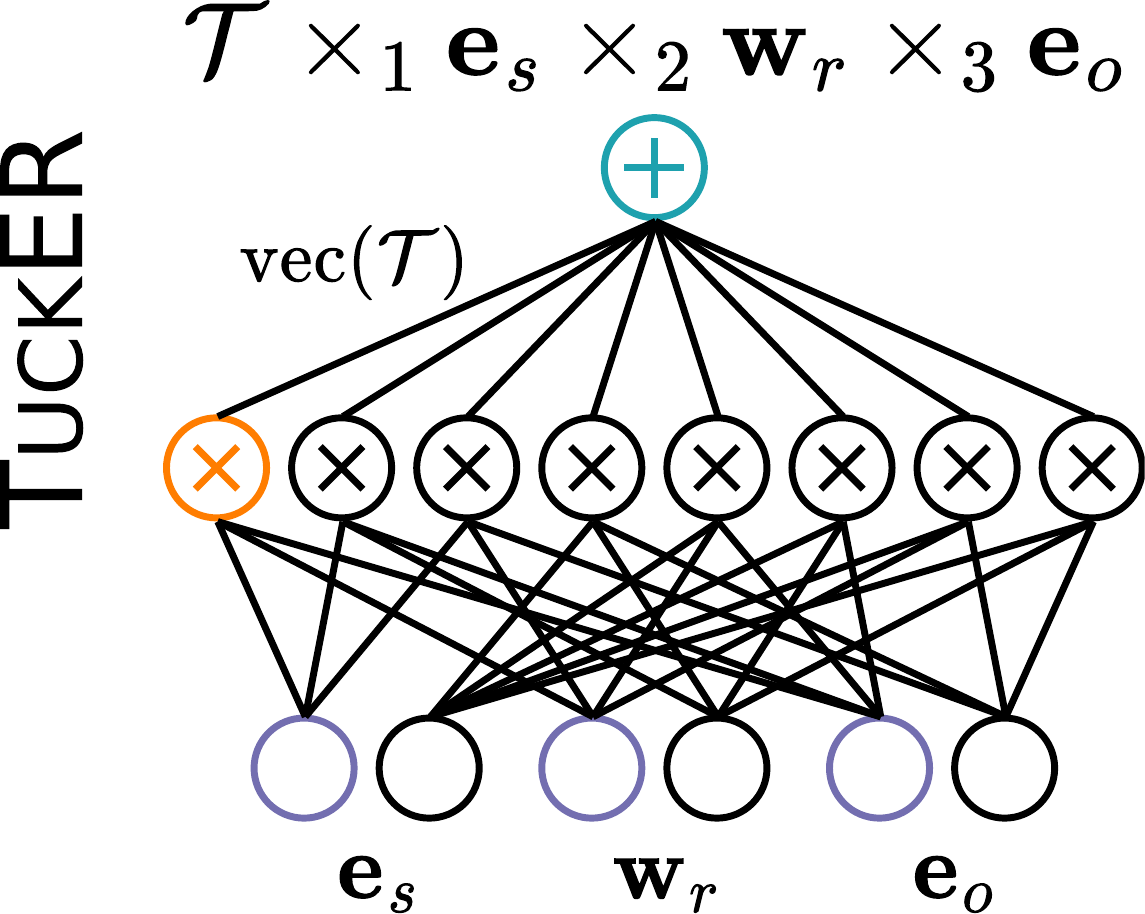}
    \end{subfigure}%
    \begin{subfigure}[t]{0.330\linewidth}
        \centering
        \includegraphics[scale=0.26]{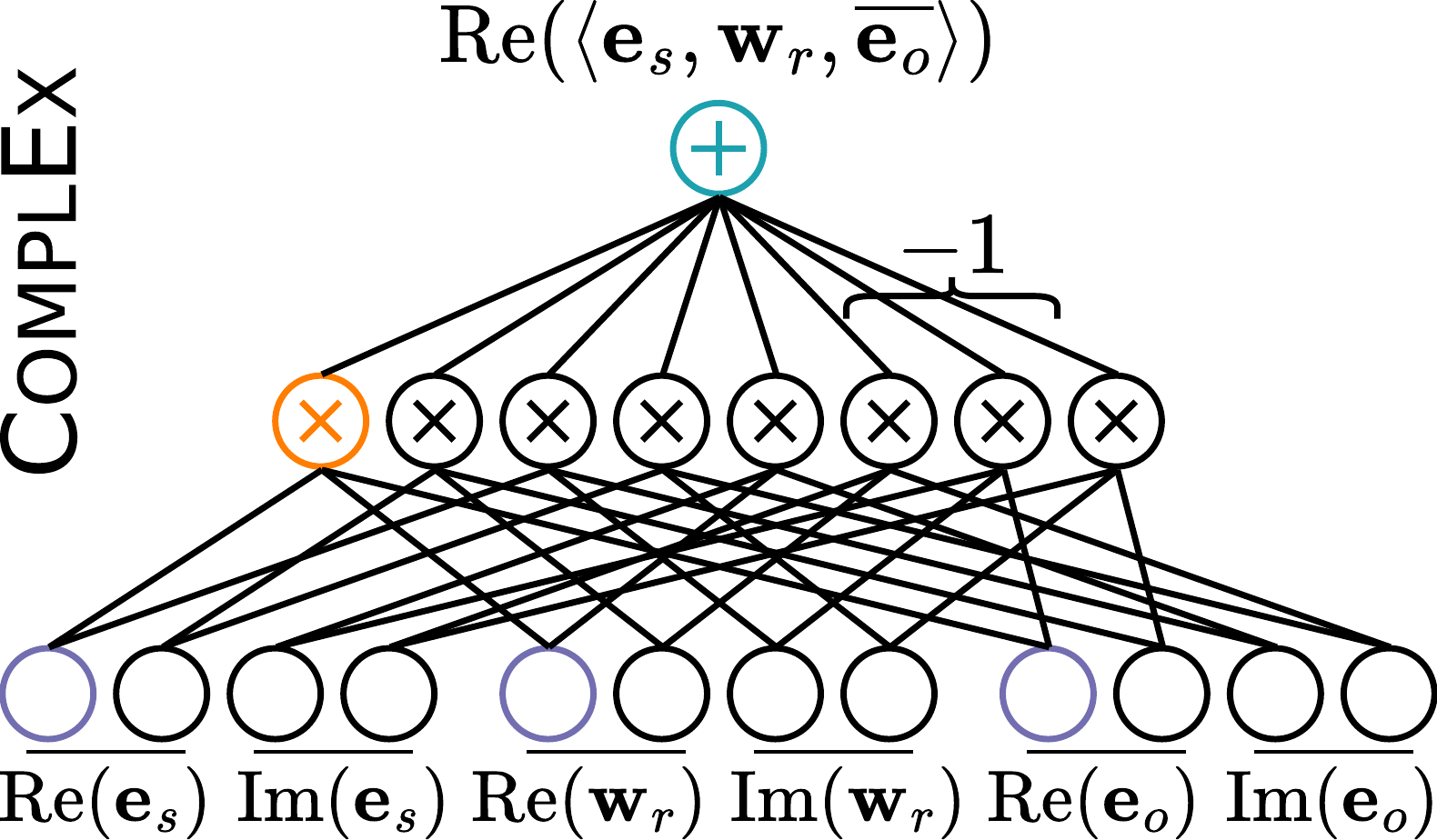}
    \end{subfigure}
    \caption{\textbf{Interpreting the score functions  of \CP, \RESCAL, \TuckER, \ComplEx as circuits} over 2-dimensional embeddings. Input, product and sum units are coloured in purple, orange and blue, respectively. Output sum units are labelled with the score functions, and their parameters are assumed to be $1$, if not specified.
    The detailed construction is presented in \cref{app:kge-as-circuits}.
    Given a triple $(s,r,o)$, the input units map subject $s$, predicate $r$ and object $o$ to their embedding entries.
    Then, the products are evaluated before the weighted sum, which outputs the score of the input triple.}
    \label{fig:kge-circuits}
\end{figure}

\cref{app:kge-as-circuits} reports the complete proof by construction for the score functions of these models
and the circuit sizes,
while \cref{fig:kge-circuits} illustrates them.
Intuitively, since the score functions of the cited KGE models are based on products and sums as operators, they can be represented as circuits where input units map entities and predicates into the corresponding embedding components (similarly to look-up tables).
As the inputs of each sum unit are product units that share the same scope $\{S,R,O\}$ and fully decompose it in their input units, they satisfy smoothness and decomposability (\cref{defn:smoothness-decomposability}).

Smooth and decomposable circuits enable summing over all possible (partial) assignments to $\vX$ by (i) performing summations at input units over values in their domains, and (ii) evaluating the circuit once in a feed-forward way \citep{choi2020pc,vergari2021compositional}.
This re-interpretation of score functions allows to ``push'' summations to the input units of a circuit, greatly reducing complexity, as detailed in the following proposition.

\begin{prop}[Efficient summations]
    \label{prop:efficient-summations}
    Let $\phi$ be a smooth and decomposable circuit over $\vX=\{S,R,O\}$ that encodes the score function of a KGE model.
    The sum $\sum_{s\in\calE} \sum_{r\in\calR} \sum_{o\in\calE} \phi(s,r,o)$ or any other summation over subjects, predicates or objects can be computed in time $\calO((|\calE| + |\calR|) \cdot |\phi|)$.
\end{prop}

However, \textit{these sums are in logarithm space}, as we have that $p(s,r,o) \propto \exp \phi(s,r,o)$ (see \cref{sec:background}).
As a consequence, summation in this context \emph{does not} correspond to marginalising variables in probability space.
This drives us to reinterpret the semantics of these circuit structures as to operate directly in probability space, rather than in logarithm space,
i.e., by encoding non-negative functions.

\shortsection{\ldots to Probabilistic Circuits}
\label{sec:casting-kge-models-to-pcs}

We now present how to reinterpret the semantics of the computational graphs of KGE score functions to directly output non-negative values for any input.
That is, we cast them as \emph{probabilistic circuits} (PCs) \citep{choi2020pc,vergari2021compositional,darwiche2002knowledge}.
First, we define our subclass of
PCs that encodes a possibly unnormalised probability distribution over triples in a KG, but allows for efficient marginalisation.

\begin{defn}[Generative KGE circuit]
    \label{def:generative-kge-circuits}
    A \emph{generative KGE circuit} (\OurModel) is a smooth and decomposable PC $\phi_\mathsf{pc}$ over variables $\vX=\{S,R,O\}$ that encodes a probability distribution over triples, i.e., $\phi_\mathsf{pc}(s,r,o) \propto p(s,r,o)$ for any $(s,r,o)\in\calE\times\calR\times\calE$.
\end{defn}

Since our \OurModels are smooth and decomposable (\cref{defn:smoothness-decomposability}) they guarantee the
efficient
computation of
$Z$
in time $\calO((|\calE| + |\calR|) \cdot |\phi_\mathsf{pc}|)$ (\cref{prop:efficient-summations}).
This is in contrast with existing KGE models, for which it would require the evaluation of the whole score function on each possible triple (\cref{sec:background}).
For example, assume a non-negative \CP score function $\phi^+_\CP(s,r,o) = \langle \ve_s,\vw_r,\ve_o \rangle \in\bbR_+$ for some embeddings $\ve_s,\vw_r,\ve_o\in\bbR^d$.
Then, we can compute $Z$ by pushing the outer summations inside the trilinear product, i.e., $Z = \langle \sum_{s\in\calE} \ve_s, \sum_{r\in\calR} \vw_r, \sum_{o\in\calE} \ve_o \rangle$, which can be done in time $\calO((|\calE| + |\calR|) \cdot d)$.
In the following sections, we propose two ways to turn
the computational graphs of \CP, \RESCAL, \TuckER and \ComplEx into \OurModels without
additional space requirements.

\shortsubsection{Non-negative restriction of a score function}
\label{sec:non-negative-restriction}

The most natural way of casting existing KGE models to \OurModels is by constraining the computational units of their circuit structures (\cref{sec:from-kge-models-to-circuits}) to output non-negative values only.
We will refer to this conversion method as the \emph{non-negative restriction} of a model.
To achieve the non-negative restriction of \CP, \RESCAL  and \TuckER we can simply restrict the embedding values and additional parameters in their score functions to be non-negative, as products and sums are operators closed in $\bbR_+$.
Thus, each input unit $n$ over variables $S$ or $O$ (resp. $R$)
in a non-negative \OurModel
encodes a function $l_n$ (\cref{defn:circuit}) modelling an unnormalised categorical distribution over entities (resp. predicates).
For example, each $i$-th entry $e_{si}$ of the embedding $\ve_s\in\bbR_+^d$ associated to an entity $s\in\calE$ becomes a parameter of the $i$-th unnormalised categorical distribution over $S$.
See \cref{fig:interpreting-non-negative} for an example.
We denote the non-negative restriction of these KGEs as \NNegCP, \NNegRESCAL and \NNegTuckER, respectively.

However, deriving \NNegComplEx \citep{trouillon2016complex} by restricting the embedding values of \ComplEx to be non-negative is not sufficient, because its score function includes a subtraction as it operates on complex numbers.
To overcome this, we re-parameterise the imaginary part of each complex embedding to be always greater than or equal to its real part.
\cref{app:realising-monocomplex} details this procedure.
Even if this reparametrisation allows for more flexibility,
imposing non-negativity on \OurModels can restrict their ability to capture intricate interactions over subjects, predicates and objects given a fixed number of learnable parameters \citep{valiant1979negation}.
We empirically confirm this in our experiments 
in \cref{sec:empirical-evaluation}.
Therefore, we propose an alternative way of representing KGEs as PCs via \emph{squaring}.

\shortsubsection{Squaring the score function}
\label{sec:non-monotonic-squaring}

Squaring works by taking the score function of a KGE model 
$\phi$, and multiplying it with itself to obtain $\phi^2=\phi\cdot\phi$.
Note that $\phi^2$ would be guaranteed to be a PC, as it always outputs non-negative values. %
The challenge is to represent the product of two circuits as a smooth and decomposable PC as to guarantee efficient marginalisation (\cref{prop:efficient-summations}).\footnote{In fact, even though $\phi^2$ can be easily built by introducing a product unit over two copies of $\phi$ as sub-circuits, the resulting circuit would be not decomposable (\cref{defn:smoothness-decomposability}), as the sub-circuits are defined on the same scope.}
In general, this task is known to be \#P-hard~\citep{vergari2021compositional}.

However, it can be done efficiently if the two circuits are \emph{compatible} \citep{vergari2021compositional}, as we further detail in \cref{app:tractable-product}.
Intuitively, the circuit representations of the score functions $\phi$ of \CP, \RESCAL, \TuckER and \ComplEx~(see \cref{fig:kge-circuits}) are simple enough that every product unit is defined over the same  scope $\{S,R,O\}$ and fully decomposes it on its input units.
As such, these circuits can be easily multiplied with any other smooth and decomposable circuit, a property also known as \textit{omni-compatibility} \citep{vergari2021compositional}.
This property enables us to build the squared version of these KGE models, which we denote as \SquaredCP, \SquaredRESCAL, \SquaredTuckER and \SquaredComplEx, as PCs that are still smooth and decomposable.
Note that these squared \OurModels do allow for negative parameters, and hence can be more expressive.
The next theorem, instead, guarantees that we can normalize them efficiently.

\begin{thm}[Efficient summations on squared \OurModels]
    \label{thm:marginalisation-squared-circuits}
    Performing summations as stated in \cref{prop:efficient-summations} on \SquaredCP, \SquaredRESCAL, \SquaredTuckER and \SquaredComplEx can be done in time $\calO((|\calE| + |\calR|)\cdot |\phi|^2)$.
\end{thm}

\begin{figure}[!t]
    \centering
    \includegraphics[width=.9\textwidth]{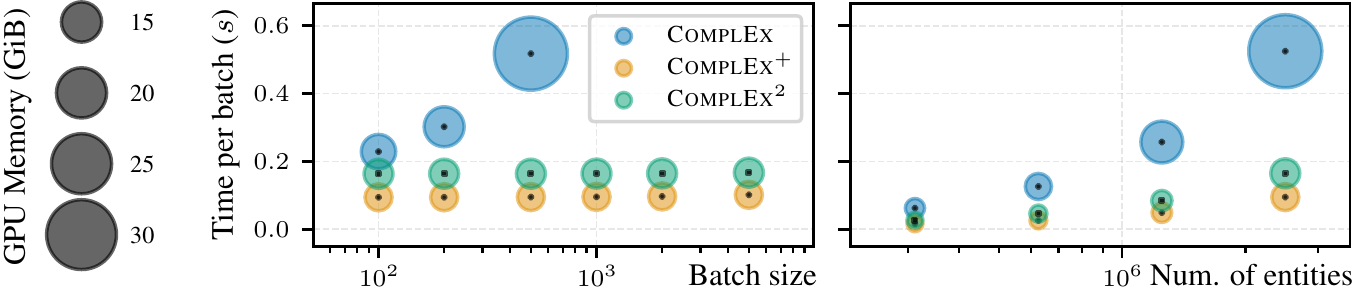}
    \caption{\textbf{\OurModels scale better.} Time (in seconds) and peak GPU memory (in GiB as bubble sizes) required for computing the PLL objective and back-propagating through it for a single batch on \textsf{ogbl-wikikg2}, by increasing the batch size and number of entities. See \cref{app:training-benchmark-details} for details.}
    \label{fig:pll-scaling}
\end{figure}

For instance, the partition function $Z$ of \SquaredCP with embedding size $d$ would require $\calO((|\calE| + |\calR|)\cdot d^2)$ operations to be computed, while a simple feed-forward pass for a batch of $|B|$ triples is still $\calO(|B|\cdot d)$.
While in this case marginalisation requires an increase in complexity that is quadratic in $d$, it is still faster than the brute force approach to compute $Z$ (see \cref{sec:background} and \cref{app:partition-function-squared-circuits}).

\paragraph{Quickly distilling KGEs to squared \OurModels.}
Consider a squared \OurModel obtained by initialising its parameters with those of its energy-based KGE counterpart.
If the score function of the original KGE model \emph{always} assigns non-negative scores to triples, then 
the ``distilled'' squared \OurModel will output the \textit{same exact ranking of the original model for the answers to any link prediction queries}. 
Although the premise of the non-negativity of the scores might not be guaranteed,  we observe that, in practice, learned KGE models do assign positive scores to all or most of the triples of common KGs (see \cref{app:distribution-scores}).
Therefore, we can use this observation to
either instantly distil a \OurModel or provide a good heuristic to initialise its parameters and fine-tune them (by MLE or PLL maximisation).
In both cases, the result is that we can convert the original KGE model into a \OurModel that
provides comparable probabilities, enable efficient marginalisation, sampling, and the integration of logical constraints with little or no loss of performance for link prediction (\cref{sec:empirical-evaluation-kbc}).

\shortsubsection{On the Training Efficiency of \OurModels}
\label{sec:training-efficiency}

\OurModels also offer an unexpected opportunity to better scale the computation of the PLL objective (\cref{eq:pll-objective}) on very large knowledge graphs.
This is because computing the PLL for a batch of $|B|$ triples with \OurModels obtained via non-negative restriction and by squaring (\cref{sec:non-negative-restriction,sec:non-monotonic-squaring}) does not require storing a matrix of size $\calO(|\calE|\cdot |B|)$ to fully exploit GPU parallelism \citep{jain2020knowledge}.
For instance, in \cref{app:complexity-pll-objective} we show that computing the PLL for \CP \citep{lacroix2018cp-kbc} with embedding size $d$ requires time $\calO(|\calE|\cdot |B|\cdot d)$ and additional space $\calO(|\calE|\cdot |B|)$.
On the other hand, for \SquaredCP (resp. \NNegCP) it requires time $\calO((|\calE| + |B|)\cdot d^2)$ (resp. $\calO((|\calE| + |B|)\cdot d)$) and space $\calO(|B|\cdot d)$.
\cref{tab:summary-complexity} summarises similar reduced complexities for other instances of \OurModels, such as \NNegComplEx and \SquaredComplEx.
The reduced time and memory requirements with \OurModels allow us to use larger batch sizes and better scale to large knowledge graphs.
\cref{fig:pll-scaling} clearly highlights this when measuring the time and GPU memory required to train these models on a KG with millions of entities such as \textsf{ogbl-wikikg2}~\citep{hu2020ogb}.

\shortsubsection{Sampling new triples with \OurModels}
\label{sec:sampling}

\OurModels only allowing non-negative parameters, such as \NNegCP, \NNegRESCAL and \NNegTuckER, support \emph{ancestral sampling} as sum units can be interpreted as marginalised discrete latent variables, similarly to the latent variable interpretation in mixture models \citep{poon2011sum,peharz2017latent-spn} (see \cref{app:sampling-procedures} for details).
This is however not possible in general for \NNegComplEx and \OurModels obtained by squaring, as 
negative parameters break this interpretation.
Luckily, as these circuits still support efficient marginalisation (\cref{prop:efficient-summations,thm:marginalisation-squared-circuits}) and hence also conditioning, we can perform \emph{inverse transform sampling}.
That is, to generate a triple $(s,r,o)$, we can sample in an autoregressive fashion, e.g., first $s \sim p(S)$, then $r \sim p(R\mid s)$
and $o \sim p(O\mid s,r)$, 
 hence requiring only three marginalization steps.

\shortsection{Injection of Logical Constraints}
\label{sec:injection-logical-constraints}

Converting KGE models to PCs provides the opportunity to ``embed'' logical constraints in the neural link predictor such that (i) predictions are always guaranteed to satisfy the constraints at test time, and (ii) training can still be done by efficient MLE (or PLL).
This is in stark contrast with previous approaches for KGEs, which relax the constraints or enforce them only at training time (see \cref{sec:related-works}).
Consider, as an example, the problem of integrating the logical constraints induced by a schema of a KG, i.e., enforcing that triples 
not satisfying a \textit{domain constraint} have probability zero. 
\begin{defn}[Domain constraint]
    \label{defn:domain-constraint}
    Given a predicate $r\in\calR$ and $\kappa_S(r),\kappa_O(r)\subseteq\calE$
    the sets of all subjects and objects that are semantically coherent with respect to $r$, a \emph{domain constraint} is a propositional logic formula defined as
    \begin{equation}
        \label{eq:domain-constraint-logic-formula}
        K_r \equiv S \in \kappa_S(r) \land R = r \land O \in \kappa_O(r) \equiv (\lor_{u\in\kappa_S(r)} S = u) \land R = r \land (\lor_{v\in\kappa_O(r)} O = v).
    \end{equation}    
\end{defn}
\vspace{-5pt}
Given $\calR=\{r_1,\ldots,r_m\}$ a set of predicates, the disjunction $K \equiv K_{r_1} \lor \ldots \lor K_{r_m}$ encodes all the domain constraints that are defined in a KG.
An input triple $(s,r,o)$ satisfies $K$, written as $(s,r,o)\models K$, if 
$s\in\kappa_S(r)$ and $o\in\kappa_O(r)$.
To design \OurModels such as their predictions always satisfy logical constraints (which might not be necessarily domain constraints), we follow  \citet{ahmed2022semantic} and define a score function to represent a probability distribution $p_K$ that assigns probability mass only to triples that satisfy the constraint $K$, i.e.,
$\phi_\mathsf{pc}(s,r,o) \cdot c_K(s,r,o) \propto p_K(s,r,o)$.
Here, $\phi_\mathsf{pc}$ is a \OurModel 
and $c_K(s,r,o) = \Ind{(s,r,o)\models K}$ is an indicator function that ensures that zero mass is assigned to triples violating $K$.
In words, we are ``cutting'' the support of $\phi_{\mathsf{pc}}$, as illustrated in \cref{fig:constraint-injection-example}.

Computing $p_K(s,r,o)$ exactly but naively would require computing a new partition function $Z_K=\sum_{s'\in\calE} \sum_{r'\in\calR} \sum_{o'\in\calE} (\phi_\mathsf{pc}(s',r',o') \cdot c_K(s',r',o'))$, which is impractical
as previously discussed
(\cref{sec:background}).
Instead, we compile $c_K$ as a smooth and decomposable circuit, sometimes called a constraint or logical circuit \citep{darwiche2002knowledge,ahmed2022semantic}, by leveraging compilers from the \emph{knowledge compilation} literature ~\citep{darwiche2011sdd,oztok2015top-down-sdds}.
In a nutshell, $c_K$ is another circuit over variables $S,R,O$ that outputs 1 if an input triple satisfies the encoded logical constraint $K$ and 0 otherwise.
See \cref{defn:constraint-circuit} for a formal definition of such circuits.
Then, similarly to what we have showed for computing squared circuits that enable efficient marginalisation (\cref{sec:non-monotonic-squaring}), the satisfaction of compatibility between a \OurModel $\phi_\mathsf{pc}$ and a constraint circuit $c_K$ enable us to compute $Z_K$ efficiently, as certified by the following theorem.

\begin{figure}[!t]
    \centering
    \includegraphics[width=0.95\linewidth]{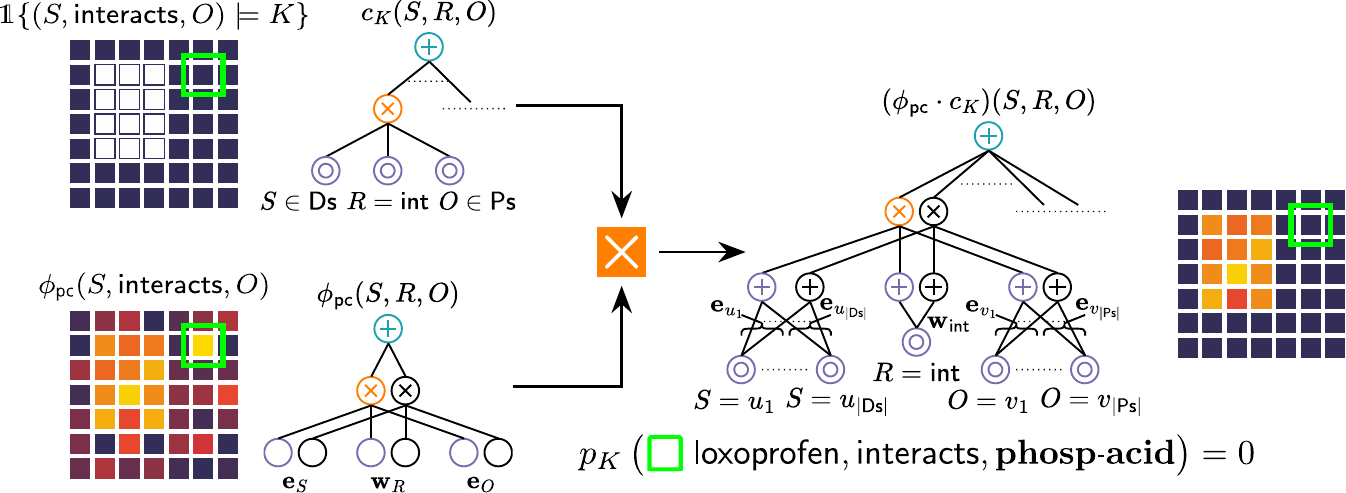}

    \caption{\textbf{Injection of domain constraints.} Given a circuit $c_K$ encoding domain constraints and a \OurModel $\phi_\mathsf{pc}$, the probability assigned by the product circuit $\phi_\mathsf{pc} \cdot c_K$ to the inconsistent triple showed in \cref{sec:introduction} is 0, and a positive probability is assigned to consistent triples only, e.g., for the \textsf{interacts} predicate those involving drugs (\textsf{Ds}) as subjects and proteins (\textsf{Ps}) as objects. Best viewed in colours.}
    \label{fig:constraint-injection-example}
    \vspace{-7pt}
\end{figure}
\begin{thm}[Tractable integration of constraints in \OurModels]
    \label{thm:tractable-integration-knowledge}
    Let $c_K$ be a constraint circuit encoding a logical constraint $K$ over variables $\{S,R,O\}$.
    Then exactly computing
    the partition function $Z_K$
    of
    the product
    $\phi_\mathsf{pc}(s,r,o) \cdot c_K(s,r,o) \propto p_K(s,r,o)$
    for any \OurModel $\phi_\mathsf{pc}$ derived from \CP, \RESCAL, \TuckER or \ComplEx (\cref{sec:casting-kge-models-to-pcs})
    can be done in time $\calO((|\calE| + |\calR|)\cdot |\phi_\mathsf{pc}|\cdot |c_K|)$.
\end{thm}

In \cref{prop:domain-constraints-circuit} we show that the compilation of domain constraints $K$ (\cref{defn:domain-constraint}) is straightforward and results in a constraint circuit $c_K$ having compact size.
For example, the size of the constraint circuit encoding the domain constraints of \textsf{ogbl-biokg} is approximately $|c_K| = 307\cdot 10^3$.
To put this number in perspective, the size of the circuit for \ComplEx with embedding size $1000$ is the much larger $375\cdot 10^6$.
Since $|c_K|$ is much smaller, by the same argument on the efficiency of \OurModels obtained via squaring (\cref{sec:training-efficiency}) it results that the integration of logical constraints adds a negligible overhead.

\shortsection{Related Work}
\label{sec:related-works}

\paragraph{SOTA KGEs and current limitations.}
A plethora of ways to represent and learn KGEs has been proposed, see \citep{hogan2021knowledge} for a review. 
KGEs such as \CP and \ComplEx are still the de facto go-to choices in many applications \citep{bonner2022understanding,ruffinelli2020teach-dog-tricks,chen2021relational-auxiliary-objective}.
Works performing density estimation in embedding space~\citep{xiao2016transg,chen2021probabilistic} can sample embeddings, but to sample triples one would need to train a decoder.
Several works try to modify training for KGEs as to introduce a penalty for triples that do not satisfy given logical constraints 
\citep{chang2014typed,krompass2015tc-learning-kg,hubert2023enhancing,minervini2016schema-latent-factors,ding2018improving,guo2020knowledge}, or casting it as a min-max game \citep{minervini2017adversarial}.  
Unlike our \OurModels (\cref{sec:injection-logical-constraints}),
none of these approaches
guarantee that test-time predictions satisfy the constraints.
Moreover, several heuristics have been proposed to calibrate the probabilistic predictions of KGE models ex-post \citep{tabacof2019probability,zhu2023closer-look-calibration}. 
As showed in \citep{tresp2021brain}, the triple distribution can be modelled autoregressively as $p(S,R,O) = p(S)\cdot p(O\mid S)\cdot p(R\mid S,O)$ where each conditional distribution is encoded by a neural network.
However, differently from our \OurModels, 
integrating constraints exactly or computing \textit{any} marginal (thus conditional) probability is inefficient.
KGE models based on non-negative tensor decompositions \citep{tresp2015learning} are equivalent to \OurModels obtained by non-negative restriction (\cref{sec:non-negative-restriction}), but are generally trained by minimizing different non-probabilistic losses.

\shortparagraph{Circuits.} 
Circuits provide a unifying framework for several tractable probabilistic models such as sum-product networks (SPNs) and hidden Markov models, which are smooth and decomposable PCs~\citep{choi2020pc}, as well as compact logical representations \citep{darwiche2002knowledge,ahmed2022semantic}.
See \citep{vergari2019tractable,choi2020pc,darwiche2009modeling} for an overview.
PCs with negative parameters are also called non-monotonic \citep{shpilka2010arithmetic}, but are surprisingly not as well investigated as their monotonic counterparts, i.e., PCs with only non-negative parameters, at least from the learning perspective.
Similarly to our construction for \NNegComplEx (\cref{app:realising-monocomplex}),
\citet{dennis2016twinspn} constrains the output of the non-monotonic sub-circuits of a larger PC to be less than their monotonic counterparts.
Squaring a circuit has been investigating for tractably computing several divergences \citep{vergari2021compositional} and is related to the Born-rule of quantum mechanics \citep{novikov2021ttde}.

\shortparagraph{Circuits for relational data.}
Logical circuits to compile formulas in first-order logic (FOL) \citep{fierens2015inference} have been used to reason over relational data, e.g. via exchangeability \citep{van2011lifted,niepert2014tractability}.
Other formalisms such as tractable Markov Logic~\citep{webb2013tractable}, probabilistic logic bases \citep{niepert2015learning}, relational SPNs~\citep{nath2015rspn,nath2015learning} and generative clausal networks \citep{ventola2022generative} use underlying circuit-like structures to represent probabilistic models over a tractable fragment of FOL formulas.
These works assume that every atom in a grounded formula is associated to a random variable, also called the possible world semantics in probabilistic logic programs \citep{sato1997prism} and databases (PDBs)~\citep{dalvi2007efficient}.
In
this semantics, TractOR \citep{friedman2020symbolic-pdb} casts answering complex queries over KGEs as to performing inference in PDBs.
Differently from these works, our \OurModels are models defined over only three variables (\cref{sec:background}).
In \cref{app:interpreting-density-triples} we reconcile these two semantics by interpreting the probability of a triple to be proportional to that of all KGs containing it.

\begin{table}[!t]
    \tablesize
    \setlength{\tabcolsep}{5pt}
    \begin{minipage}{0.4\linewidth}
        \caption{\textbf{\OurModels are competitive with their energy-based counterparts.} Best average test MRRs of \CP, \ComplEx and \OurModels trained with the PLL and MLE objectives (\cref{eq:pll-objective,eq:mle-objective}).
        For standard deviations and training times see \cref{tab:best-results-kbc-extended}.}
    \label{tab:best-results-kbc}
    \end{minipage}
    \hspace{2em}
    \begin{minipage}{0.575\linewidth}
    \begin{tabular}{l
            cc
            cc
            cc}
        \toprule
        \multirow{2}{*}{\textbf{Model}}
        & \multicolumn{2}{c}{\textsf{\tablesize FB15k-237}}
        & \multicolumn{2}{c}{\textsf{\tablesize WN18RR}}
        & \multicolumn{2}{c}{\textsf{\tablesize ogbl-biokg}} \\
        \cmidrule(lr){2-3}
        \cmidrule(lr){4-5}
        \cmidrule(lr){6-7}
        & PLL & MLE
        & PLL & MLE
        & PLL & MLE \\
        \midrule
        \CP                 
            & 0.310 &   ---
            & \bfseries 0.105 & ---
            & 0.831 &   --- \\
        \NNegCP
            & 0.237 & 0.230 %
            & 0.027 & 0.026 %
            & 0.496 & 0.501\\
        \SquaredCP          
            & \bfseries 0.315 &  0.282 
            & \bfseries 0.104 &  0.091 
            & \bfseries 0.848 &  0.829 \\
        \midrule
        \ComplEx            
            & \bfseries 0.342 & ---
            & \bfseries 0.471 & ---
            & 0.829 & --- \\
        \NNegComplEx        
            & 0.214 & 0.205 %
            & 0.030 & 0.029 %
            & 0.503 & 0.516 \\
        \SquaredComplEx     
            & 0.334 & 0.300
            & 0.420 & 0.391
            & \bfseries 0.858 & 0.840 \\
        \bottomrule
    \end{tabular}
    \end{minipage}
\end{table}

\shortsection{Empirical Evaluation}
\label{sec:empirical-evaluation}

We aim to answer the following research questions:
\textbf{RQ1)} are \OurModels competitive with commonly used KGEs for link prediction? 
\textbf{RQ2)} Does integrating domain constraints in \OurModels benefit training and prediction?;
\textbf{RQ3)} how good are the triples sampled from \OurModels?

\shortsubsection{Link Prediction (\textbf{RQ1})}
\label{sec:empirical-evaluation-kbc}

\paragraph{Experimental setting.}
We evaluate \OurModels on standard KG benchmarks for link prediction\footnote{Code is available at \url{https://github.com/april-tools/gekcs}.}: \textsf{FB15k-237} \citep{toutanova2015observed}, \textsf{WN18RR} \citep{dettmers2018conv2d-kge} and \textsf{ogbl-biokg}  \citep{hu2020ogb}, whose statistics can be found in \cref{app:datasets-statistics}.
As usual \citep{nickel2016kgreview,ruffinelli2020teach-dog-tricks,pezeshkpour2020revisiting-evaluation-kbc}, we assess the models for 
predicting objects (queries $(s,r,?)$) and subjects (queries $(?,r,o)$), and report their \emph{mean reciprocal rank} (MRR) and \emph{fraction of hits at $k$} (Hits@$k$) (see \cref{app:metrics}).
We remark that our aim in this Section is \textit{not to score the new state-of-the-art link prediction performance on these benchmarks}.
Instead, we aim to rigorously assess how close \OurModels can be to commonly used and reasonably tuned KGE models.
We focus on  \CP and \ComplEx as they currently are the go-to models of choice for link prediction \citep{lacroix2018cp-kbc,ruffinelli2020teach-dog-tricks,chen2021relational-auxiliary-objective}.
We  compare them against our \OurModels \NNegCP, \NNegComplEx, \SquaredCP and \SquaredComplEx (\cref{sec:casting-kge-models-to-pcs}).
\cref{app:experimental-setting} collects all the details about the model hyperparameters and training for reproducibility.

\shortparagraph{Link prediction results.}
\cref{tab:best-results-kbc} reports the MRR and times for all benchmarks and models when trained by PLL or MLE.
First, \SquaredCP and \SquaredComplEx achieve competitive scores when compared to \CP and \ComplEx.
Moreover, \SquaredCP (resp. \SquaredComplEx) always outperforms \NNegCP (resp. \NNegComplEx), thus providing empirical evidence that negative embedding values are crucial for model expressiveness.
Concerning times,
\cref{tab:best-results-kbc-extended} shows that
squared \OurModels
can train much faster on large KGs
(see \cref{sec:training-efficiency}):
\SquaredCP and \SquaredComplEx require less than half the training time of \CP and \ComplEx on \textsf{ogbl-biokg}, while also unexpectedly scoring the current
SOTA
MRR on it.\footnote{Across non-ensemble methods and according to the \href{https://ogb.stanford.edu/docs/leader\_linkprop/\#ogbl-biokg}{OGB leaderboard}, updated at the time of this writing.}
We experiment also on the much larger \textsf{ogbl-wikikg2} KG \citep{hu2020ogb}, comprising millions of entities. 
Even more remarkably, we are able to score an MRR of 0.572 after just \textasciitilde 3 hours with \SquaredComplEx trained by PLL with a batch size of $10^4$ and embedding size $d=100$.
To put this in context, we were able to score 0.562 with the best configuration of \ComplEx \textit{but after \textasciitilde 3 days}, as we could not fit in memory more than a batch size 500.\footnote{A smaller ($d=50$) and highly tuned version of \ComplEx  achieves 0.639 MRR but still after days \citep{chen2021relational-auxiliary-objective}.}
The same trends are shown for the Hits@$k$ (\cref{tab:best-results-kbc-hits}) and likelihood (\cref{tab:average-ll}) metrics.
\shortparagraph{Distilling \OurModels.}
\cref{tab:parameters-distillation} reports the results achieved by \SquaredCP and \SquaredComplEx initialised with the parameters of learned \CP and \ComplEx (see \cref{sec:non-monotonic-squaring}) and confirms we can quickly turn an EBM into \OurModel, thus inheriting all the perks of being a tractable generative model.

\begin{figure}[!t]
    \tablesize
    \centering
    \setlength{\tabcolsep}{5pt}
    \raisebox{56pt}{
    \begin{tabular}{l@{\ }rrrrr}
        \toprule
        \multirow{2}{*}{\textbf{Model}}        & \multirow{2}{*}{$k$} & \multicolumn{4}{c}{\textbf{Embedding size}} \\
                                               &     &    10  &    50 &    200 &   1000 \\
        \midrule
        \multirow{3}{*}{\ComplEx}              &   1 &  99.68 & 99.90 & 99.93 &  99.94 \\
                                               &  20 &  99.81 & 99.79 & 99.85 &  99.91 \\
                                               & 100 &  99.60 & 99.44 & 99.60 &  99.77 \\
        \midrule
        \multirow{3}{*}{\SquaredComplEx}       &   1 &  82.50 & 94.22 & 99.30 &  99.50 \\
                                               &  20 &  86.50 & 96.70 & 99.42 &  99.64 \\
                                               & 100 &  90.66 & 97.71 & 99.23 &  98.78 \\
        \midrule
        \multirow{3}{*}{\ConstrSquaredComplEx} &   1 & \bf 100.00 & \bf 100.00 & \bf 100.00 & \bf 100.00 \\
                                               &  20 & \bf 100.00 & \bf 100.00 & \bf 100.00 & \bf 100.00 \\
                                               & 100 & \bf 100.00 & \bf 100.00 & \bf 100.00 & \bf 100.00 \\
        \bottomrule
    \end{tabular}
    }
    \hspace{30pt}
    \raisebox{5pt}{
    \includegraphics[scale=.925]{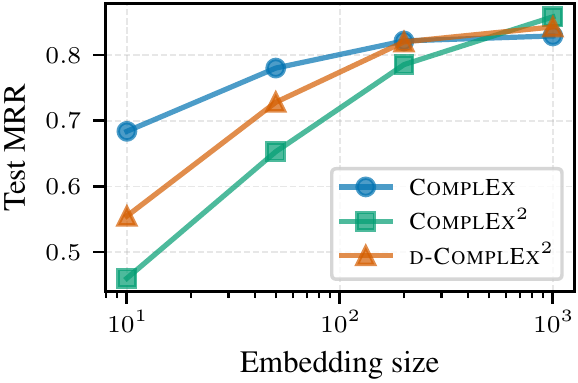}
    }
    \caption{\textbf{\OurModels with domain constraints guarantee domain-consistent predictions.}
    Semantic consistency scores (Sem@$k$) \citep{hubert2022new-strategies} on \textsf{ogbl-biokg} achieved by \ComplEx, \SquaredComplEx and its integration with domain constraints (\ConstrSquaredComplEx) (left), and MRRs computed on test queries (right). %
    \ComplEx infers 200+ triples violating constraints as the highest scoring completions ($k=1$).
    }
    \label{fig:rank-ablation-domain-constraints}
\end{figure}

\shortparagraph{Calibration study.}
We also measure how well calibrated the predictions of the models in \cref{tab:best-results-kbc} are, which is essential to ensure trustworthiness in critical tasks.
For example, given a perfectly calibrated model, for all the triples predicted with a probability of 80\%, exactly 80\% of them would actually exist \citep{zhu2023closer-look-calibration}.
On all KGs but \textsf{WN18RR}, \OurModels achieve lower empirical calibration errors~\citep{guo2017calibration} and better calibrated curves than their counterparts, as we report in \cref{app:calibration-diagrams}.
The worse performance of all models on \textsf{WN18RR} can be explained by the distribution shift that exists between its training and test split, which we better confirm in \cref{sec:empirical-evaluation-sampling}.

\shortsubsection{Integrating Domain Constraints (\textbf{RQ2})}
\label{sec:empirical-evaluation-dc}

We focus on \textsf{ogbl-biokg} \citep{hu2020ogb}, as it contains the domain metadata for each entity (i.e., disease, drug, function, protein, or side effect).
Given the entity domains allowed for each predicate,
we formulate
domain constraints as in \cref{defn:domain-constraint}.
First, we want to estimate how likely are the models to predict triples that do not satisfy the domain constraints.
We focus on \ComplEx and \SquaredComplEx, as they have been shown to achieve the best results in \cref{sec:empirical-evaluation-kbc} and introduce
\ConstrSquaredComplEx as the constraints-aware version of 
\SquaredComplEx (\cref{sec:injection-logical-constraints}).
For each test query $(s,r,?)$ (resp. $(?,r,o)$), we compute the  Sem@$k$ score \citep{hubert2022new-strategies} as the average percentage of triples in the first $k$ positions of the rankings of potential object (resp. subject) completions that satisfy the domain constraints (see \cref{app:metrics}).

\cref{fig:rank-ablation-domain-constraints} 
highlights how both \ComplEx and \SquaredComplEx systematically predict object (or subject) completions that violate domain constraints even for large embedding sizes.
For instance,
a Sem@1 score of 99\% (resp. 99.9\%) means that \textasciitilde3200 (resp. \textasciitilde320) predicted test triples violate domain constraints.
While for \ComplEx and \SquaredComplEx there is no theoretical guarantee of consistent predictions with respect to the domain constraints,
\ConstrSquaredComplEx
always guarantee consistent predictions \textit{by design}.
Furthermore,  we observe a significant improvement in terms of MRR when
integrating constraints for smaller embedding sizes, as reported in 
\cref{fig:rank-ablation-domain-constraints}.

\shortsubsection{Quality of sampled triples (\textbf{RQ3})}
\label{sec:empirical-evaluation-sampling}

Inspired by the literature on evaluating deep generative models for images, we propose a metric akin to the \emph{kernel Inception distance} \citep{binkowski2018demystifying-mmd} to evaluate the quality of the triples we can sample with \OurModels.

\begin{defn}[Kernel triple distance (KTD)]
    \label{defn:kernel-triple-distance}
    Given $\bbP,\bbQ$ two probability distributions over triples,
    and a positive definite kernel $k\colon\bbR^h\times\bbR^h\to\bbR$,
    we define the \emph{kernel triple distance} $\operatorname{KTD}(\bbP,\bbQ)$ as the squared \emph{maximum mean discrepancy} \citep{gretton2012mmd} between  triple latent representations obtained via a map
    $\psi\colon\calE\times\calR\times\calE\to\bbR^h$ that projects triples to an $h$-dimensional embedding, i.e., %
    \begin{equation*}
        \operatorname{KTD}(\bbP,\bbQ) = \bbE_{x,x'\sim \bbP} [k(\psi(x),\psi(x'))] + \bbE_{y,y'\sim \bbQ} [k(\psi(y),\psi(y'))] - 2\cdot\bbE_{{x\sim \bbP, y\sim \bbQ}} [k(\psi(x),\psi(y))].
    \end{equation*}
\end{defn}

\begin{table}[!t]
\begin{minipage}{.36\textwidth}
        \caption{\textbf{\OurModels trained by MLE generate new
        likely triples.} Empirical KTD scores between test triples and triples generated by baselines and \OurModels trained with the PLL objective or by MLE (\cref{eq:pll-objective,eq:mle-objective}). Lower is better. For standard deviations see \cref{tab:ktd-samples-extended}.
    }
    \label{tab:ktd-samples}
    \end{minipage}\hfill\begin{minipage}{.6\textwidth}
    \tablesize
    \centering
\begin{tabular}{lcccccc}
        \toprule
        \textbf{Model}
        & \multicolumn{2}{c}{\textsf{\tablesize FB15k-237}}
        & \multicolumn{2}{c}{\textsf{\tablesize WN18RR}}
        & \multicolumn{2}{c}{\textsf{\tablesize ogbl-biokg}} \\
        \cmidrule(lr){2-3} \cmidrule(lr){4-5} \cmidrule(lr){6-7}
        Training set
            & \multicolumn{2}{c}{0.055}
            & \multicolumn{2}{c}{0.260}
            & \multicolumn{2}{c}{0.029} \\
        Uniform        
            & \multicolumn{2}{c}{0.589}
            & \multicolumn{2}{c}{0.766}
            & \multicolumn{2}{c}{1.822} \\
        NNMFAug        
            & \multicolumn{2}{c}{0.414}
            & \multicolumn{2}{c}{0.607}
            & \multicolumn{2}{c}{0.518} \\
        \midrule
        & PLL & MLE
        & PLL & MLE
        & PLL & MLE \\
        \cmidrule(lr){2-3} \cmidrule(lr){4-5} \cmidrule(lr){6-7}
        \NNegCP        
            & 0.404 & 0.433
            & 0.633 & \bfseries 0.578
            & 0.966 & 0.738 \\
        \SquaredCP     
            & 0.253 & \bfseries 0.070
            & 0.768 & 0.768
            & 0.039 & \bfseries 0.017 \\
        \cmidrule(lr){1-7}
        \NNegComplEx
            & 0.336 & 0.323
            & 0.456 & 0.478
            & 0.175 & 0.097 \\
        \SquaredComplEx
            & 0.326 & \bfseries 0.102
            & 0.338 & \bfseries 0.278
            & 0.104 & \bfseries 0.034 \\
        \bottomrule
    \end{tabular}
\end{minipage}

\end{table}

An empirical estimate of the KTD score  close to zero indicates that there is little difference between the two triple distributions $\bbP$ and $\bbQ$ (see \cref{app:empirical-ktd-score}).
For images, $\psi$  is typically chosen as the last embedding of a SOTA neural classifier.
We choose $\psi$ to be the $L_2$-normed outputs of the product units of a circuit \citep{vergari2018sum,vergari2019visualizing}, specifically the
SOTA
\ComplEx
learned by \citet{chen2021relational-auxiliary-objective} with $h=4000$.
We choose $k$ as the polynomial kernel $k(\vx,\vy) = (\vx^\top\vy + 1)^3$, following \citet{binkowski2018demystifying-mmd}.

\cref{tab:ktd-samples} shows the empirical KTD scores computed between the test triples and the generated ones,
and \cref{fig:tsne-triple-embeddings} visualises triple embeddings.
We employ two baselines: a uniform probability distribution over all possible triples and NNMFAug~\citep{chauhan2021probabilistic}, the only work to address triple sampling to the best of our knowledge.
We also report the KTD scores for training triples as an empirical lower bound.
Squared \OurModels achieve lower KTD scores with respect to the ones obtained by non-negative restriction, confirming again a better estimation of the joint distribution.
In addition, they achieve far lower KTD scores 
than all competitors when learning by MLE (\cref{eq:mle-objective}), which justifies its usage as an objective. %
Lastly, we confirm the distribution shift on \textsf{WN18RR}: training set KTD scores are far from zero, but
even in this challenging scenario, \SquaredComplEx scores KTD values that are closer to the training ones.

\section{Conclusions and Future Work}
\label{sec:conclusion}

We proposed to re-interpret the representation and learning of widely used KGE models such as \CP, \RESCAL, \TuckER and \ComplEx, as
generative models, overcoming some of the classical limitation of their usual EBM interpretation (see \cref{sec:introduction,sec:background}).
\OurModel-variants for other KGE models whose scores are multilinear maps can be readily devised in the same way.
Moreover, we conjecture that other KGE models defining score functions having a distance-based semantics such as TransE~\citep{bordes2013transe} and RotatE~\citep{sun2019rotate} can be reinterpreted to be \OurModels as well.
Our \OurModels open up a number of interesting future directions.
First, we plan to investigate how the enhanced efficiency and calibration of \OurModels can help in complex reasoning tasks beyond link prediction~\citep{arakelyan2021complex-qa-tnorms}.
Second, we can leverage the rich literature on learning the structure of circuits \citep{vergari2015simplifying,vergari2019tractable} to devise smaller and sparser KGE circuit architectures that better capture the triple distribution or sporting structural properties that can make reasoning tasks other than marginalisation efficient \citep{di2017fast,dang2020strudel,ventola2022generative}.

\begin{ack}

We would like to acknowledge Iain Murray for his thoughtful feedback and suggestions on a draft of this work.
In addition, we thank Xiong Bo and Ricky Zhu for pointing out related works in the knowledge graphs literature.
AV was supported by the "UNREAL: Unified Reasoning Layer for Trustworthy ML" project (EP/Y023838/1) selected by the ERC and funded by UKRI EPSRC.
RP was supported by the Graz Center for Machine Learning (GraML).
\end{ack}

\bibliography{references}

\cleardoublepage
\appendix

\counterwithin{table}{section}
\counterwithin{figure}{section}
\renewcommand{\thetable}{\thesection.\arabic{table}}
\renewcommand{\thefigure}{\thesection.\arabic{figure}}

\section{Proofs}
\label{sec:proofs}

\subsection{KGE Models as Circuits}
\label{app:kge-as-circuits}

\begin{figure}[!t]
    \centering
    \begin{subfigure}[t]{0.3\linewidth}
        \centering
        \includegraphics[scale=0.36]{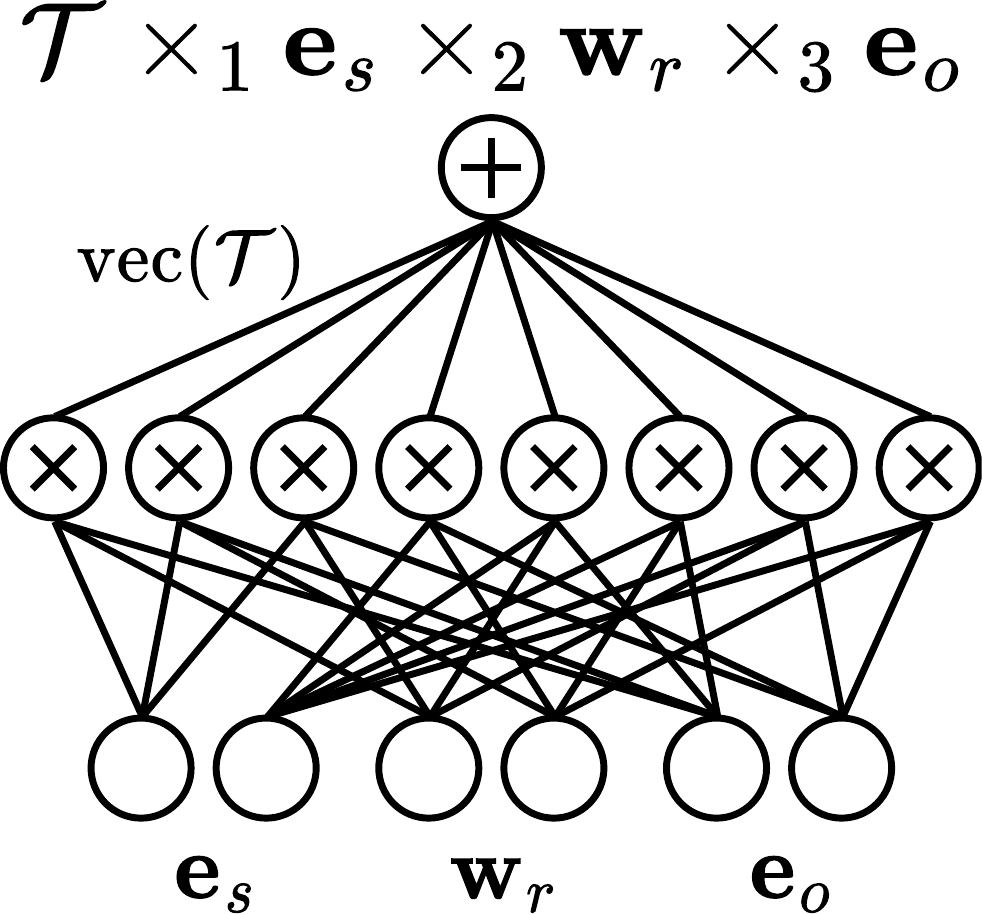}
        \caption{}
        \label{fig:tucker-eval-plain}
    \end{subfigure}
    \hfill
    \begin{subfigure}[t]{0.65\linewidth}
        \centering
        \setlength\arraycolsep{3pt}
        \begin{tikzpicture}[inner sep=2pt, align=left]
            \node[text width=64pt, text height=6pt] (e-caption)
            {\scriptsize $\calE$ embeddings};
            \node[draw, very thick, white!20!violet, text=black, text width=80pt, text height=6pt, below=6pt of e-caption.west, anchor=north west] (e1)
            {\scriptsize $\ve_\mathsf{loxoprofen} \hfill =  \begin{bmatrix}-0.5 & 0.3 \end{bmatrix}$};
            \node[draw, very thick, white!20!violet, text=black, text width=80pt, text height=6pt, below=1pt of e1] (e2)
            {\scriptsize $\ve_\mathsf{COX2} \hfill =  \begin{bmatrix} 0.1 & -0.4 \end{bmatrix}$};
            \node[text width=80pt, text height=6pt, below=1pt of e2] (en)
            {\scriptsize $\ve_\mathsf{phos\text{-}acid} \hfill = \begin{bmatrix}
                -0.9 & 0.4
            \end{bmatrix}$};

            \node[text width=64pt, text height=6pt, below=4pt of en.south west, anchor=north west] (r-caption)
            {\scriptsize $\calR$ embeddings};
            \node[text width=80pt, text height=6pt, below=4pt of r-caption.west, anchor=north west] (r1)
            {\scriptsize $\vw_\mathsf{inhibits} \hfill = \begin{bmatrix} 0.9 & -0.6 \end{bmatrix}$};
            \node[draw, very thick, white!20!violet, text=black, text width=80pt, text height=6pt, below=1pt of r1] (r2)
            {\scriptsize $\vw_\mathsf{interacts} \hfill = \begin{bmatrix} 0.8 &\quad 0.2 \end{bmatrix}$};
            \node[text width=136pt, text height=6pt, right=24pt of e-caption.north east, anchor=north west, align=center] (score-function)
            {\scriptsize $\phi_\TuckER(\mathsf{loxprofen}, \mathsf{interacts}, \mathsf{COX2})$};
            \node[below=-1pt of score-function.south, anchor=north] (eval-circuit)
            {\includegraphics[scale=0.36]{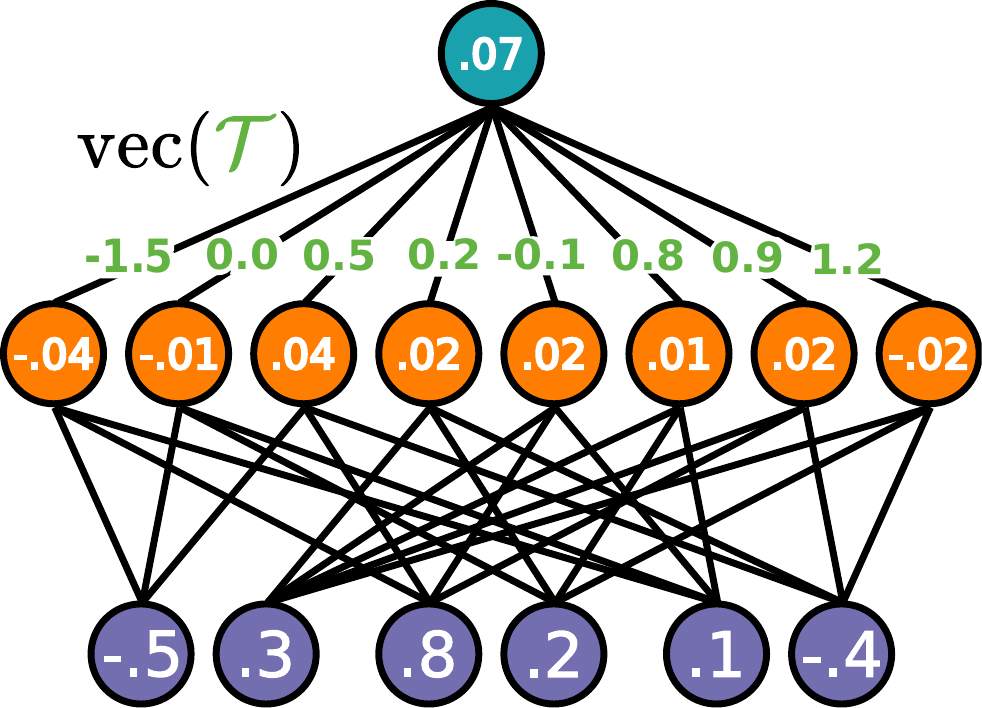}};

            \node[text width=34pt, text height=6pt, anchor=north, align=center, below=-4pt of eval-circuit.south] (p-embedding)
            {\scriptsize $\vw_\mathsf{interacts}$};

            \node[text width=34pt, text height=6pt, anchor=west, align=left, right=2pt of p-embedding] (o-embedding)
            {\scriptsize $\ve_\mathsf{COX2}$};

            \node[text width=34pt, text height=6pt, anchor=east, align=right, left=0pt of p-embedding] (s-embedding)
            {\scriptsize $\ve_\mathsf{loxoprofen}$};
        \end{tikzpicture}
        \caption{}
        \label{fig:tucker-eval-output}
    \end{subfigure}
    \caption{\textbf{Evaluation of circuit representations of score functions as in neural networks.}
    Feed-forward evaluation of the \TuckER score function as a circuit over 2-dimensional embeddings and parameterised by the core tensor $\calT$ (see proof of \cref{prop:kge-models-as-circuits} below) (a).
    Given a triple $(\mathsf{loxoprofen},\mathsf{interacts},\mathsf{COX2})$, the input units (\cref{defn:circuit}) map subject, predicate and object to their embedding entries (in violet boxes).
    Then, the circuit is evaluated similarly to neural networks: the products (in orange) are evaluated before the weighted sum (in blue), which is parameterised by the core tensor values (in green) (b).
    The output of the circuit is the score of the input triple.
    }
    \label{fig:tucker-eval}
\end{figure}

\begin{reprop}[Score functions of KGE models as circuits]{prop:kge-models-as-circuits}
    The computational graphs of the score functions $\phi$ of \CP, \RESCAL, \TuckER and \ComplEx are smooth and decomposable circuits over $\vX=\{S,R,O\}$, whose evaluation cost is $\cost(\phi)\in\Theta(|\phi|)$, where $|\phi|$ denotes the number of edges in the circuit, also called its size. For example, the size of the circuit for \CP is $|\phi_\CP|\in\calO(d)$.
\end{reprop}
\begin{proof}
    We present the proof by construction for \TuckER \citep{balazevic2019tucker},
    as \CP \citep{lacroix2018cp-kbc}, \RESCAL \citep{nickel2011rescal} and \ComplEx \citep{trouillon2016complex} define score functions that are a specialisation of it \cite{balazevic2019tucker} (see below).
    Given a triple $(s,r,o)\in\calE\times\calR\times\calE$, the \TuckER score function computes
    \begin{equation}
        \label{eq:tucker-score-function}
        \phi_\TuckER(s,r,o) = \calT\times_1 \ve_s\times_2 \vw_r\times_3 \ve_o = \sum\nolimits_{i=1}^{d_e} \sum\nolimits_{j=1}^{d_r} \sum\nolimits_{k=1}^{d_e} \tau_{ijk}e_{si}w_{rj}e_{ok}
    \end{equation}
    where $\calT\in\bbR^{d_e\times d_r\times d_e}$ denotes the core tensor, $\times_n$ denotes the tensor product along the $n$-th mode, and $d_e,d_r$ denote respectively the embedding sizes of entities and predicates (which might not be equal).
    To see how this parametrization generalises that of \CP, \RESCAL and \ComplEx, consider for example the score function of \CP on $d$-dimensional embeddings.
    It can be obtained by (i) setting the core tensor $\calT$ to be a \emph{diagonal tensor} having ones on the superdiagonal, and (ii) having two distinct embedding instances for each entity that are used depending on their role (either subject or object) in a triple.
    The embeddings $\ve_s,\ve_o\in\bbR^{d_e}$ (resp. $\vw_r\in\bbR^{d_r}$) are rows of the matrix $\vE\in\bbR^{|\calE|\times d_e}$ (resp. $\vW\in\bbR^{|\calR|\times d_r}$), which associates an embedding to each entity (resp. predicate).

    \shortparagraph{Constructing the circuit.}
    For the construction of the equivalent circuit it suffices to (i) create an input unit for each $i$-th entry of an embedding for subjects, predicates and objects, as to implement a look-up table that computes the corresponding embedding value for an entity or predicate, and (ii) transform the tensor multiplications into corresponding sum and product units.
    We start by introducing the input units $l^S_i$, $l^R_j$ and $l^O_k$ for $1\leq i,k\leq d_e$ and $1\leq j\leq d_r$ as parametric mappers over variables $S$, $R$ and $O$, respectively.
    The input units $l^S_i$ and $l^O_k$ (resp. $l^R_j$) are parameterised by the matrix $\vE$ (resp. $\vW$) such that $l^S_i(s;\vE) = e_{si}$ and $l^O_i(o;\vE) = e_{oi}$ for some $s,o\in\calE$ (resp. $l^R_j(r;\vW) = w_{ri}$ for some $r\in\calR$).
    To encode the tensor products in \cref{eq:tucker-score-function}
    we introduce $d_e^2\cdot d_r$ product units $\phi_{ijk}$, each of them computing the product of a combination of the outputs of input units.
    \begin{equation*}
        \phi_{ijk}(s,r,o) = l^S_i(s)\cdot l^R_j(r)\cdot l^O_k(o)
    \end{equation*}
    Finally, a sum unit $\phi_\text{out}$ parameterised by the core tensor $\calT\in\bbR^{d_e\times d_r\times d_r}$ computes a weighted summation of the outputs given by the product units, i.e.,
    \begin{equation*}
        \phi_\text{out}(s,r,o) = \!\!\!\!\!\!\sum_{\substack{(i,j,k)\in \\ [d_e]\times[d_r]\times [d_e]}}\!\!\!\!\!\! \tau_{ijk}\cdot \phi_{ijk}(s,r,o)
    \end{equation*}
    where $[d]$ denotes the set $\{1,\ldots,d\}$ and $\tau_{ijk}$ is the $(i,j,k)$-th entry of $\calT$.
    We now observe that the constructed circuit $\phi_\text{out}$ encodes the \TuckER score function (\cref{eq:tucker-score-function}), as $\phi_\TuckER(s,r,o) = \phi_\text{out}(s,r,o)$ for any input triple $(s,r,o)\in\calE\times\calR\times\calE$.

    \shortparagraph{Circuit evaluation and properties.}
    Evaluating the score function of \TuckER corresponds to performing a feed-forward pass of its circuit representation, where each computational unit is evaluated once, as we illustrate in \cref{fig:tucker-eval}.
    As such, the cost of evaluating the score function is proportional to the size of its circuit representation, i.e., $\cost(\phi) \in \Theta(|\phi|)$ where $|\phi| \in \calO(d_e^2\cdot d_r)$ is the number of edges.
    In \cref{tab:base-circuit-size} we show how the sizes of the circuit representation of the other score functions increases with respect to the embedding size.
    Finally, since each product unit $\phi_{ijk}$ is defined on the same scope (see \cref{defn:circuit}) $\{S,R,O\}$ and fully decompose it into its inputs (i.e., into $\{S\},\{R\},\{O\}$), and the inputs of the sum unit $\phi_\text{out}$ are all defined over the same scope, we have that the circuit satisfies smoothness and decomposability (\cref{defn:smoothness-decomposability}).
\end{proof}

\begin{table}[t!]
    \centering
    \tablesize
    \caption{\textbf{Score functions as compact circuits.} Asymptotic size of circuits encoding the score functions of \CP, \RESCAL, \ComplEx and \TuckER, with respect to the embedding size. For \TuckER, $d_e$ and $d_r$ denote the embedding sizes for entities and predicates, respectively.\\}
    \label{tab:base-circuit-size}
    \begin{tabular}{cc@{\qquad\qquad}cc}
        \toprule
        \textbf{KGE Model} & \textbf{Circuit Size} & \textbf{KGE Model} & \textbf{Circuit Size} \\
        \midrule
        \CP      & $\calO(d)$ & \RESCAL & $\calO(d^2)$ \\
        \ComplEx & $\calO(d)$ & \TuckER & $\calO(d_e^2\cdot d_r)$ \\
        \bottomrule
    \end{tabular}
\end{table}

Furthermore, in \cref{lem:omni-compatibility-kge-gecks} we show that the circuit representations of \CP, \RESCAL, \TuckER and \ComplEx and the proposed \OurModels (\cref{sec:casting-kge-models-to-pcs}) satisfy a structural property known as \emph{omni-compatibility} (see \cref{defn:omni-compatibility}).
In a nutshell, the score functions of the aforementioned KGE models and \OurModels are circuits that fully decompose their scope $\{S,R,O\}$ into $(\{S\}, \{R\}, \{O\})$.
The satisfaction of this property will be useful to prove both \cref{thm:marginalisation-squared-circuits} and \cref{thm:tractable-integration-knowledge} later in this appendix.

\begin{alem}[KGE models and derived \OurModels are omni-compatible]
    \label{lem:omni-compatibility-kge-gecks}
    The circuit representation of the score functions of \CP, \RESCAL, \TuckER, \ComplEx and their \OurModels counterparts obtained by non-negative restriction (\cref{sec:non-negative-restriction}) or squaring (\cref{sec:non-monotonic-squaring}) are omni-compatible (see \cref{defn:omni-compatibility}).
\end{alem}
\begin{proof}
To begin,
we note that to comply with \cref{defn:omni-compatibility}
every omni-compatible circuit shall contain product units that fully factorise over their scope.
In other words, for every product unit $n$ with scope $\scope(n)=\vX$, its scope shall decompose as $(\{X_1\},\{X_2\},\ldots,\{X_{|\scope(n)|}\})$.
To see why, consider a circuit $\phi$ with a product unit $n$ whose scope is decomposed as $\scope(n)=(\vX, \vY)$.
It is easy to construct another circuit $\phi^{\prime}$ that is not compatible with $\phi$ by having a product unit $m$ with scope $\scope(m)=\scope(n)$ decomposed in a way that it cannot be rearranged by introducing additional decomposable product units (\cref{defn:smoothness-decomposability}), e.g., $\scope(m)=(\vZ, \vW)$ with $\vZ\cap\vX\neq\emptyset$ and $\vW\cap\vX\neq\emptyset$.
As such, every omni-compatible circuit over $\vX$ must be representable in the
form $\sum_{i=1}^{N}\theta_{i}\prod_{k=1}^{|\vX|}l_{ik}(X_k)$ without any increase in its size.

Now, it is easy to verify that the circuit representations of \CP, \RESCAL, \TuckER and \ComplEx follow the above form, with a different number $N$ of product units feeding the single sum unit,
but each one decomposing its scope $\{S,R,O\}$ into $(\{S\}, \{R\}, \{O\})$ (see \cref{fig:kge-circuits}).
From this it
follows that \NNegCP, \NNegRESCAL, \NNegTuckER and \NNegComplEx are omni-compatible as well, as they share the same structure of their energy-based counterpart, while just enforcing non-negative activations
via reparametrisation
(see \cref{sec:non-negative-restriction}).

Finally, we note that \SquaredCP, \SquaredRESCAL, \SquaredTuckER and \SquaredComplEx are still omni-compatible because the square operation yields the following fully-factorised representation: $(\sum_{i=1}^{N}\theta_{i}\prod_{k=1}^{|\vX|}l_{ik}(X_k))^2$ $=\sum_{i=1}^{N}\sum_{j=1}^{N}\theta_{i}\theta_{j}\prod_{k=1}^{|\vX|}l_{ik}(X_k)\prod_{k=1}^{|\vX|}l_{jk}(X_k)$
which can be easily rewritten as $\sum_{h=1}^{N^2}\omega_{h}\prod_{k=1}^{|\vX|}l_{hk}(X_k)$
where now $h$ ranges over the Cartesian product of $i\in[N]$ and $ j\in[N]$, $\omega_h$ is the product of $\theta_i\theta_j$ and  $l_{hk}$ is a new input unit that encodes $l_{ik}(X_k)l_{jk}(X_k)$ for a certain variable index $k$.
\end{proof}

\subsection{Efficient Summations over Circuits}
\label{app:efficient-summations}

\begin{reprop}[Efficient Summations]{prop:efficient-summations}
    Let $\phi$ be a smooth and decomposable circuit over $\vX=\{S,R,O\}$ that encodes the score function of a KGE model.
    The sum $\sum_{s\in\calE} \sum_{r\in\calR} \sum_{o\in\calE} \phi(s,r,o)$ or any other summation over subjects, predicates or objects can be computed in time $\calO((|\calE| + |\calR|) \cdot |\phi|)$.
\end{reprop}
\begin{proof}
    A proof for the computation of marginal probabilities in smooth and decomposable \emph{probabilistic circuits} (PCs) defined over discrete variables in linear time with respect to their size can be found in \citep{choi2020pc}.
    This proof also applies for computing summations in smooth and decomposable circuits that do not necessarily corresponds to marginal probabilities \citep{vergari2021compositional}.
    The satisfaction of smoothness and decomposability (\cref{defn:smoothness-decomposability}) in a circuit $\phi$ permits to push outer summations inside the computational graph until input units are reached, where summations are actually performed independently and on smaller sets of variables (i.e., $\{S\}$, $\{R\}$, $\{O\}$ in our case), and then to evaluate the circuit only once.

    Here we take into account the computational cost of summing over each input unit (see proof of \cref{prop:kge-models-as-circuits}), which is $\calO(|\calE|)$ (resp. $\calO(|\calR|)$) for those defined on variables $S,O$ (resp. $R$).
    Since the size of the circuit $|\phi|$ must be at least the number of input units, we retrieve that the overall complexity for computing summations as stated in the proposition is $\calO((|\calE| + |\calR|) \cdot |\phi|)$.

    As an example, consider the \CP score function computing $\phi_\CP(s,r,o) = \langle \ve_s,\vw_r,\ve_o \rangle$ for some triple $(s,r,o)$ and embeddings $\ve_s,\vw_r,\ve_o\in\bbR^d$.
    We can compute $\sum_{s\in\calE} \sum_{r\in\calR} \sum_{o\in\calE} \phi_\CP(s,r,o)$ by pushing the outer summations inside the trilinear product, i.e., by computing it as $\langle \sum_{s\in\calE} \ve_s, \sum_{r\in\calR} \vw_r, \sum_{o\in\calE} \ve_o \rangle$, which requires time $\calO((|\calE| + |\calR|)\cdot d)$.
\end{proof}

\subsection{Efficient Summations over Squared Circuits}
\label{app:marginalisation-squared-circuits}

\begin{rethm}[Efficient summations of squared \OurModels]{thm:marginalisation-squared-circuits}
    Performing summations as stated in \cref{prop:efficient-summations} on \SquaredCP, \SquaredRESCAL, \SquaredTuckER and \SquaredComplEx can be done in time $\calO((|\calE| + |\calR|)\cdot |\phi|^2)$.
\end{rethm}
\begin{proof}
    In \cref{lem:omni-compatibility-kge-gecks} we showed that the circuit representations $\phi$ of \CP, \RESCAL, \TuckER and \ComplEx are omni-compatible (see \cref{defn:omni-compatibility}).
    As a consequence, $\phi$ is compatible (see \cref{defn:compatibility}) with itself.
    Therefore, \cref{prop:tractable-product} ensures that we can construct the product circuit $\phi \cdot \phi$ (i.e., $\phi^2$) as a smooth and decomposable circuit having size $\calO(|\phi|^2)$ in time $\calO(|\phi|^2)$.
    Since $\phi^2$ is still smooth and decomposable, \cref{prop:efficient-summations} guarantees that we can perform summations in time $\calO((|\calE| + |\calR|)\cdot |\phi|^2)$.
\end{proof}

\subsection{Circuits encoding Domain Constraints}
\label{app:circuit-domain-constraints}

In \cref{defn:support-determinism} we introduce the concepts of \emph{support} and \emph{determinism}, whose definition is useful to describe \emph{constraint circuits} in \cref{defn:constraint-circuit}.

\begin{adefn}[Support and Determinism \citep{choi2020pc,vergari2021compositional}]
    \label{defn:support-determinism}
    In a circuit the \emph{support} of a computational unit $n$ over variables $\vX$ computing $\phi_n(\vX)$ is defined as the set of value assignments to variables in $\vX$ such that the output of $n$ is non-zero, i.e., $\supp(n) = \{\vx\in\val(\vX) \mid \phi_n(\vX) \neq 0\}$.
    A sum unit $n$ is \emph{deterministic} if its inputs have disjoint \emph{supports}, i.e., $\forall i,j\in\inscope(n), i \neq j\colon \supp(i)\cap\supp(j) = \emptyset$.
\end{adefn}

\begin{adefn}[Constraint Circuit \citep{ahmed2022semantic}]
    \label{defn:constraint-circuit}
    Given a propositional logic formula $K$, a \emph{constraint circuit} $c_K$ is a smooth and decomposable PC over variables $\vX$ with \emph{deterministic} sum units (\cref{defn:support-determinism}) and indicator functions as input units, such that $c_K(\vx)=\Ind{\vx\models K}$ for any $\vx\in\val(\vX)$.
\end{adefn}

In general, we can compile any propositional logic formula into a constraint circuit (\cref{defn:constraint-circuit}) by leveraging knowledge compilation techniques \citep{darwiche2002knowledge,darwiche2011sdd,oztok2015top-down-sdds}.
For domain constraints (\cref{defn:domain-constraint}) this compilation process is straightforward, as we detail in the following proposition and proof.

\begin{aprop}[Circuit encoding domain constraints]
    \label{prop:domain-constraints-circuit}
    Let $K=K_{r_1}\lor\ldots\lor K_{r_m}$ be a disjunction of domain constraints defined over a set of predicates $\calR=\{r_1,\ldots, r_m\}$ and a set of entities $\calE$ (\cref{defn:domain-constraint}).
    We can compile $K$ into a constraint circuit $c_K$ (\cref{defn:constraint-circuit}) defined over variables $\vX=\{S,R,O\}$ having size $\calO(|\calE| \cdot |\calR|)$ in the worst case and $\calO(|\calE| + |\calR|)$ in the best case.
\end{aprop}
\begin{proof}
    Let $K=K_{r_1}\lor\ldots\lor K_{r_m}$ be a disjunction of domain constraints (\cref{defn:domain-constraint}) where
    \begin{equation*}
        K_r \equiv S \in \kappa_S(r) \land R = r \land O \in \kappa_O(r) \equiv (\lor_{u\in\kappa_S(r)} S = u) \land R = r \land (\lor_{v\in\kappa_O(r)} O = v).
    \end{equation*}
    Note that the disjunctions in $K$ are deterministic, i.e., only one of their argument can be true at the same time.
    This enables us to construct the constraint circuit $c_K$ such that $c_K(s,r,o) = \Ind{(s,r,o)\models K}$ for any triple
    by simply replacing conjunctions and disjunctions with product and sum units, respectively.
    Note that $c_K$ is indeed smooth and decomposable (\cref{defn:smoothness-decomposability}), as the inputs of the sum units are product units having scope $\{S,R,O\}$ that are fully factorised into $(\{S\}, \{R\}, \{O\})$.
    Moreover, $K$ is a disjunction of $|\calR|$ conjunctive formulae having $\calO(|\calE|)$ terms, and therefore $|c_K|=\calO(|\calE| \cdot |\calR|)$ in the worst case.
    In the best case of every predicate sharing the same subject and object domains $\kappa_S,\kappa_O\subseteq\calE$, we can simplify $K$ into a conjunction of three disjunctive expressions, i.e.,
    \begin{equation*}
        K \equiv (\lor_{u\in\kappa_S} S=u) \land (\lor_{r\in\calR} R=r) \land (\lor_{v\in\kappa_O} O=v)
    \end{equation*}
    that can be easily compiled into a constraint circuit $c_K$ having size $\calO(|\calE|+|\calR|)$, by again noticing that disjunctions are deterministic.
    In real-world KGs like \textsf{ogbl-biokg} \citep{hu2020ogb} several predicates share the same subject and object domains, and this permits to have much smaller constraint circuits.
\end{proof}

\subsection{Efficient Integration of Domain Knowledge in \OurModels}
\label{app:tractable-integration-knowledge}

\begin{rethm}[Tractable integration of constraints in \OurModels]{thm:tractable-integration-knowledge}
    Let $c_K$ be a constraint circuit encoding a logical constraint $K$ over variables $\{S,R,O\}$.
    Then exactly computing
    the partition function $Z_K$
    of
    the product
    $\phi_\mathsf{pc}(s,r,o) \cdot c_K(s,r,o) \propto p_K(s,r,o)$
    for any \OurModel $\phi_\mathsf{pc}$ derived from \CP, \RESCAL, \TuckER or \ComplEx (\cref{sec:casting-kge-models-to-pcs})
    can be done in time $\calO((|\calE| + |\calR|)\cdot |\phi_\mathsf{pc}|\cdot |c_K|)$.
\end{rethm}
\begin{proof}
    In \cref{lem:omni-compatibility-kge-gecks} we showed that the \OurModels $\phi_\mathsf{pc}$ derived from \CP, \RESCAL, \TuckER and \ComplEx via non-negative restriction (\cref{sec:non-negative-restriction}) or squaring (\cref{sec:non-monotonic-squaring}) are omni-compatible (see \cref{defn:omni-compatibility}).
    As a consequence, $\phi_\mathsf{pc}$ is always compatible with $c_K$ regardless of the encoded logical constraint $K$, since constraint circuits are by definition smooth and decomposable (\cref{defn:constraint-circuit}).
    By applying \cref{prop:tractable-product}, we retrieve that we can construct $\phi_\mathsf{pc}\cdot c_K$ as a smooth and decomposable circuit of size $\calO(|\phi_\mathsf{pc}|\cdot |c_K|)$ and in time $\calO(|\phi_\mathsf{pc}|\cdot |c_K|)$.
    As the resulting product circuit is smooth and decomposable, \cref{prop:efficient-summations} guarantees that we can compute its partition function $Z_K = \sum_{s\in\calE} \sum_{r\in\calR} \sum_{o\in\calE} (\phi_\mathsf{pc}(s,r,o) \cdot c_K(s,r,o))$ in time $\calO((|\calE| + |\calR|)\cdot |\phi_\mathsf{pc}|\cdot |c_K|)$.
\end{proof}

\section{Circuits}
\label{app:circuits}

\subsection{Tractable Product of Circuits}
\label{app:tractable-product}

In this section, we provide the formal definition of \emph{compatibility} (\cref{defn:compatibility}) and \emph{omni-compatibility} (\cref{defn:omni-compatibility}), as stated by \citet{vergari2021compositional}.
Given two compatible circuits, \cref{prop:tractable-product} guarantees that we can represent their product as a smooth and decomposable circuit efficiently.

\begin{adefn}[Compatibility]
    \label{defn:compatibility}
    Two circuits $\phi,\phi'$ over variables $\vX$ are \emph{compatible} if (1) they are smooth and decomposable, and (2) any pair of product units $n\in\phi,m\in\phi'$ having the same scope can be rearranged into binary products that are mutually compatible and decompose their scope in the same way, i.e., $(\scope(n)=\scope(m)) \implies (\scope(n_i) = \scope(m_i),\ n_i\text{ and }m_i\text{ are compatible})$ for some rearrangements of the inputs of $n$ (resp. $m$) into $n_1,n_2$ (resp. $m_1,m_2$).
\end{adefn}

\begin{adefn}[Omni-compatibility]
    \label{defn:omni-compatibility}
    A circuit $\phi$ over variables $\vX$ is \emph{omni-compatible} if it is compatible with any smooth and decomposable circuit over $\vX$.
\end{adefn}

\begin{aprop}[Tractable product of circuits]
    \label{prop:tractable-product}
    Let $\phi,\phi'$ be two compatible (\cref{defn:compatibility}) circuits.
    We can represent the product circuit $\phi\cdot\phi'$ computing the product of the outputs of $\phi$ and $\phi'$ as a smooth and decomposable circuit having size $\calO(|\phi|\cdot|\phi'|)$ in time $\calO(|\phi|\cdot|\phi'|)$.
    Moreover, if both $\phi$ and $\phi'$ are omni-compatible (\cref{defn:omni-compatibility}), then also the product circuit $\phi\cdot\phi'$ is omni-compatible. 
\end{aprop}

\cref{prop:tractable-product} allows us to compute the partition function and any other marginal probability in \OurModels obtained via squaring efficiently (see \cref{sec:non-monotonic-squaring} and \cref{thm:marginalisation-squared-circuits}).
In addition, \cref{prop:tractable-product} is a crucial theoretical result that allows us to inject logical constraints in \OurModels in a way that enable computing the partition function exactly and efficiently (see \cref{sec:injection-logical-constraints} and \cref{thm:tractable-integration-knowledge}).

\section{From KGE Models to PCs}
\label{app:from-kge-to-pcs}

\subsection{Interpreting Non-negative Embedding Values}
\label{app:interpreting-non-negative}

In \cref{fig:interpreting-non-negative} we interpret the embedding values of \OurModels obtained via non-negative restriction -- \NNegCP, \NNegRESCAL, \NNegTuckER, \NNegComplEx -- (\cref{sec:non-negative-restriction}) as the parameters of unnormalised categorical distributions over entities (elements in $\calE$) or predicates (elements in $\calR$).

\begin{figure}[H]
    \centering
    \begin{tikzpicture}[inner sep=0pt, align=left]
        \node[text width=34pt, text height=6pt, rotate=45, anchor=south west, align=right] (e1)
        {\footnotesize $\ve_\mathsf{loxoprofen}$};
        \node[text width=34pt, text height=6pt, right=20pt of e1.south west, rotate=45, anchor=south west, align=right] (e2)
        {\footnotesize $\ve_\mathsf{COX2}$};
        \node[text width=17pt, text height=6pt, right=32pt of e2.south, rotate=90, anchor=south west, align=center] (e-dotdotdot)
        {$\vdots$};
        \node[text width=34pt, text height=6pt, right=44pt of e2.south west, rotate=45, anchor=south west, align=right] (en)
        {\footnotesize $\ve_\mathsf{phos\text{-}acid}$};
        \node[text width=80pt, text height=8pt, right=20pt of en.north, align=center] (e-caption)
        {\small $\calE$ embeddings};

        \node[text width=16pt, text height=20pt, draw, black, fill=violet, thick, above=4pt of e1.east, anchor=south] (cat1-e1)
        {};
        \node[text width=12pt, text height=4pt, above=3pt of cat1-e1.north, anchor=south] (e1-emb1)
        {\footnotesize $0.5$};
        \node[text width=16pt, text height=4pt, draw, black, fill=violet, thick, right=0pt of cat1-e1.south east, anchor=south west] (cat1-e2)
        {};
        \node[text width=12pt, text height=4pt, above=3pt of cat1-e2.north, anchor=south] (e2-emb1)
        {\footnotesize $0.1$};
        \node[text width=34pt, text height=5pt, above=12pt of e-dotdotdot.east, rotate=90, anchor=west] (e-dotdotdot-emb1)
        {$\vdots$};
        \node[text width=16pt, text height=36pt, draw, black, fill=violet, thick, right=30pt of cat1-e2.south east, anchor=south west] (cat1-en)
        {};
        \node[text width=12pt, text height=4pt, above=3pt of cat1-en.north, anchor=south] (en-emb1)
        {\footnotesize $0.9$};
        \node[text width=96, text height=6pt, right=4pt of cat1-en.east] (cat1-formula-e)
        {\scriptsize $\mathrm{Cat}_1 \left( p_u = e_{u1} / \sum\limits_{u\in\calE} e_{u1} \right)$};

        \node[text width=16pt, text height=12pt, draw, black, fill=sky, thick, above=56pt of e1.east, anchor=south] (cat2-e1)
        {};
        \node[text width=12pt, text height=4pt, above=3pt of cat2-e1.north, anchor=south] (e1-emb2)
        {\footnotesize $0.3$};
        \node[text width=16pt, text height=16pt, draw, black, fill=sky, thick, right=0pt of cat2-e1.south east, anchor=south west] (cat2-e2)
        {};
        \node[text width=12pt, text height=5pt, above=3pt of cat2-e2.north, anchor=south] (e2-emb2)
        {\footnotesize $0.4$};
        \node[text width=10pt, text height=4pt, above=10pt of e-dotdotdot-emb1.east, rotate=90, anchor=west, align=center] (e-dotdotdot-emb2)
        {$\vdots$};
        \node[text width=16pt, text height=16pt, draw, black, fill=sky, thick, right=30pt of cat2-e2.south east, anchor=south west] (cat2-en)
        {};
        \node[text width=12pt, text height=4pt, above=3pt of cat2-en.north, anchor=south] (en-emb2)
        {\footnotesize $0.4$};
        \node[text width=96, text height=6pt, right=4pt of cat2-en.east] (cat2-formula-e)
        {\scriptsize $\mathrm{Cat}_2 \left( p_u = e_{u2} / \sum\limits_{u\in\calE} e_{u2} \right )$};

        \node[text width=34pt, text height=6pt, right=190pt of e1.south west, rotate=45, anchor=south west, align=right] (r1)
        {\footnotesize $\vw_\mathsf{inhibits}$};
        \node[text width=34pt, text height=6pt, right=20pt of r1.south west, rotate=45, anchor=south west, align=right] (r2)
        {\footnotesize $\vw_\mathsf{interacts}$};
        \node[text width=17pt, text height=6pt, right=32pt of r2.south, rotate=90, anchor=south west, align=center] (r-dotdotdot)
        {$\vdots$};
        \node[text width=34pt, text height=6pt, right=44pt of r2.south west, rotate=45, anchor=south west, align=right] (rn)
        {\footnotesize $\vw_\mathsf{reacts}$};
        \node[text width=80pt, text height=8pt, right=20pt of rn.north, align=center] (r-caption)
        {\small $\calR$ embeddings};

        \node[text width=16pt, text height=12pt, draw, black, fill=purple, thick, above=4pt of r1.east, anchor=south] (cat1-r1)
        {};
        \node[text width=12pt, text height=4pt, above=3pt of cat1-r1.north, anchor=south] (r1-emb1)
        {\footnotesize $0.3$};
        \node[text width=16pt, text height=16pt, draw, black, fill=purple, thick, right=0pt of cat1-r1.south east, anchor=south west] (cat1-r2)
        {};
        \node[text width=12pt, text height=4pt, above=3pt of cat1-r2.north, anchor=south] (r2-emb1)
        {\footnotesize $0.4$};
        \node[text width=34pt, text height=5pt, above=12pt of r-dotdotdot.east, rotate=90, anchor=west] (r-dotdotdot-emb1)
        {$\vdots$};
        \node[text width=16pt, text height=16pt, draw, black, fill=purple, thick, right=30pt of cat1-r2.south east, anchor=south west] (cat1-rn)
        {};
        \node[text width=12pt, text height=4pt, above=3pt of cat1-rn.north, anchor=south] (rn-emb1)
        {\footnotesize $0.4$};
        \node[text width=102, text height=6pt, right=4pt of cat1-rn.east] (cat1-formula-r)
        {\scriptsize $\mathrm{Cat}_1 \left( p_r = w_{r1} / \sum\limits_{r\in\calR} w_{r1} \right)$};

        \node[text width=16pt, text height=24pt, draw, black, fill=gold, thick, above=56pt of r1.east, anchor=south] (cat2-r1)
        {};
        \node[text width=12pt, text height=4pt, above=3pt of cat2-r1.north, anchor=south] (r1-emb2)
        {\footnotesize $0.6$};
        \node[text width=16pt, text height=8pt, draw, black, fill=gold, thick, right=0pt of cat2-r1.south east, anchor=south west] (cat2-r2)
        {};
        \node[text width=12pt, text height=4pt, above=3pt of cat2-r2.north, anchor=south] (r2-emb2)
        {\footnotesize $0.2$};
        \node[text width=10pt, text height=5pt, above=10pt of r-dotdotdot-emb1.east, rotate=90, anchor=west, align=center] (r-dotdotdot-emb2)
        {$\vdots$};
        \node[text width=16pt, text height=12pt, draw, black, fill=gold, thick, right=30pt of cat2-r2.south east, anchor=south west] (cat2-rn)
        {};
        \node[text width=12pt, text height=4pt, above=3pt of cat2-rn.north, anchor=south] (en-emb2)
        {\footnotesize $0.3$};
        \node[text width=102, text height=6pt, right=4pt of cat2-rn.east] (cat2-formula-r)
        {\scriptsize $\mathrm{Cat}_2 \left( p_r = w_{r2} / \sum\limits_{r\in\calR} w_{r2} \right)$};
    \end{tikzpicture}
    \caption{\textbf{Non-negative embeddings parameterise categorical distributions.} 2-dimensional embeddings of \OurModels obtained via non-negative restrictions (\cref{sec:non-negative-restriction}) can be seen as the parameters of two categorical distributions over entities (left) or predicates (right) up to renormalisation.}
    \label{fig:interpreting-non-negative}
\end{figure}
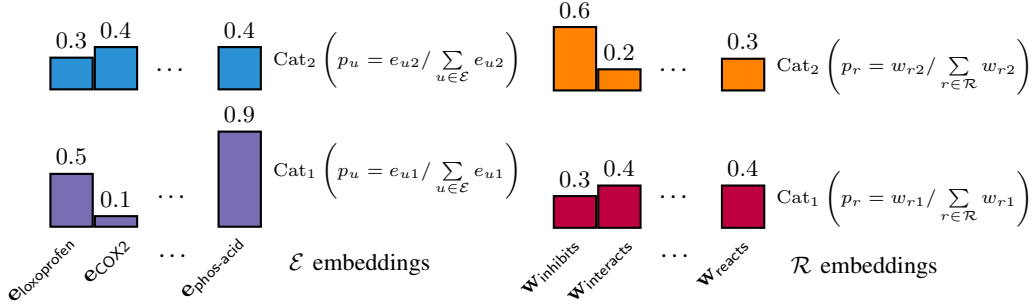

\subsection{Realising the Non-negative Restriction of \ComplEx}
\label{app:realising-monocomplex}

As anticipated in \cref{sec:non-negative-restriction}, for the \ComplEx \citep{trouillon2016complex} score function restricting the real and imaginary parts to be non-negative is not sufficient to obtain a PC due to the presence of a subtraction, as showed in the following equation.
\begin{equation}
    \label{eq:complex-score-function}
    \begin{split}
        \phi_\ComplEx(s,r,o) &= \langle\Re(\ve_s),\Re(\vw_r),\Re(\ve_o)\rangle + \langle\Im(\ve_s),\Re(\vw_r),\Im(\ve_o)\rangle \\
        &+ \langle\Re(\ve_s),\Im(\vw_r),\Im(\ve_o)\rangle - \langle\Im(\ve_s),\Im(\vw_r),\Re(\ve_o)\rangle
    \end{split}
\end{equation}
Here $\ve_s,\vw_r,\ve_o\in\bbC^d$ are the embeddings associated to the subject, predicate and object, respectively.
Under the restriction of embedding values to be non-negative, we ensure that $\phi_\ComplEx(s,r,o)\geq 0$ for any input triple by enforcing the additional constraint
\begin{equation}
    \label{eq:complex-constraint}
    \langle\Re(\ve_s),\Re(\vw_r),\Re(\ve_o)\rangle \geq \langle\Im(\ve_s),\Im(\vw_r),\Re(\ve_o)\rangle,
\end{equation}
which can be simplified into the two distinct inequalities
\begin{equation*}
    \forall u\in\calE\quad \Re(e_{ui})\geq\Im(e_{ui}) \qquad\text{and}\qquad \forall r\in\calR \quad \Re(w_{ri})\geq\Im(w_{ri}).
\end{equation*}
In other words, we want the real part of each embedding value to be always greater or equal than the corresponding imaginary part.
We implement this constraint in practice by reparametrisation of the imaginary part in function of the real part, i.e.,
\begin{align}
    \label{eq:complex-constraint-ent}
    \forall u\in\calE & \quad \Im(e_{ui}) = \Re(e_{ui}) \cdot \sigma(\theta_{ui}) \\
    \label{eq:complex-constraint-rel}
    \forall r\in\calR & \quad \Im(w_{ri}) = \Re(w_{ri}) \cdot \sigma(\gamma_{ri})
\end{align}
where $\sigma(x)=1/(1+\exp(-x))\in[0,1]$ denotes the logistic function and $\theta_{ui},\gamma_{ri}\in\bbR$ are additional parameters associated to entities $u\in\calE$ and predicates $r\in\calR$, respectively.
The reparametrisation of the imaginary parts using \cref{eq:complex-constraint-ent,eq:complex-constraint-rel} is a sufficient condition for the satisfaction of the constraint showed in \cref{eq:complex-constraint}, and also maintains the same number of learnable parameters of \ComplEx.

\subsection{Sampling from \OurModels with Non-negative Parameters}
\label{app:sampling-procedures}

\paragraph{Parameters interpretation.}
Sum units with non-negative parameters in smooth and decomposable PCs can be seen as marginalised discrete latent variables, similarly to the latent variable interpretation in mixture models \citep{poon2011sum,peharz2017latent-spn}.
That is, the non-negative parameters of a sum unit are the parameters of a (possibly unnormalised) categorical distribution over assignments to a latent variable.
For \NNegCP and \NNegRESCAL (\cref{sec:non-negative-restriction}), the non-negative parameters of the sum unit encode a uniform and unnormalised categorical distribution, as they are all fixed to $1$ (see \cref{fig:kge-circuits}).
By contrast, in \NNegTuckER these parameters are the vectorisation of the core tensor $\calT$ (see the proof of \cref{prop:kge-models-as-circuits}), and hence they are learned.
The input units of \NNegCP, \NNegRESCAL and \NNegTuckER can be interpreted as unnormalised categorical distribution over entities or predicates, as detailed in \cref{app:interpreting-non-negative}.

\shortparagraph{Sampling from \NNegCP, \NNegComplEx, \NNegTuckER.}
Thanks to the latent variable interpretation,
ancestral sampling in \NNegCP, \NNegRESCAL and \NNegTuckER can be performed by
(1) sampling an assignment to the latent variable associated to the single sum unit, i.e., one of its input branches,
(2) selecting the corresponding combination of subject-predicate-object input units, and
(3) sampling a subject, predicate and object respectively from each of the indexed unnormalized categorical distributions.

\subsection{Learning Complexity of \OurModels}
\label{app:inference-complexity}

In \cref{tab:summary-complexity} we summarise the complexities of computing the PLL and MLE objectives (\cref{eq:pll-objective,eq:mle-objective}) for KGE models and \OurModels.
Asymptotically, \OurModels manifest better time and space complexities with respect to the number of entities $|\calE|$, batch size $|B|$ and embedding size.
This makes \OurModels more efficient than traditional KGE models during training, both in time and memory (see \cref{sec:training-efficiency} and \cref{fig:pll-scaling}).

\begin{table}[H]
    \tablesize
    \centering
    \caption{\textbf{Summary of complexities for exactly computing the PLL and MLE objectives.} Time and space complexity of computing $\log p(o\mid s,r)$ and the partition function $Z$. These complexities are respectively lower bounds of the complexities of computing the PLL and MLE objectives, as we have that $|\calE|\gg |\calR|$ for large real-world KGs. For \CP, \RESCAL and \ComplEx and \OurModels derived from them, $d$ denotes the size of both entity and predicate embeddings. For \TuckER and \OurModels derived from it, $d_e$ and $d_r$ denote the embedding sizes for entities and predicates, respectively.\\
    }
    \label{tab:summary-complexity}
    \setlength{\tabcolsep}{4pt}
    \begin{tabular}{lllll}
    \toprule
    \multirow{2}{*}{\textbf{Model}} & \multicolumn{2}{c}{\textbf{Complexity of} $\log p(o\mid s,r)$} & \multicolumn{2}{c}{\textbf{Complexity of} $Z$} \\
    \cmidrule(lr){2-3}
    \cmidrule(lr){4-5}
    & \multicolumn{1}{c}{\textbf{Time}} & \multicolumn{1}{c}{\textbf{Space}} & \multicolumn{1}{c}{\textbf{Time}} & \multicolumn{1}{c}{\textbf{Space}} \\
    \cmidrule(lr){2-2}\cmidrule(lr){3-3}
    \cmidrule(lr){4-4}\cmidrule(lr){5-5}
    \CP           & $\calO(|\calE|\cdot |B|\cdot d)$   & $\calO(|\calE|\cdot |B|)$ & $\calO(|\calE|^2\cdot |\calR|\cdot d)$   & $\calO(d)$ \\
    \RESCAL       & $\calO(|\calE|\cdot |B| \cdot d + |B|\cdot d^2)$   & $\calO(|\calE|\cdot |B|)$ & $\calO(|\calE|^2\cdot |\calR|\cdot d^2)$ & $\calO(d^2)$ \\
    \TuckER       & $\calO(|\calE|\cdot |B|\cdot d_e + |B|\cdot d_e^2\cdot d_r)$ & $\calO(|\calE|\cdot |B|)$ & $\calO(|\calE|^2\cdot |\calR|\cdot d_e^2\cdot d_r)$ & $\calO(d_e^2\cdot d_r)$ \\
    \ComplEx      & $\calO(|\calE|\cdot |B|\cdot d)$   & $\calO(|\calE|\cdot |B|)$ & $\calO(|\calE|^2\cdot |\calR|\cdot d)$   & $\calO(d)$ \\
    \midrule
    \NNegCP       & $\calO((|\calE| + |B|)\cdot d)$ & $\calO(|B|\cdot d)$ & $\calO((|\calE| + |\calR|)\cdot d)$ & $\calO(d)$ \\
    \NNegRESCAL   & $\calO((|\calE| + |B|\cdot d)\cdot d)$ & $\calO(|B|\cdot d^2)$ & $\calO((|\calE| + |\calR|\cdot d) \cdot d)$ & $\calO(d^2)$ \\
    \NNegTuckER   & $\calO((|\calE| + |B|\cdot d_e\cdot d_r)\cdot d_e)$ & $\calO(|B|\cdot d_e\cdot d_r)$ & $\calO(|\calE|\cdot d_e + |\calR|\cdot d_r + d_e^2\cdot d_r)$ & $\calO(d_e^2\cdot d_r)$ \\
    \NNegComplEx  & $\calO((|\calE| + |B|)\cdot d)$    & $\calO(|B|\cdot d)$   & $\calO((|\calE| + |\calR|)\cdot d)$ & $\calO(d)$ \\
    \midrule
    \SquaredCP       & $\calO((|\calE| + |B|)\cdot d^2)$ & $\calO(|B|\cdot d)$ & $\calO((|\calE| + |\calR|)\cdot d^2)$ & $\calO(d^2)$ \\
    \SquaredRESCAL   & $\calO((|\calE| + |B|)\cdot d^2)$ & $\calO(|B|\cdot d^2)$ & $\calO((|\calE| + |\calR|\cdot d)\cdot d^2)$ & $\calO(|\calR|\cdot d^2)$ \\
    \SquaredTuckER   & $\calO((|\calE| + |B|\cdot d_r)\cdot d_e^2)$ & $\calO(|B|\cdot d_e\cdot d_r)$ & $\calO(|\calE|\cdot d_e^2 + |\calR|\cdot d_r^2 + d_e^2\cdot d_r)$ & $\calO(d_e^2\cdot d_r)$ \\
    \SquaredComplEx  & $\calO((|\calE| + |B|)\cdot d^2)$ & $\calO(|B|\cdot d)$ & $\calO((|\calE| + |\calR|)\cdot d^2)$ & $\calO(d^2)$ \\
    \bottomrule
    \end{tabular}
\end{table}

\subsubsection{Computing the Partition Function}
\label{app:partition-function-squared-circuits}

In this section we derive the computational complexity of computing the partition function for \OurModels obtained via squaring (\cref{sec:non-monotonic-squaring}).
For a summary of these complexities, see \cref{tab:summary-complexity}.

\shortparagraph{\SquaredCP and \SquaredComplEx.}
Here we derive the partition function of \SquaredCP.
For \SquaredComplEx the derivation is similar, as the score function of \ComplEx can be written in terms of trilinear products just like \CP (see \cref{eq:complex-score-function}).
The score function $\phi_\SquaredCP$ encoded by \SquaredCP can be written as
\begin{equation*}
    \phi_\SquaredCP(s,r,o) = \langle\ve_s,\vw_r,\ve_o \rangle^2 = \sum_{i=1}^d \sum_{j=1}^d e_{si}e_{sj}w_{ri}w_{rj}e_{oi}e_{oj}
\end{equation*}
where $\ve_s,\ve_o\in\bbR^d$ (resp. $\vw_r\in\bbR^d$) are rows of the matrices $\vU,\vV\in\bbR^{|\calE|\times d}$ (resp. $\vW\in\bbR^{|\calR|\times d}$), which associate to each entity (resp. predicate) a vector.
By leveraging the \href{https://en.wikipedia.org/wiki/Einstein_notation}{\emph{einsum notation}} for brevity, 
the partition function of $\phi_\SquaredCP$ can be written as
\begin{align*}
    Z &= \sum_{s\in\calE}\sum_{r\in\calR}\sum_{o\in\calE} \phi_\SquaredCP(s,r,o) = \sum_{i=1}^d \sum_{j=1}^d \left( \sum_{s\in\calE} e_{si}e_{sj} \right) \left( \sum_{r\in\calR} w_{ri}w_{rj} \right) \left( \sum_{o\in\calE} e_{oi}e_{oj} \right) \\
    &= \vU'_{ij} \vW'_{ij} \vV'_{ij}
\end{align*}
where $\vU' = \vU^\top\vU$, $\vW' = \vW^\top\vW$ and $\vV' = \vV^\top\vV$ are $d\times d$ matrices.
With the simplest algorithm for matrix multiplication, we recover that computing $Z$ requires time $\calO(|\calE|\cdot d^2 + |\calR|\cdot d^2)$ and additional space $\calO(d^2)$.

\shortparagraph{\SquaredRESCAL.}
The score function $\phi_\SquaredRESCAL$ encoded by \SquaredRESCAL can be written as
\begin{equation*}
    \phi_\SquaredRESCAL(s,r,o) = \left( \ve_s^\top\vW_r\ve_o \right)^2 = \!\!\!\! \sum_{(i,j,k,l)\in [d]^4} \!\!\!\! e_{si}e_{sk}w_{rij}w_{rkl}e_{oj}e_{ol}
\end{equation*}
where $[d]$ denotes the set $\{1,\ldots,d\}$, $\ve_s,\ve_o\in\bbR^d$ are rows of the matrix $\vE\in\bbR^{|\calE|\times d}$ and $\vW_r\in\bbR^{d\times d}$ are slices along the first mode of the tensor $\calW\in\bbR^{|\calR|\times d\times d}$, which consists of stacked matrix embeddings associated to predicates.
The partition function of $\phi_\SquaredRESCAL$ can be written as
\begin{align*}
    Z &= \sum_{s\in\calE}\sum_{r\in\calR}\sum_{o\in\calE} \phi_\SquaredRESCAL(s,r,o) = \!\!\!\! \sum_{(i,j,k,l)\in [d]^4} \!\! \left( \sum_{s\in\calE} e_{si}e_{sk} \right) \left( \sum_{r\in\calR} w_{rij}w_{rkl} \right) \left( \sum_{o\in\calE} e_{oj}e_{ol} \right) \\
    &= \vE'_{ik} \calW_{rij} \calW_{rkl} \vE'_{jl} \addtocounter{equation}{1}\tag{\theequation} \label{eq:pf-sqrescal-einsum}
\end{align*}
where $\vE' = \vE^\top \vE \in \bbR^{d\times d}$.
The complexity of computing $Z$ depends on the order of tensor contractions in the einsum operation showed in \cref{eq:pf-sqrescal-einsum}.
By optimising the order of tensor contractions (e.g., by using software libraries like \href{https://optimized-einsum.readthedocs.io/en/stable/}{opt\_einsum}),
we retrieve that computing $Z$ requires time $\calO(|\calE|\cdot d^2 + |\calR|\cdot d^3)$ and additional space $\calO(|\calR|\cdot d^2)$.
Notice that the time complexity here is slightly lower than the theoretical upper bound given in \cref{thm:marginalisation-squared-circuits}, which would be $\calO(|\calE|\cdot d^2 + |\calR|\cdot d^4)$.

\shortparagraph{\SquaredTuckER.}
Lastly, we present the derivation of the partition function for \SquaredTuckER.
The score function $\phi_\SquaredTuckER$ encoded by \SquaredTuckER can be written as
\begin{align*}
    \phi_\SquaredTuckER(s,r,o) &= \left( \calT \times_1 \ve_s \times_2 \vw_r \times_3 \ve_o \right)^2 \\
    &= \!\!\!\!\sum_{\substack{(i,j,k)\in\\ [d_e]\times [d_r]\times [d_e]}} \sum_{\substack{(l,m,n)\in\\ [d_e]\times [d_r]\times [d_e]}}\!\!\!\! \tau_{ijk}\tau_{lmn} e_{si}e_{sl} w_{rj}w_{rm} e_{ok}e_{on}
\end{align*}
where $\ve_s,\ve_o\in\bbR^{d_e}$ are rows of the matrix $\vE\in\bbR^{|\calE|\times d_e}$, $\vw_r$ is a row of the matrix $\vW\in\bbR^{|\calR|\times d_r}$, and $\calT\in\bbR^{d_e\times d_r\times d_e}$ denotes the core tensor.
The partition function of $\phi_\SquaredTuckER$ can be written as
\begin{align*}
    Z &= \sum_{s\in\calE} \sum_{r\in\calR} \sum_{o\in\calE} \phi_\SquaredTuckER(s,r,o) \\
    &= \!\!\!\!\sum_{\substack{(i,j,k)\in\\ [d_e]\times [d_r]\times [d_e]}} \sum_{\substack{(l,m,n)\in\\ [d_e]\times [d_r]\times [d_e]}}\!\!\!\! \tau_{ijk}\tau_{lmn} \left( \sum_{s\in\calE} e_{si}e_{sl} \right) \left( \sum_{r\in\calR} w_{rj}w_{rm} \right) \left( \sum_{o\in\calE} e_{ok}e_{on} \right) \\
    &= \calT_{ijk} \calT_{lmn} \vE'_{il} \vW'_{jm} \vE'_{kn} \addtocounter{equation}{1}\tag{\theequation} \label{eq:pf-sqtucker-einsum}
\end{align*}
where $\vE' = \vE^\top \vE \in \bbR^{d_e\times d_e}$ and $\vW' = \vW^\top \vW \in \bbR^{d_r\times d_r}$.
Similarly to \SquaredRESCAL, %
by optimising the order of tensor contractions in the einsum operation showed in \cref{eq:pf-sqtucker-einsum},
we retrieve that computing $Z$ requires time $\calO(|\calE|\cdot d_e^2 + |\calR|\cdot d_r^2 + d_e^2\cdot d_r)$ and additional space $\calO(d_e^2\cdot d_r)$.
Similarly to \SquaredRESCAL, the time complexity is lower than the theoretical upper bound given in \cref{thm:marginalisation-squared-circuits}, which would be $\calO(|\calE|\cdot d_e^2 + |\calR|\cdot d_r^2 + d_e^4\cdot d_r^2)$.

\subsubsection{Complexity of Computing the PLL Objective}
\label{app:complexity-pll-objective}

In this section, we show that \OurModels enable to better scale the computation of the PLL objective (\cref{eq:pll-objective}) with respect to energy-based KGE models (see \cref{sec:background}).
We present this concept for $\CP$ and \OurModels derived from it (\cref{sec:casting-kge-models-to-pcs}), as for the other score functions it is similar.

\shortparagraph{Complexity of the PLL objective on \CP.}
Let $\phi_\CP(s,r,o) = \langle \ve_s,\vw_r,\ve_o \rangle = \sum_{i=1}^d e_{si}w_{ri}e_{oi}$ be the score function of \CP \citep{lacroix2018cp-kbc},
where $\ve_s,\ve_o\in\bbR^d$ (resp. $\vw_r\in\bbR^d$) are rows of the matrices $\vU,\vV\in\bbR^{|\calE|\times d}$ (resp. $\vW\in\bbR^{|\calR|\times d}$), which associate to each entity (resp. predicate) a vector.
Given a training triple $(s,r,o)$, the computation of the term $\log p(o\mid s,r) = \phi(s,r,o) - \log \sum_{o'\in\calE} \exp \phi(s,r,o')$ requires evaluating $\phi_\CP(s,r,o')$ for all objects $o'\in\calE$.
In order to fully exploit GPU parallelism \citep{jain2020knowledge},
this is usually done with the matrix-vector multiplication $\vV(\ve_s\odot \vw_r) \in \bbR^{|\calE|}$, where $\odot$ denotes the Hadamard product \citep{lacroix2018cp-kbc,chen2021relational-auxiliary-objective}.
Therefore, computing $\log p(o\mid s,r)$ for each triple $(s,r,o)$ in a mini-batch $B\subset\calE\times\calR\times\calE$ such that $|\calE|\gg|\calB|$ requires time $\calO(|\calE|\cdot |B|\cdot d)$ and space $\calO(|\calE|\cdot |B|)$.
For the other terms of the PLL objective (i.e., $\log p(s\mid r,o)$ and $\log p(r\mid s,o)$) the derivation is similar.
Moreover, for real-world large KGs it is reasonable to assume that $|\calE|\gg|\calR|$ and therefore the cost of computing $\log p(r\mid s,o)$ is negligible.

\shortparagraph{Complexity of the PLL objective on \OurModels.}
\OurModels obtained from \CP either by non-negative restriction (\cref{sec:non-negative-restriction}) or by squaring (\cref{sec:non-monotonic-squaring}) encode $\phi_\mathsf{pc}(s,r,o) \propto p(s,r,o)$ for any input triple.
As such, the component $\log p(o\mid s,r)$ of the PLL objective can be written as
\begin{equation}
    \label{eq:pll-for-pcs}
    \log p(o\mid s,r) = \log \phi_\mathsf{pc}(s,r,o) - \log \sum_{o'\in\calE} \phi_\mathsf{pc}(s,r,o').
\end{equation}
The absence of the exponential function in the summed terms in \cref{eq:pll-for-pcs} allows us to push the outer summation inside the circuit computing $\phi_\mathsf{pc}(s,r,o)$, and to sum over the input units relative to objects.
For instance, for \NNegCP we can write
\begin{align*}
    \sum_{o'\in\calE} \phi_\NNegCP(s,r,o') &= \sum_{o'\in\calE} \sum_{i=1}^d e_{si} w_{ri} e_{oi} = \sum_{i=1}^d e_{si} w_{ri}  \left( \sum_{o\in\calE} e_{o'i} \right) = (\ve_s \odot \vw_r)^\top (\vV^\top \mathbf{1}_\calE)
\end{align*}
where $\ve_s,\vw_r,\ve_o \in \bbR^d_+$, $\vV\in\bbR^{|\calE|\times d}_+$ denotes the matrix whose rows are object embeddings, and $\mathbf{1}_\calE = [1\ldots 1]^{|\calE|}$ is a vector of ones.
Note that $\vV^\top \mathbf{1}_\calE \in\bbR^d_+$ does not depend on the input triple.
Therefore, given a mini-batch of triples $B$, computing $\log p(o\mid s,r)$ requires time $\calO((|\calE|+|B|) \cdot d)$ and space $\calO(|B|\cdot d)$, which is much lower than the complexity on \CP showed above, and we can still leverage GPU parallelism.
For \SquaredCP, the complexity is similar to the derivation of the partition function complexity showed in \cref{app:partition-function-squared-circuits}.
That is, for \SquaredCP we can write
\begin{align*}
    \sum_{o'\in\calE} \phi_\SquaredCP(s,r,o') &= \sum_{o'\in\calE} \left( \sum_{i=1}^d e_{si} w_{ri} e_{oi} \right)^2 = \sum_{i=1}^d \sum_{j=1}^d e_{si}e_{sj} w_{ri}w_{rj} \left( \sum_{o'\in\calE} e_{o'i}e_{o'j} \right) \\
    &= (\ve_s \odot \vw_r)^\top (\vV^\top\vV) (\ve_s \odot \vw_r)
\end{align*}
where $\ve_s,\vw_r,\ve_o \in \bbR^d$, $\vV\in\bbR^{|\calE|\times d}$.
Note that the matrix $\vV^\top\vV\in\bbR^{d\times d}$ does not depend on the input triple.
Therefore, given a mini-batch of triples $B$, computing $\log p(o\mid s,r)$ requires time $\calO((|\calE|+|B|) \cdot d^2)$ and space $\calO(|B|\cdot d)$.
While the time complexity is quadratic in the embedding size $d$, it is still much lower than the time complexity on \CP.
A similar discussion can also be carried out for the other KGE models and the corresponding \OurModels, which retrieves the complexities showed in \cref{tab:summary-complexity}.

\subsubsection{Training Speed-up Benchmark Details}
\label{app:training-benchmark-details}

In this section we report the details about the training benchmark on \ComplEx, \NNegComplEx and \SquaredComplEx, whose results are showed in \cref{fig:pll-scaling}.
We measure time and peak GPU memory usage required for computing the PLL objective (\cref{eq:pll-objective}) \emph{and} to do an optimisation step for a single batch on \textsf{ogbl-wikikg2} \citep{hu2020ogb}, a large knowledge graph with millions of entities (see \cref{tab:datasets}).
We fix the embedding size to $d=100$ for the three models.
For the benchmark with increasing batch size, we keep all the entities and increase the batch size from $100$ to $5000$.
For the benchmark with increasing number of entities, we keep the batch size fixed to $500$ (the maximum allowed for \ComplEx by our GPUs) and progressively increase the number of entities, from about $3\cdot 10^5$ to $2.5\cdot 10^6$.
We report the average time over 25 independent runs on a single Nvidia RTX A6000 with 48 GiB of memory.

\section{Distribution of Scores}
\label{app:distribution-scores}

\begin{figure}[H]
    \centering
    \includegraphics{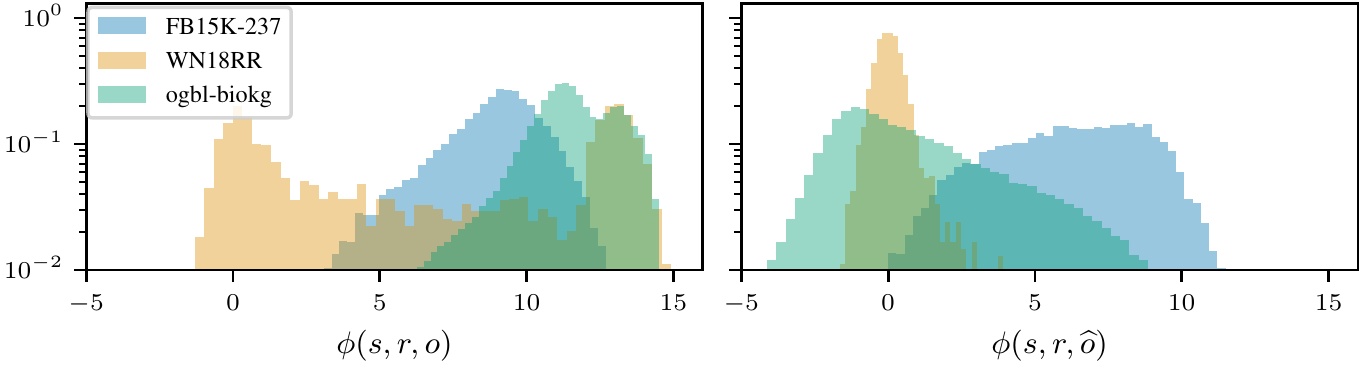}
    \caption{\textbf{Scores are mostly non-negative.} Histograms of the scores assigned by \ComplEx to existing validation triples (left) and their perturbation (right) on three data sets. The vast majority of triple scores are non-negative, suggesting that squaring them has minimal effect on the rankings.}
    \label{fig:hist-scores}
\end{figure}

In \cref{fig:hist-scores} we show the histograms of the scores assigned to the validation triples and their perturbations of three data sets (see \cref{app:datasets-statistics}).
Following \citet{socher2013reasoning}, we generate triple perturbations that are challenging for link prediction.
That is, for each validation triple $(s,r,o)$, the corresponding perturbation $(s,r,\widehat{o})$ is obtained by replacing the object with a random entity that has appeared at least once as an object in a training triple with predicate $r$.
The bottom line is that scores are mostly non-negative, and hence can be used as a heuristic to effectively initialise \OurModels or quickly distil them (e.g., on \textsf{FB15K-237}),
as we further discuss in \cref{app:additional-link-prediction}.

\section{Reconciling Knowledge Graph Embeddings Interpretations}
\label{app:interpreting-density-triples}

\paragraph{Triples as boolean variables.}
KGE models such as \CP, \RESCAL, \TuckER and \ComplEx have been historically introduced as factorizations of a tensor-representation of a KG, which we discuss next.
In fact, a KG $\calG$ can be represented as a 3-way binary tensor $\vY\in\{0,1\}^{|\calE|\times|\calR|\times|\calE|}$ in which every entry  $Y_{sro}$ is 1 if $(s, r, o)\in\calG$ and 0 otherwise \citep{nickel2016kgreview}.
Under this light, a KGE model like \RESCAL factorises every slice $\vY_r\in\{0,1\}^{|\calE|\times|\calE|}$, corresponding to the predicate $r$ as $\vE\vW_r\vE^{\top}$ where 
$\vE$ is an $|\calE|\times d$ matrix comprising the entity embeddings and $\vW_r$ is an $d\times d$ matrix containing the embeddings for the predicate $r$.
To deal with uncertainty and incomplete KGs, $Y_{sro}$ can be interpreted as a Bernoulli random variable.
As such, its distribution becomes $p(Y_{sro}=1\mid s,r,o)$ which is usually modelled as $\sigma(\phi(s,r,o))$, where $\sigma$ denotes the logistic function and $\phi$ is the score function of a KGE model.
Note that this distribution of triple introduces $|\calE\times\calR\times\calE|$ random variables, one for each possible triple.

\shortparagraph{KGE models as estimators of a distribution over KGs.}
At the same time,
the interpretation of triples as boolean variables
\textit{induces a distribution over possible KGs}, $q(\calG)$, which is the distribution over all possible binary tensors $p(\vY)$.
The probability of a KG $\calG$ can therefore be computed as the product of the likelihoods of all variables $Y_{sro}$, i.e., $q(\calG) = \prod_{(s,r,o)\in\calG} p(Y_{sro}=1\mid s,r,o) \cdot \prod_{(s,r,o)\not\in\calG} p(Y_{sro} = 0\mid s,r,o)$.
Note that (re-)normalising this distribution is intractable in general, as it would require summing over all possible $2^{|\calE\times\calR\times\calE|}$ binary tensors.
This is why historically KGE models have been interpreted as energy-based models, by directly optimising for $\phi(s,r,o)$, interpreted as the negated energy associated to every triple, and not $p(Y_{sro}=1\mid s,r,o)$ (see \cref{sec:background}).
This has been done via negative sampling or other contrastive learning objectives~\citep{bordes2011se,bordes2013transe}.
We point out that this very same interpretation can be found in the literature of probabilistic logic programming~\citep{fierens2015inference}, probabilistic databases (PDBs)~\citep{dalvi2007efficient} and statistical relational learning (see \cref{sec:related-works})
where the distribution over possible ``worlds'' is over sets of boolean assignments to ground atoms or facts, or tuples in a PDB, each interpreted as Bernoulli random variables.

\shortparagraph{Estimating a distribution over triples.}
In this work, instead, we interpret existing KGE models and our \OurModels as models that encode a possibly unnormalised probability distribution over three random variables, $S,R,O$, which induces a distribution over triples that is tractable to renormalise.\footnote{The polynomial cost of renormalising an energy-based KGE is unfortunately infeasible for real-world KGs, see \cref{sec:background}.}
To reconcile these two perspectives, we interpret the probability of a triple $p(s,r,o)$ to be proportional to the probability of all KGs $\calG$ where $(s,r,o)$ holds, i.e., those $\calG$ such that $(s,r,o)\in\calG$.
Intuitively, a triple will be more probable to exist if it does exist in highly probable KGs.
More formally, given $q$ a probability distribution over KGs, we define $p$ as an unnormalised probability distribution over triples, i.e., $\mu(s,r,o)\propto p(s,r,o)$, where
\begin{equation}
    \label{eq:density-triples}
    \mu(s,r,o) = \!\!\! \sum_{\substack{\calG\in\calH\\ (s,r,o)\in\calG}} \!\!\! q(\calG) = \sum_{\calG\in\calH} q(\calG) \cdot \Ind{(s,r,o)\in\calG} = \bbE_{\calG\sim q} [\Ind{(s,r,o)\in\calG}]
\end{equation}
and $\calH=2^{\calE\times\calR\times\calE}$ denotes the set of all possible KGs.
Computing the expectation in \cref{eq:density-triples} exactly is equivalent to  solving a \emph{weighted model counting} (WMC) problem~\citep{chavira2008wmc}, where we sum the probabilities of all possible KGs containing $(s,r,o)$.
Alternatively, it is equivalent to computing the probability of the simplest possible query in a PDB (i.e., asking for a single tuple), where each stored tuple is interpreted as an independent Bernoulli random variable $Y_{sro}$.
Therefore, we have that $\mu(s,r,o)$ is simply the likelihood that $Y_{sro}$ is true, i.e., $p(Y_{sro}=1\mid s,r,o)$.
Furthermore, the normalisation constant of $\mu$ (\cref{eq:density-triples}) can be written as
\begin{equation*}
    Z = \!\! \sum_{\substack{(s,r,o)\in\\ \calE\times\calR\times\calE}} \mu(s,r,o) = \sum_{\calG\in\calH} q(\calG) \cdot \!\!\!\! \sum_{\substack{(s,r,o)\in\\ \calE\times\calR\times\calE}} \!\! \Ind{(s,r,o)\in\calG} = \sum_{\calG\in\calH} q(\calG) \cdot |\calG| = \bbE_{\calG\sim q} [|\calG|]
\end{equation*}
which is the expected size of a KG according to the probability distribution $q$.
Written in this way, however, computing $Z$ through $q(\calG)$ is intractable.
For this reason, we directly encode $\mu$ with \OurModels and compute $Z$ by summing over all triples, and therefore without modelling $q(\calG)$.

\shortparagraph{Further interpretations in related works.}
Under the interpretation of a KG as a PDB, 
\citet{friedman2020symbolic-pdb} further decompose the likelihood that $Y_{sro}$ is true as
\begin{equation*}
    p(Y_{sro}=1\mid s,r,o) = p(E_s=1\mid s)\cdot p(T_r=1\mid r)\cdot p(E_o=1\mid o)
\end{equation*}
where $E_s,E_o,T_r$ are new Bernoulli variables that are assumed to be conditionally independent given the parameters of the PDB.
That is, instead of introducing one random variable per triple, they introduce one random variable per entity and predicate.
In this framework, they reinterpret the score function of \DistMult, a simplified variant of \CP, as an implicit circuit that models an unnormalized distribution over the collection of variables $\vZ=\{E_u\}_{u\in\calE}\cup\{T_r\}_{r\in\calR}$, trained by negative sampling.
This decomposition permits to compute the probability of any database query efficiently, which otherwise is known to be either a PTIME or a \#P-hard problem, depending on the query type \citep{dalvi2012dichotomy-ucq}.
If we were to interpret our distribution $\mu(S=s, R=r, O=o)$ as the unnormalized marginal distribution $p(E_s=1, T_r=1, E_o=1)=\sum_{\vz^{\prime}}p(E_s=1, T_r=1, E_o=1, \vZ^{\prime}=\vz^{\prime})$, where $\vZ^{\prime}=\vZ\setminus\{E_s, T_r, E_o\}$, we could equivalently compute any probabilistic query efficiently. 
Note that under this interpretation, training our \OurModels by MLE over $S, R, O$ would be equivalent to maximise a composite marginal log-likelihood \citep{varin2011composite} over $\vZ$.

\section{Empirical Evaluation}
\label{app:empirical-evaluation}

\subsection{Datasets Statistics}
\label{app:datasets-statistics}

\cref{tab:datasets} shows statistics of commonly-used datasets to benchmark KGE models for link prediction.
We employ standard benchmark datasets \citep{toutanova2015observed,dettmers2018conv2d-kge,hu2020ogb} whose number of entities (resp. predicates) ranges from $\approx$14k to $\approx$2.5M (resp. from 11 to $\approx$500).

\begin{table}[H]
    \caption{\textbf{Dataset statistics.} Statistics of multi-relational knowledge graphs: number of entities ($|\calE|$), number of predicates ($|\calR|$), number of training/validation/test triples.\\}
    \label{tab:datasets}
    \centering
    \tablesize
    \begin{tabular}{l@{\ }crrrrr}
        \toprule
        \textbf{Dataset} & & $|\calE|$ & $|\calR|$ & \textbf{\# Train} & \textbf{\# Valid} & \textbf{\# Test} \\
        \midrule
        \textsf{\tablesize FB15k-237}    & \citep{toutanova2015observed}  &           14,541 & 237 & 272,115              & 17,535              & 20,466 \\
        \textsf{\tablesize WN18RR}       & \citep{dettmers2018conv2d-kge} &           40,943 &  11 &  86,835              &  3,034              &  3,134 \\
        \textsf{\tablesize ogbl-biokg}   & \citep{hu2020ogb}              &           93,773 &  51 &   4,763 $\cdot 10^3$ &    163 $\cdot 10^3$ &    163 $\cdot 10^3$ \\
        \textsf{\tablesize ogbl-wikikg2} & \cite{hu2020ogb}               & 2.5 $\cdot 10^6$ & 535 &  16,109 $\cdot 10^3$ &    429 $\cdot 10^3$ &    598 $\cdot 10^3$ \\
        \bottomrule
    \end{tabular}
\end{table}

\subsection{Metrics}
\label{app:metrics}

\paragraph{Mean reciprocal rank and hits at $k$.}
Given a test triple $(s,r,o)$, we rank the possible object $o'$ (resp. subject $s'$) completions to link prediction queries $(s,r,?)$ (resp. $(?,r,o)$) based on their scores in descending order.
The position of the test triple $(s,r,o)$ in the ranking of object (resp. subject) completed queries $(s,r,o')$ (resp. $(s',r,o)$) is then used to compute the \emph{mean reciprocal rank} (MRR)
\begin{equation*}
    \operatorname{MRR} = \frac{1}{2|\calG_\text{test}|} \sum_{(s,r,o)\in\calG_\text{test}} \left(\frac{1}{\operatorname{rank}(o\mid s,r)} + \frac{1}{\operatorname{rank}(s\mid r,o)} \right)
\end{equation*}
where $\calG_\text{test}$ denotes the set of test triples, and $\operatorname{rank}(o\mid s,r),\operatorname{rank}(s\mid r,o)$ denote respectively the positions of the true completion $(s,r,o)$ in the rankings of object and subject completed queries.
The fraction of \emph{hits at $k$} (Hits@$k$) for $k > 0$ is computed as
\begin{equation*}
    \operatorname{Hits@}k = \frac{1}{2|\calG_\text{test}|} \sum_{(s,r,o)\in\calG_\text{test}} \left( \Ind{\operatorname{rank}(o\mid s,r) \leq k} + \Ind{\operatorname{rank}(s\mid r,o) \leq k} \right).
\end{equation*}
Consistently with existing works on link prediction \citep{ruffinelli2020teach-dog-tricks,chen2021relational-auxiliary-objective}, the MRRs and Hits@$k$ metrics are computed under the \emph{filtered} setting, i.e., we rank true completed triples against potential ones that do not appear in the union of training, validation and test splits.

\shortparagraph{Semantic consistency score.}
Let $K$ be a logical constraint encoding some background knowledge over variables $S$, $R$ and $O$.
Given a test triple $(s,r,o)$, we first rank the possible completions to link prediction queries in the same way as for computing the MRR.
Then, the \emph{semantic consistency score} (Sem@$k$) \citep{hubert2022new-strategies} for some integer $k>0$ is computed as
\begin{equation*}
    \mathrm{Sem@}k = \frac{1}{2k|\calG_\text{test}|} \sum_{(s,r,o)\in\calG_\text{test}} \left( \sum_{o'\in\calA^k_O(s,r,o)} \!\! \Ind{(s,r,o')\models K} + \!\!\! \sum_{s'\in\calA^k_S(s,r,o)} \!\! \Ind{(s',r,o)\models K} \right)
\end{equation*}
where $\calG_\text{test}$ denotes the set of test triples, $\calA^k_O(s,r,o)$ (resp. $\calA^k_S(s,r,o)$) denotes the list of the top-$k$ candidate object (resp. subject) completions to the link prediction query $(s,r,?)$ (resp. $(?,r,o)$), and $(s,r,o)\models K$ if and only if $(s,r,o)$ satisfies $K$. %

\subsection{Empirical KTD Score}
\label{app:empirical-ktd-score}

Let $\calF=\{x_i\}_{i=1}^m$, $\calG=\{y_j\}_{j=1}^n$ two sets of triples that are drawn i.i.d. from two distributions $\bbP,\bbQ$ over triples.
We compute the empirical KTD score with an \emph{unbiased estimator} \citep{gretton2012mmd} $\operatorname{KTD}_u(\calF,\calG)$ as
\begin{equation*}
    \frac{1}{m(m-1)} \sum_{i\neq j}^m k(\psi(x_i),\psi(x_j)) + \frac{1}{n(n-1)} \sum_{i\neq j}^n k(\psi(y_i),\psi(y_j)) - \frac{2}{mn} \sum_{i=1}^m \sum_{j=1}^n k(\psi(x_i),\psi(y_j)).
\end{equation*}
For each data set, we compute the empirical KTD score between $n=25,000$ triples sampled from \OurModels and $m$ test triples.
In case of $m>n$, we sample $n$ triples randomly, uniformly and without replacement from the set of test triples.
The time complexity of computing the KTD score is $\calO(nmh)$, where $h$ denotes the size of triple latent representations ($h=4000$ in our case, see \cref{sec:empirical-evaluation-sampling}).
For efficiency reasons, we therefore follow \citet{binkowski2018demystifying-mmd} and randomly extract two batches of 1000 triples each from both the generated and the test triples sets and compute the empirical KTD score on them.
We repeat this process 100 times and report the average and standard deviation in \cref{tab:ktd-samples-extended}.

\subsection{Experimental Setting}
\label{app:experimental-setting}

\paragraph{Hyperparameters.}
All models are trained by gradient descent with either the PLL or the MLE objective (\cref{eq:pll-objective,eq:mle-objective}).
We set the weights $\omega_s,\omega_r,\omega_o$ of the PLL objective all to one, as to retrieve a classical pseudo-log-likelihood \citep{varin2011composite}.
Note that \citet{chen2021relational-auxiliary-objective} set $\omega_s,\omega_o$ to one and treat $\omega_r$ as an additional hyperparameter instead that is opportunely tuned.
The models are trained until the MRR computed on the validation set does not improve after three consecutive epochs.
We fix the embedding size $d=1000$ for both \CP and \ComplEx and use Adam \citep{kingma2014adam} as optimiser with $10^{-3}$ as learning rate.
An exception is made for \OurModels obtained via non-negative restriction (\cref{sec:non-negative-restriction}), for which a learning rate of $10^{-2}$ is needed, as we observed very slow convergence rates.
We search for the batch size in $\{5\cdot 10^2, 10^3, 2\cdot 10^3, 5\cdot 10^3\}$ based on the validation MRR.
Finally, we perform 5 repetitions with different seeds and report the average MRR and two standard deviations in \cref{tab:best-results-kbc-extended}.

\shortparagraph{Parameters initialisation.}
Following \citep{lacroix2018cp-kbc,chen2021relational-auxiliary-objective}, the parameters of \CP and \ComplEx are initialised by sampling from a normal distribution $\calN(0,\sigma^2)$ with $\sigma=10^{-3}$.
Since the embedding values in \NNegCP and \NNegComplEx can be interpreted as parameters of categorical distributions over entities and predicates (see \cref{app:interpreting-non-negative}), we initialise them by sampling from a Dirichlet distribution with concentration factors set to $10^3$.
To allow unconstrained optimisation for \NNegCP and \NNegComplEx, we represent the embedding values by their logarithm and perform computations directly in log-space, i.e., summations and log-sum-exp operations instead of multiplications and summations, respectively.
Moreover, the parameters that ensure the non-negativity of \NNegComplEx (see \cref{app:realising-monocomplex}) are initialised by sampling from a normal distribution $\calN(0,\sigma^2)$ with $\sigma=10^{-2}$.
We initialise the parameters of \SquaredCP and \SquaredComplEx such that the logarithm of the scores are approximately normally distributed and centred in zero during the initial optimisation steps, since this applies for the scores given by \CP and \ComplEx.
Such initialisation therefore permits a fairer comparison.
To do so, we initialise the embedding values by sampling from a log-normal distribution $\calL\calN(\mu,\sigma^2)$, where $\mu=-\log(d)/3 - \sigma^2/2$ for \CP and $\mu=-\log(2d)/3 - \sigma^2/2$ for \ComplEx, both with $\sigma=10^{-3}$.
The mentioned values for $\mu$ can be derived via Fenton-Wilkinson approximation~\citep{fenton1960lognormal}.
Even though the parameters of \OurModels obtained via non-monotonic squaring are initialised to be non-negative, they are free of becoming negative during training (as we also confirm in practice).

\shortparagraph{Hardware.}
Experiments on the smaller knowledge graphs \textsf{FB15K-237} and \textsf{WN18RR} were run on a single Nvidia GTX 1060 with 6 GiB of memory, while those on the larger \textsf{ogbl-biokg} and \textsf{ogbl-wikikg2} were run on a single Nvidia RTX A6000 with 48 GiB of memory.

\subsection{Additional Experimental Results}
\label{app:additional-results}

\subsubsection{Link Prediction Results}
\label{app:additional-link-prediction}

In this section, we present the additional results regarding the link prediction experiments showed in \cref{sec:empirical-evaluation-kbc} and analyse different metrics and learning settings.

\shortparagraph{Statistical tests and hits at $k$.}
\cref{tab:best-results-kbc-extended} shows the best test average MRRs (see \cref{app:metrics}) with two standard deviation and average training time across 5 independent runs with different seeds.
We highlight the best results in bold according to a one-sided Mann–Whitney U test with a confidence level of 99\%.
The showed results in terms of MRRs are also confirmed in \cref{tab:best-results-kbc-hits}, which shows the best average Hits@$k$ (see \cref{app:metrics}) with $k\in\{1,3,10\}$.

\shortparagraph{Average log-likelihood.}
For the best \OurModels for link prediction showed in \cref{tab:best-results-kbc-extended},
we report the average log-likelihood of test triples and two standard deviations (across 5 independent runs) in \cref{tab:average-ll}.
We again highlight the best results in bold, according to a one-sided Mann–Whitney U test.

\shortparagraph{Quickly distilling parameters.}
As discussed in \cref{sec:non-monotonic-squaring}, since learned KGE models mostly assign non-negative scores to triples (see \cref{app:distribution-scores}) we can initialise the parameters of \OurModels obtained by squaring with the parameters of already-learned KGE models, without losing much in terms of link prediction performances.
Here, we test this hypothesis and fine-tune \OurModels initialised in this way by using either the PLL or MLE objectives (\cref{eq:pll-objective,eq:mle-objective}).
To do so, we first collect the parameters of the best \CP and \ComplEx found for link prediction (see \cref{sec:empirical-evaluation-kbc}).
Then, we initialise \OurModels derived by squaring with these parameters and fine-tune them until the MRR computed on validation triples does not improve after three consecutive epochs.
We employ Adam \citep{kingma2014adam} as optimiser, and we search for the batch size in $\{5\cdot 10^2, 10^3, 2\cdot 10^3, 5\cdot 10^3\}$ and learning rate in $\{10^{-3}, 10^{-4}\}$, as fine-tuning may require a lower learning rate than the one used in previous experiments (see \cref{app:experimental-setting}).
\cref{tab:parameters-distillation} shows the MRRs achieved by \CP, \ComplEx and the corresponding \OurModels obtained via squaring that are initialised by distilling the parameters from the already-trained \CP and \ComplEx.
On \textsf{FB15K-237} and \textsf{WN18RR}, distilling parameters induces a substantial improvement in terms of MRR with respect to \SquaredCP and \SquaredComplEx whose parameters have been initialised randomly (see \cref{app:experimental-setting}).
Furthermore, for \SquaredComplEx and on \textsf{WN18RR} and \textsf{ogbl-biokg} we achieved similar MRRs with respect to \ComplEx \emph{without} the need of fine-tuning.

\begin{table}[H]
    \tablesize
    \caption{\textbf{\OurModels are competitive with their energy-based counterparts.} Best test MRRs (and two standard deviations) of \CP, \ComplEx and \OurModels trained with the PLL and MLE objectives (\cref{eq:pll-objective,eq:mle-objective}). In parentheses we show the average training time (in minutes).\\
    }
    \label{tab:best-results-kbc-extended}
    \centering
    \setlength{\tabcolsep}{1pt}
    \begin{tabular}{l
            c>{\scriptsize}rc>{\scriptsize}r
            c>{\scriptsize}rc>{\scriptsize}r
            c>{\scriptsize}rc>{\scriptsize}r}
        \toprule
        \multirow{2}{*}{\textbf{Model}}
        & \multicolumn{4}{c}{\textsf{\tablesize FB15k-237}}
        & \multicolumn{4}{c}{\textsf{\tablesize WN18RR}}
        & \multicolumn{4}{c}{\textsf{\tablesize ogbl-biokg}} \\
        \cmidrule(lr){2-5}
        \cmidrule(lr){6-9}
        \cmidrule(lr){10-13}
        & \multicolumn{2}{c}{PLL} & \multicolumn{2}{c}{MLE}
        & \multicolumn{2}{c}{PLL} & \multicolumn{2}{c}{MLE}
        & \multicolumn{2}{c}{PLL} & \multicolumn{2}{c}{MLE} \\
        \midrule
        \CP                 
            & 0.310 \textpm 0.001 &   (8) & \multicolumn{2}{c}{---}
            & \bfseries 0.105 \textpm 0.007 &  (11) & \multicolumn{2}{c}{---}
            & 0.831 \textpm 0.001 & (136) & \multicolumn{2}{c}{---} \\
        \NNegCP
            & 0.237 \textpm 0.003 &   (1) & 0.230 \textpm 0.003 &  (1) %
            & 0.027 \textpm 0.002 &   (1) & 0.026 \textpm 0.001 &  (1) %
            & 0.496 \textpm 0.013 & (172) & 0.501 \textpm 0.010 & (142) \\
        \SquaredCP          
            & \bfseries 0.315 \textpm 0.003 &   (8) & 0.282 \textpm 0.004 &  (7)
            & \bfseries 0.104 \textpm 0.001 &  (23) & 0.091 \textpm 0.004 & (23)
            & \bfseries 0.848 \textpm 0.001 &  (66) & 0.829 \textpm 0.001 &  (61) \\
        \midrule
        \ComplEx            
            & \bfseries 0.342 \textpm 0.005 &  (36) & \multicolumn{2}{c}{---}
            & \bfseries 0.471 \textpm 0.002 &  (16) & \multicolumn{2}{c}{---}
            & 0.829 \textpm 0.001 & (180) & \multicolumn{2}{c}{---} \\
        \NNegComplEx        
            & 0.214 \textpm 0.003 &  (10) & 0.205 \textpm 0.006 &  (5)   %
            & 0.030 \textpm 0.001 &   (6) & 0.029 \textpm 0.001 &  (3) %
            & 0.503 \textpm 0.014 & (245) & 0.516 \textpm 0.009 & (212) \\
        \SquaredComplEx     
            & 0.334 \textpm 0.001 &  (10) & 0.300 \textpm 0.003 & (16)
            & 0.420 \textpm 0.011 &  (37) & 0.391 \textpm 0.004 & (19)
            & \bfseries 0.858 \textpm 0.001 &  (71) & 0.840 \textpm 0.001 & (59) \\
        \bottomrule
    \end{tabular}
    \vspace{-10pt}
\end{table}

\begin{table}[H]
    \tablesize
    \caption{\textbf{Hits@$k$ results.} Average test Hits@$k$ for $k\in\{1,3,10\}$ of \CP, \ComplEx and \SquaredCP and \SquaredComplEx trained with the PLL or MLE objectives.\\}
    \label{tab:best-results-kbc-hits}
    \centering
    \begin{tabular}{l@{\ \ \ }
            r@{\ \ }r@{\ \ }r@{\quad}r@{\ \ }r@{\ \ }r@{\quad}
            r@{\ \ }r@{\ \ }r@{\quad}r@{\ \ }r@{\ \ }r@{\quad}
            r@{\ \ }r@{\ \ }r@{\quad}r@{\ \ }r@{\ \ }r@{\quad}}
        \toprule
        & \multicolumn{6}{c}{\textsf{\tablesize FB15k-237}}
        & \multicolumn{6}{c}{\textsf{\tablesize WN18RR}}
        & \multicolumn{6}{c}{\textsf{\tablesize ogbl-biokg}} \\
        \cmidrule(r){2-7}
        \cmidrule(r){8-13}
        \cmidrule(r){14-19}
        & \multicolumn{3}{c}{PLL} & \multicolumn{3}{c}{MLE}
        & \multicolumn{3}{c}{PLL} & \multicolumn{3}{c}{MLE}
        & \multicolumn{3}{c}{PLL} & \multicolumn{3}{c}{MLE} \\
        \cmidrule(r){2-4}\cmidrule(r){5-7}
        \cmidrule(r){8-10}\cmidrule(r){11-13}
        \cmidrule(r){14-16}\cmidrule(r){17-19} 
        \multicolumn{1}{r}{\textbf{Model} $k=$} & 1 & 3 & 10 & 1 & 3 & 10 & 1 & 3 & 10 & 1 & 3 & 10 & 1 & 3 & 10 & 1 & 3 & 10 \\
        \cmidrule(r){2-4}\cmidrule(r){5-7}
        \cmidrule(r){8-10}\cmidrule(r){11-13}
        \cmidrule(r){14-16}\cmidrule(r){17-19}
        \CP                 
            & 22.4 & 34.1 & 48.2 & --- & --- & ---
            &  7.5 & 12.1 & 16.8 & --- & --- & ---
            & 76.4 & 88.1 & 95.0 & --- & --- & --- \\
        \NNegCP             
            & 17.0 & 25.8 & 36.7 & 16.7 & 24.9 & 35.4
            &  1.7 &  2.7 &  4.5 &  1.6 &  2.5 &  4.4
            & 38.0 & 54.4 & 73.4 & 38.4 & 55.0 & 74.4 \\
        \SquaredCP          
            & 23.1 & 34.8 & 48.2 & 20.5 & 30.8 & 43.5
            &  6.7 & 12.1 & 17.6 &  5.9 & 10.7 & 15.3
            & 78.6 & 89.5 & 95.7 & 76.1 & 88.1 & 95.0 \\
        \midrule
        \ComplEx            
            & 25.2 & 37.5 & 52.5 & --- & --- & ---
            & 43.3 & 48.6 & 54.6 & --- & --- & ---
            & 76.1 & 87.9 & 95.0 & --- & --- & --- \\
        \NNegComplEx        
            & 15.7 & 23.1 & 31.7 & 15.0 & 22.1 & 30.4
            &  1.5 &  2.7 &  4.5 &  1.6 &  2.5 &  4.4
            & 38.8 & 55.1 & 74.1 & 40.0 & 56.7 & 75.9 \\
        \SquaredComplEx     
            & 24.5 & 36.9 & 51.1 & 21.6 & 33.0 & 46.7
            & 36.0 & 45.6 & 52.4 & 34.5 & 42.3 & 46.9
            & 80.0 & 90.1 & 95.8 & 77.5 & 88.8 & 95.4 \\
        \bottomrule
    \end{tabular}
\end{table}

\begin{table}[H]
    \caption{\textbf{Better distribution estimation with \OurModels obtained via squaring.} Average log-likelihood of test triples achieved by baselines and \OurModels trained with the PLL or MLE objectives.\\}
    \label{tab:average-ll}
    \tablesize
    \centering
    \setlength{\tabcolsep}{4pt}
    \begin{tabular}{l
            cccccc
            cccccc}
        \toprule
        \textbf{Model}
        & \multicolumn{2}{c}{\textsf{\tablesize FB15k-237}}
        & \multicolumn{2}{c}{\textsf{\tablesize WN18RR}}
        & \multicolumn{2}{c}{\textsf{\tablesize ogbl-biokg}} \\
        \midrule
        Uniform
            & \multicolumn{2}{c}{-24.638}
            & \multicolumn{2}{c}{-23.638}
            & \multicolumn{2}{c}{-26.829}   \\
        NNMFAug
            & \multicolumn{2}{c}{-19.270}
            & \multicolumn{2}{c}{-22.938}
            & \multicolumn{2}{c}{-17.562}   \\
        \midrule
        & PLL & MLE
        & PLL & MLE
        & PLL & MLE \\
        \cmidrule(lr){2-3} \cmidrule(lr){4-5} \cmidrule(lr){6-7}
        \NNegCP        
            & -16.773 \textpm 0.040 & -16.592 \textpm 0.059
            & \bfseries -21.987 \textpm 0.006 & -22.103 \textpm 0.010
            & -17.900 \textpm 0.048 & -17.416 \textpm 0.049 \\
        \SquaredCP     
            & -17.105 \textpm 0.031 & \bfseries -15.982 \textpm 0.028
            & -24.911 \textpm 0.241 & -26.352 \textpm 0.077 
            & -17.231 \textpm 0.059 & \bfseries -16.533 \textpm 0.013 \\
        \midrule
        \NNegComplEx   
            & -17.507 \textpm 0.035 & -17.592 \textpm 0.039
            & -21.233 \textpm 0.058 & -21.432 \textpm 0.008
            & -18.716 \textpm 0.088 & -17.749 \textpm 0.019 \\
        \SquaredComplEx
            & -17.100 \textpm 0.026 & \bfseries -15.744 \textpm 0.041
            & \bfseries -19.522 \textpm 0.530 & \bfseries -19.739 \textpm 0.214
            & -17.340 \textpm 0.022 & \bfseries -16.518 \textpm 0.003 \\
        \bottomrule
    \end{tabular}
\end{table}

\begin{table}[H]
    \caption{\textbf{Distilling parameters can improve performances.} Test MRRs achieved by \CP, \ComplEx and \OurModels obtained by squaring (\cref{sec:non-monotonic-squaring}). For \SquaredCP and \SquaredComplEx we report the best MRRs achieved by distilling the parameters from the already-learned \CP and \ComplEx (denoted with $\star$), and with \dag\ we denote those results for which further fine-tuning with the PLL or MLE objectives did not bring better results.
    We underline results for which distilling parameters increased the MRR.\\}
    \label{tab:parameters-distillation}
    \centering
    \tablesize
    \begin{tabular}{lllllll}
        \toprule
        \multirow{2}{*}{\textbf{Model}}
        & \multicolumn{2}{c}{\textsf{\tablesize FB15k-237}}
        & \multicolumn{2}{c}{\textsf{\tablesize WN18RR}}
        & \multicolumn{2}{c}{\textsf{\tablesize ogbl-biokg}} \\
        \cmidrule(lr){2-3}
        \cmidrule(lr){4-5}
        \cmidrule(lr){6-7}
        & PLL & MLE
        & PLL & MLE
        & PLL & MLE \\
        \midrule
        \CP                 
            & 0.311 & ---
            & 0.108 & ---
            & 0.831 & --- \\
        \ComplEx            
            & 0.344 & ---
            & 0.470 & ---
            & 0.829 & --- \\
        \midrule
        \SquaredCP          
            & 0.317 & 0.285
            & 0.103 & 0.089
            & 0.849 & 0.830 \\
        \SquaredCP $\star$     
            & \underline{0.327} & \underline{0.315}
            & 0.102 & \underline{0.115}
            & 0.851 & 0.828 \dag \\
        \midrule
        \SquaredComplEx     
            & 0.333 & 0.301
            & 0.416 & 0.390
            & 0.859 & 0.839 \\
        \SquaredComplEx $\star$
            & \underline{0.342}      & \underline{0.340} 
            & \underline{0.462} \dag & \underline{0.463}
            & 0.859 & 0.828 \dag \\
        \bottomrule
    \end{tabular}
\end{table}

\subsubsection{Quality of Sampled Triples Results}
\label{app:sampling-triples-results}

In this section, we provide additional results regarding the evaluation of the quality of triples sampled by \OurModels (see \cref{sec:empirical-evaluation-sampling}).
For these experiments, we search for the same hyperparameters as for the experiments on link prediction (see \cref{app:experimental-setting}), and train \OurModels until the average log-likelihood computed on validation triples does not improve after three consecutive epochs.

\cref{tab:ktd-samples-extended} shows the mean empirical KTD score and one standard deviation (see \cref{app:empirical-ktd-score}).
In addition, we visualise triple embeddings of sampled and test triples in \cref{fig:tsne-triple-embeddings} by leveraging t-SNE~\citep{vandermaaten2008tsne} as a method for visualising high-dimensional data.
In particular, we apply the t-SNE method implemented in \href{https://scikit-learn.org/stable/modules/generated/sklearn.manifold.TSNE.html}{scikit-learn} with perplexity $50$ and number of iterations $5\cdot 10^3$, while other parameters are fixed to their default value.
As showed in \cref{fig:tsne-triple-embeddings-ogbl-biokg}, an empirical KTD score close to zero translates to an high clusters similarity between embeddings of sampled and test triples.

\begin{table}[H]
    \caption{\textbf{\OurModels trained by MLE generate new likely triples.} Empirical KTD scores between test triples and triples generated by baselines and \OurModels trained with the PLL objective or by MLE (\cref{eq:pll-objective,eq:mle-objective}). Lower is better.\\}
    \label{tab:ktd-samples-extended}
    \tablesize
    \centering
    \begin{tabular}{l
            c@{\:}>{\scriptsize}rc@{\:}>{\scriptsize}rc@{\:}>{\scriptsize}r
            c@{\:}>{\scriptsize}rc@{\:}>{\scriptsize}rc@{\:}>{\scriptsize}r}
        \toprule
        \textbf{Model}
        & \multicolumn{4}{c}{\textsf{\tablesize FB15k-237}}
        & \multicolumn{4}{c}{\textsf{\tablesize WN18RR}}
        & \multicolumn{4}{c}{\textsf{\tablesize ogbl-biokg}} \\
                \cmidrule(lr){2-5} \cmidrule(lr){6-9} \cmidrule(lr){10-13}
        Training set
            & \multicolumn{4}{c}{0.055\:{\scriptsize \textpm 0.007}}
            & \multicolumn{4}{c}{0.260\:{\scriptsize \textpm 0.013}}
            & \multicolumn{4}{c}{0.029\:{\scriptsize \textpm 0.010}} \\
        Uniform        
            & \multicolumn{4}{c}{0.589\:{\scriptsize \textpm 0.012}}
            & \multicolumn{4}{c}{0.766\:{\scriptsize \textpm 0.036}}
            & \multicolumn{4}{c}{1.822\:{\scriptsize \textpm 0.044}} \\
        NNMFAug        
            & \multicolumn{4}{c}{0.414\:{\scriptsize \textpm 0.014}}
            & \multicolumn{4}{c}{0.607\:{\scriptsize \textpm 0.028}}
            & \multicolumn{4}{c}{0.518\:{\scriptsize \textpm 0.035}} \\
        \midrule
        & \multicolumn{2}{c}{PLL} & \multicolumn{2}{c}{MLE}
        & \multicolumn{2}{c}{PLL} & \multicolumn{2}{c}{MLE}
        & \multicolumn{2}{c}{PLL} & \multicolumn{2}{c}{MLE} \\
        \cmidrule(lr){2-5} \cmidrule(lr){6-9} \cmidrule(lr){10-13}
        \NNegCP        
            & 0.404 & \textpm 0.016 & 0.433 & \textpm 0.015
            & 0.633 & \textpm 0.033 & \bfseries 0.578 & \bfseries \textpm 0.029
            & 0.966 & \textpm 0.040 & 0.738 & \textpm 0.030 \\
        \SquaredCP     
            & 0.253 & \textpm 0.014 & \bfseries 0.070 & \bfseries \textpm 0.007
            & 0.768 & \textpm 0.036 & 0.768 & \textpm 0.036
            & 0.039 & \textpm 0.009 & \bfseries 0.017 & \bfseries \textpm 0.013 \\
        \cmidrule(lr){1-13}
        \NNegComplEx
            & 0.336 & \textpm 0.016 & 0.323 & \textpm 0.015
            & 0.456 & \textpm 0.018 & 0.478 & \textpm 0.019
            & 0.175 & \textpm 0.019 & 0.097 & \textpm 0.013 \\
        \SquaredComplEx
            & 0.326 & \textpm 0.016 & \bfseries 0.102 & \bfseries \textpm 0.010
            & 0.338 & \textpm 0.020 & \bfseries 0.278 & \bfseries \textpm 0.017
            & 0.104 & \textpm 0.010 & \bfseries 0.034 & \bfseries \textpm 0.007 \\
        \bottomrule
    \end{tabular}
\end{table}

\begin{figure}[H]
    \centering
    \begin{subfigure}[t]{0.32\linewidth}
        \includegraphics[scale=0.49]{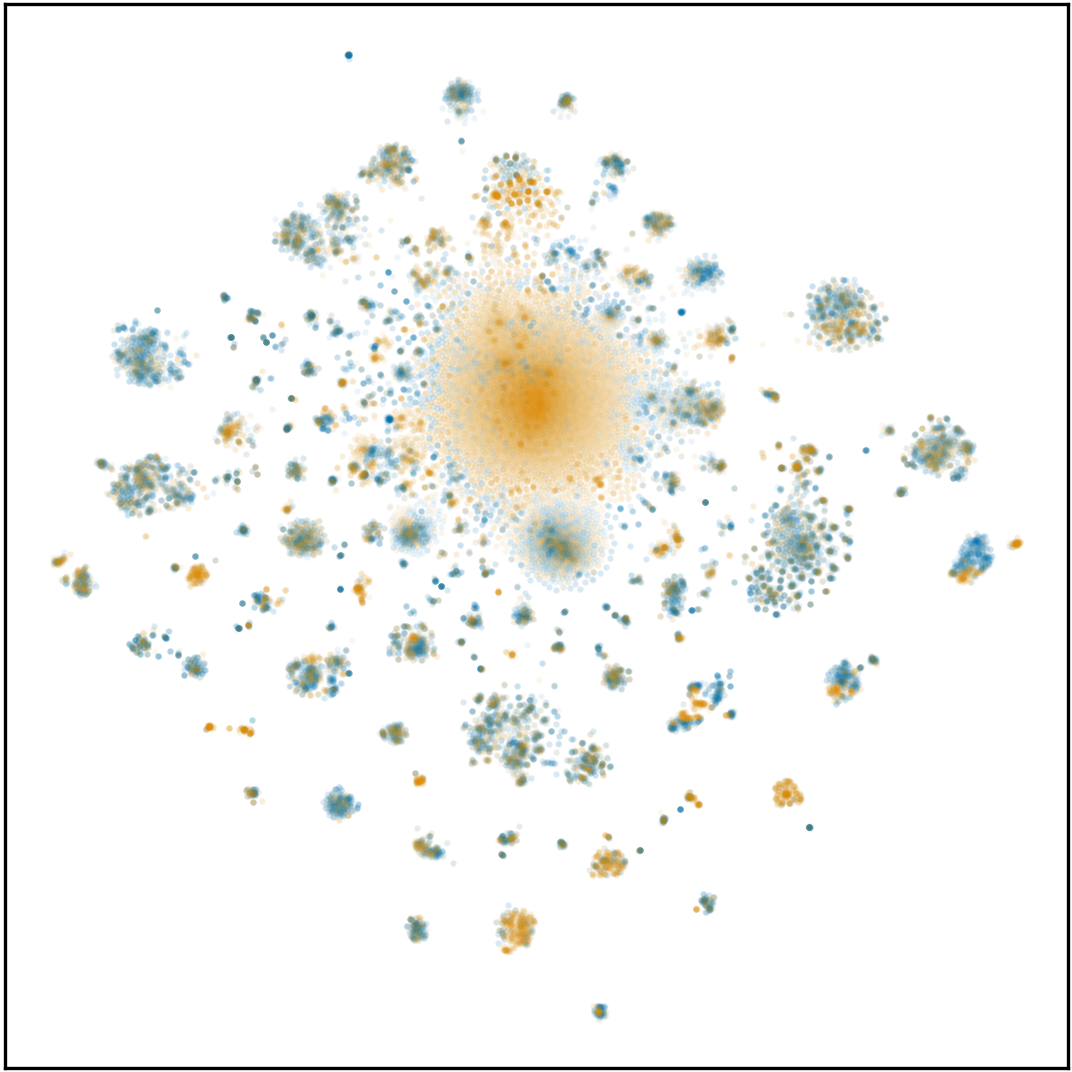}
        \caption{\textsf{FB15k-237}\\ $\operatorname{KTD}:=0.102 \pm 0.010$}
        \label{fig:tsne-triple-embeddings-fb15k-237}
    \end{subfigure}
    \begin{subfigure}[t]{0.32\linewidth}
        \includegraphics[scale=0.49]{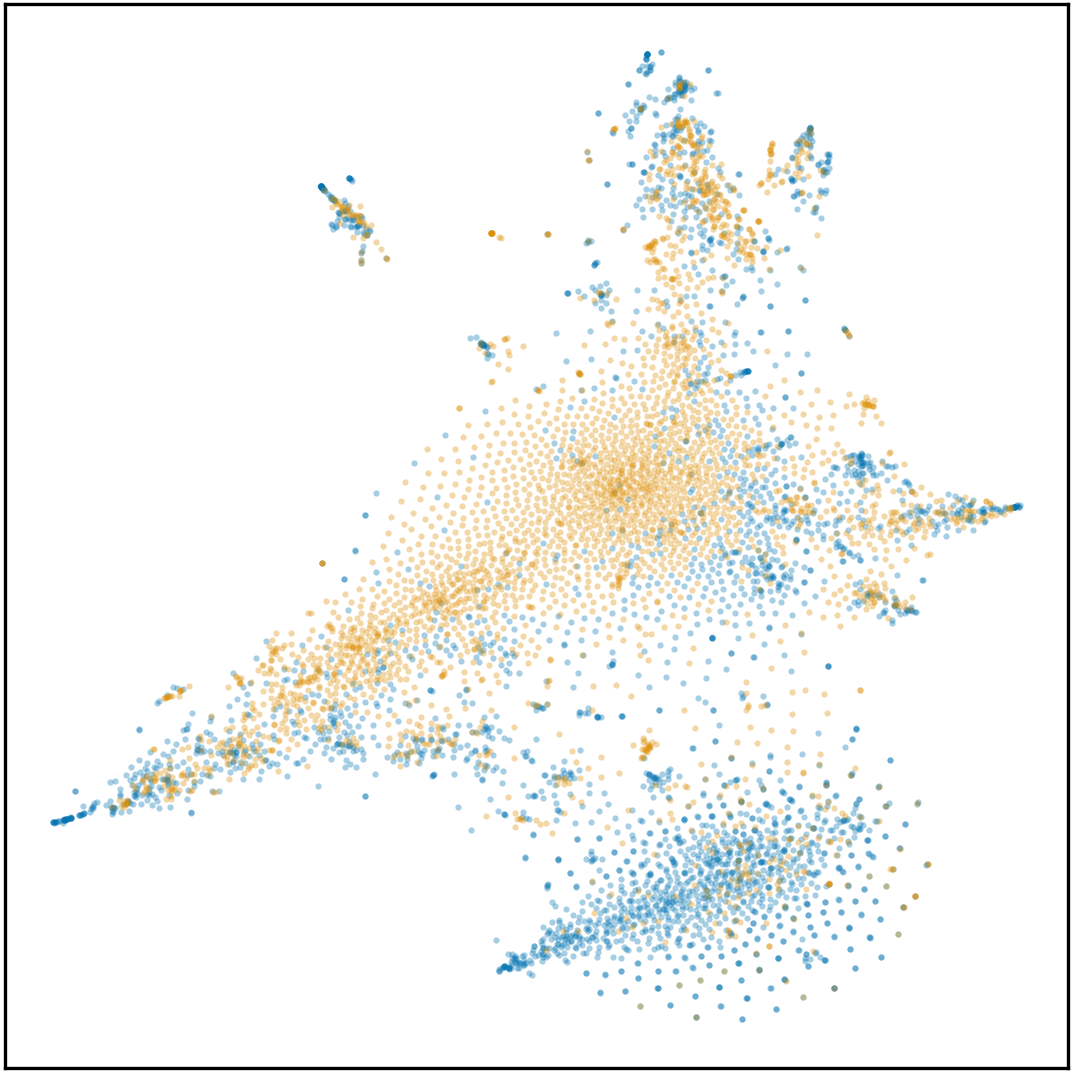}
        \caption{\textsf{WN18RR}\\ $\operatorname{KTD}:=0.278 \pm 0.017$}
        \label{fig:tsne-triple-embeddings-wn18rr}
    \end{subfigure}
    \begin{subfigure}[t]{0.32\linewidth}
        \includegraphics[scale=0.49]{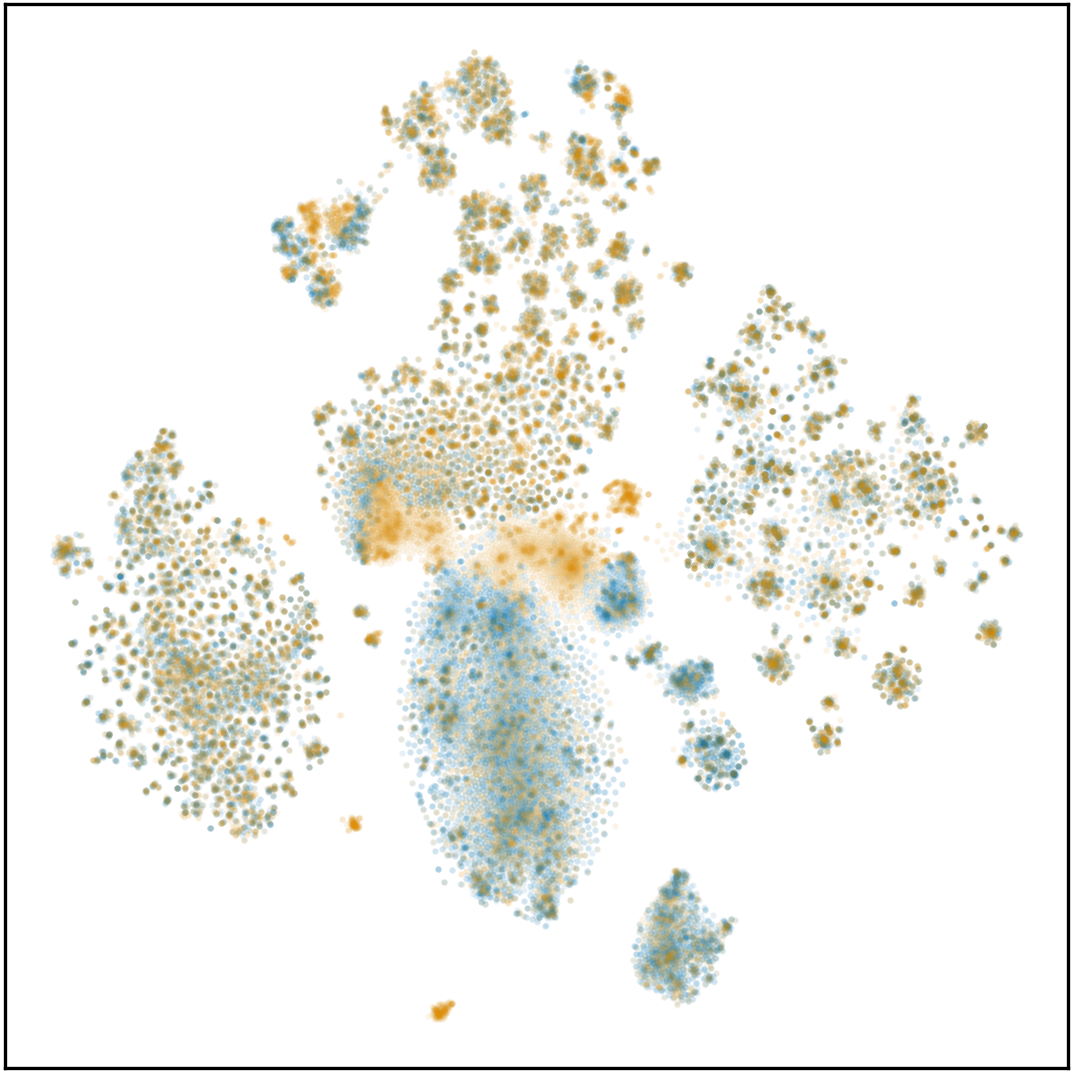}
        \caption{\textsf{ogbl-biokg}\\ $\operatorname{KTD}:=0.034 \pm 0.007$}
        \label{fig:tsne-triple-embeddings-ogbl-biokg}
    \end{subfigure}
    \caption{\textbf{Sampled triples are \textit{close} to test triples.} t-SNE \citep{vandermaaten2008tsne} visualisations of the embeddings of test triples (in blue) and triples sampled by \SquaredComplEx (in orange).
    The distribution shift between training and test triples on \textsf{WN18RR} mentioned in \cref{sec:empirical-evaluation-sampling} is further confirmed in \cref{fig:tsne-triple-embeddings-wn18rr}, as it shows a region of test triples (at the bottom and in blue) that is not covered by many generated triples.
    }
    \label{fig:tsne-triple-embeddings}
\end{figure}

\subsubsection{Calibration Diagrams}
\label{app:calibration-diagrams}

Existing works on studying the calibration of KGE models are based on interpreting each possible triple $(s,r,o)$ as an independent Bernoulli random variable $Y_{sro}$ whose likelihood is determined by the score function $\phi$, i.e., $\Pr(Y_{sro} = 1 \mid s,r,o) = \sigma(\phi(s,r,o))$~\citep{tabacof2019probability,pezeshkpour2020revisiting-evaluation-kbc,zhu2023closer-look-calibration}, where $\sigma$ denotes the logistic function.
While \OurModels encode a probability distribution over all possible triples, this does not impede us to reinterpret them to model the likelihood of each $Y_{sro}$ by still considering scores in log-space as negated energies (see \cref{sec:background}).
Therefore, to evaluate the calibration of \OurModels encoding a non-negative score function $\phi_\mathsf{pc}$ (see \cref{sec:casting-kge-models-to-pcs}) we compute the probability of a triple $(s,r,o)$ being true as $p(Y_{sro}=1\mid s,r,o) := \sigma(\log \phi_\mathsf{pc}(s,r,o))$.
However, the usage of the logistic function might give misleading results in case of scores not being centred around zero on average.
Therefore, we also report calibration diagrams (see paragraph below) where the $p(Y_{sro}=1\mid s,r,o)$ is obtained via min-max normalisation of the scores given by KGE models (the logarithm of the scores given for \OurModels), where the minimum and maximum are computed on the training triples.
Note that several ex-post (re-)calibration techniques are available \citep{tabacof2019probability,zhu2023closer-look-calibration}, but they should benefit \OurModels as they do with existing KGE models.

\shortparagraph{Setting and metrics.}
To plot calibration diagrams, we follow \citet{socher2013reasoning} and sample challenging negative triples, i.e., for each test triple $(s,r,o)$ we sample an unobserved perturbed one $(s,r,\widehat{o})$ by replacing the object with an entity that has appeared at least once with the predicate $r$ in the training data.
We then compute the \emph{empirical calibration error} (ECE) \citep{zhu2023closer-look-calibration} as $\mathrm{ECE} := \frac{1}{b} \sum_{i=1}^b |p_j - f_j|$, where $b$ is the number of uniformly-chosen bins for the interval $[0,1]$ of triple probabilities, and $p_j,f_j$ are respectively the average probability and relative frequency of actually existing triples in the $j$-th bin.
The lower the ECE score, the better calibrated are the predictions, as they are closer to the empirical frequency of triples that do exist in each bin.
The calibration curves are plotted by considering the relative frequency of existing triples in each bin, and curves closer to the main diagonal indicate better calibrated predictions.

\shortparagraph{\OurModels are more calibrated out-of-the-box.}
\cref{fig:calibration-cp} (resp. \cref{fig:calibration-complex}) show calibration diagrams for \OurModels derived from \CP and \ComplEx trained with the MLE objective (\cref{eq:mle-objective}) (resp. PLL objective (\cref{eq:pll-objective})).
In 19 cases over 24, \OurModels obtained via squaring (\cref{sec:non-monotonic-squaring}) achieve lower ECE scores and better calibrated curves than \CP and \ComplEx.
While \OurModels obtained via non-negative restriction (\cref{sec:non-negative-restriction}) are not well calibrated when using the logistic function, on \textsf{ogbl-biokg} \citep{hu2020ogb} they are still better calibrated than \CP and \ComplEx when probabilities are obtained via min-max normalisation.
Furthermore, on \textsf{WN18RR} \OurModels achieved the highest ECE scores (corresponding to poorly-calibrated predictions), which could be explained by the distribution shift between training and test triples that was observed for this KG in \cref{sec:empirical-evaluation-sampling} and further confirmed in \cref{app:sampling-triples-results}.

\captionsetup[sub]{font=large}

\begin{figure}
    \begin{subfigure}{0.95\textwidth}
        \begin{minipage}{0.15\textwidth}
            \subcaption{}
            \label{fig:calibration-cp-pll}
        \end{minipage}
        \begin{minipage}{0.85\textwidth}
            \includegraphics[scale=0.75]{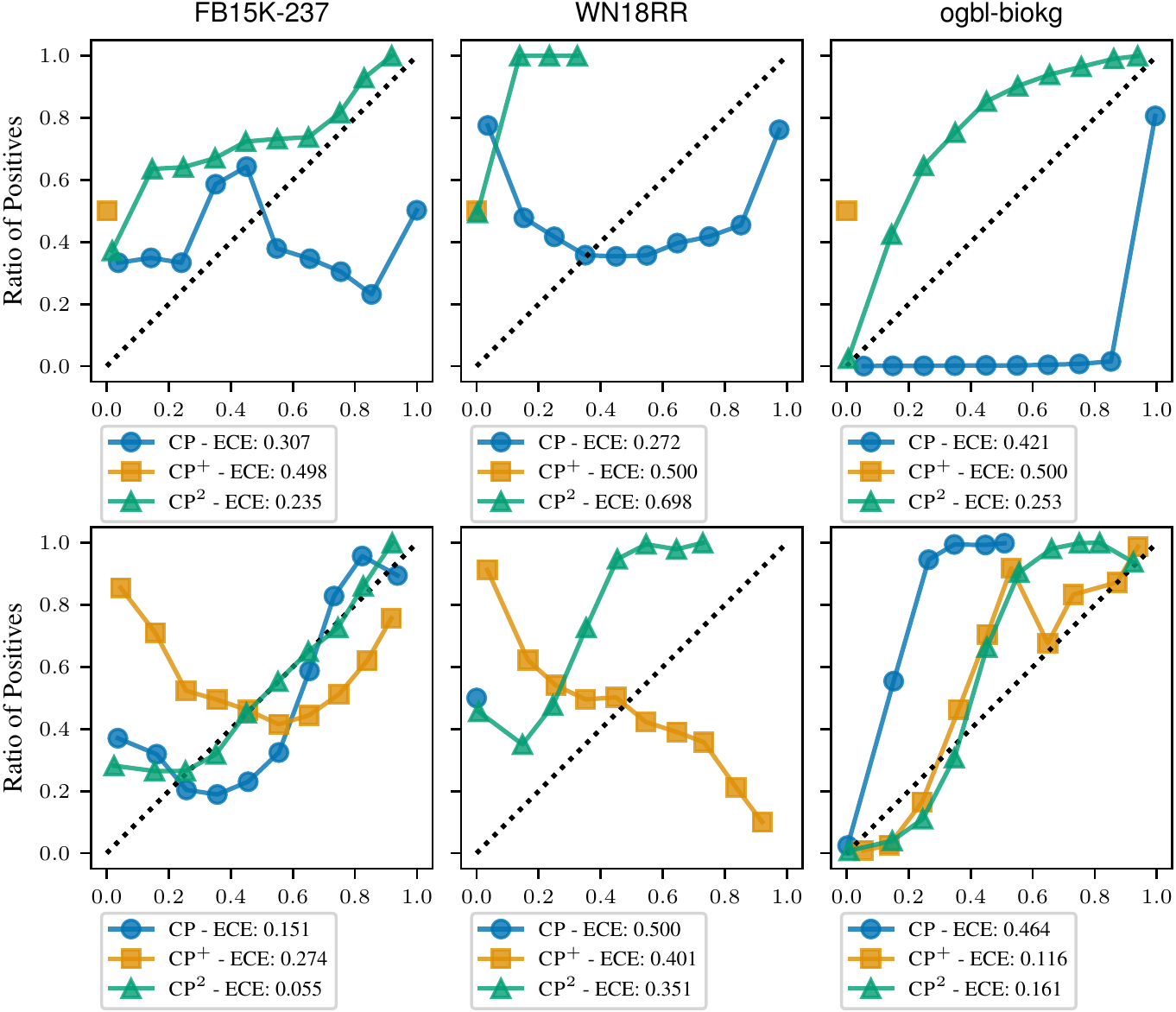}
        \end{minipage}
    \end{subfigure}
    \par
    \vspace{1cm}
    \begin{subfigure}{0.95\textwidth}
        \begin{minipage}{0.15\textwidth}
            \subcaption{}
            \label{fig:calibration-cp-mle}
        \end{minipage}
        \begin{minipage}{0.85\textwidth}
            \includegraphics[scale=0.75]{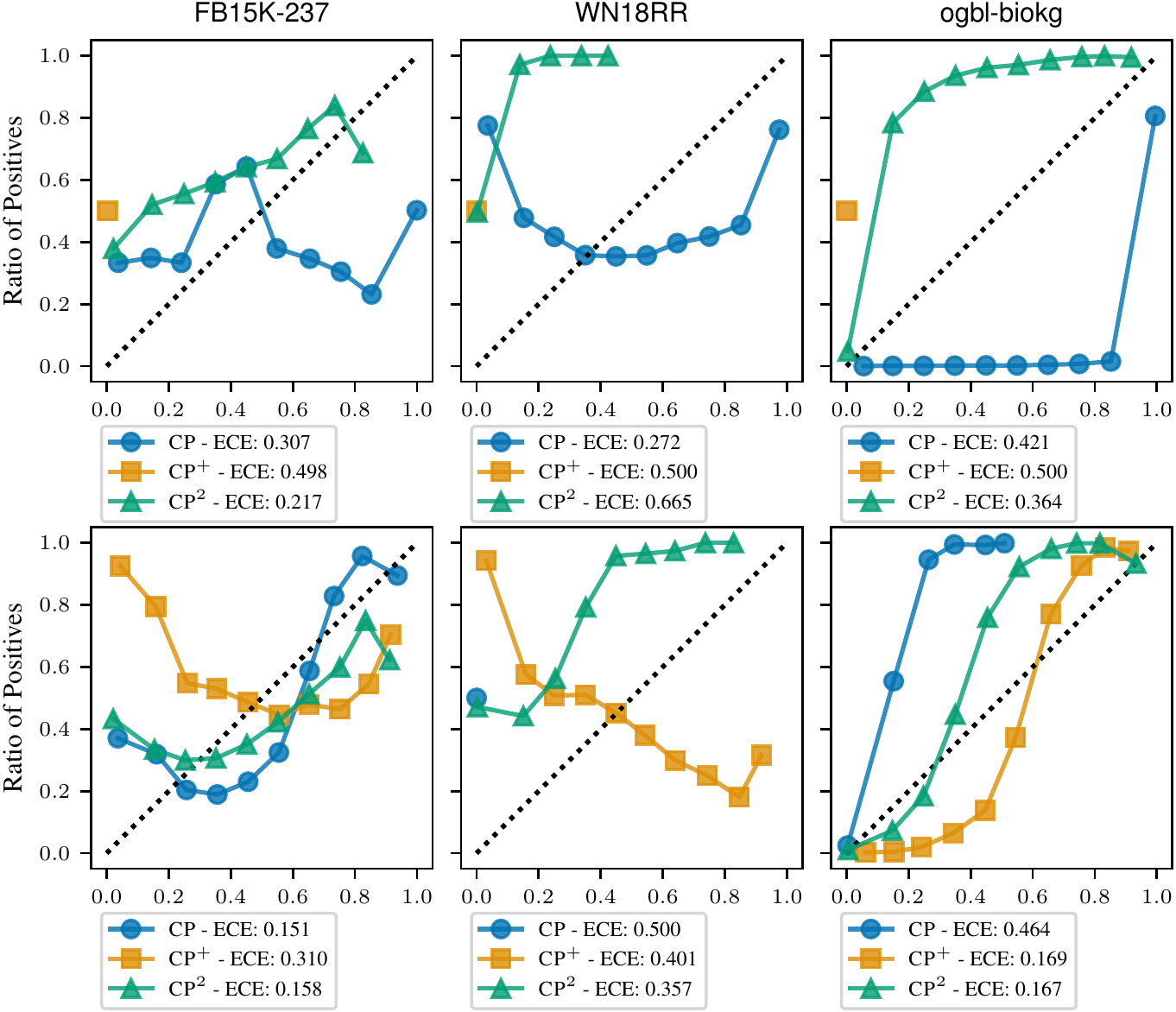}
        \end{minipage}
    \end{subfigure}
    \caption{\textbf{Better calibrated predictions with \SquaredCP.} Calibration diagrams of \CP, \NNegCP and \SquaredCP trained with either the PLL (\cref{fig:calibration-cp-pll}) or MLE (\cref{fig:calibration-cp-mle}) objectives. The probability of triples are obtained via the application of the logistic function (rows above) and min-max normalisation (rows below). See \cref{app:calibration-diagrams} for details.
    The calibration curves for \NNegCP where triple probabilities are obtained with the logistic function do not provide any meaningful information, as the logarithm of their scores are generally distributed over large negative values.
    }
    \label{fig:calibration-cp}
\end{figure}

\begin{figure}
    \begin{subfigure}{0.95\textwidth}
        \begin{minipage}{0.15\textwidth}
            \subcaption{}
            \label{fig:calibration-complex-pll}
        \end{minipage}
        \begin{minipage}{0.85\textwidth}
            \includegraphics[scale=0.75]{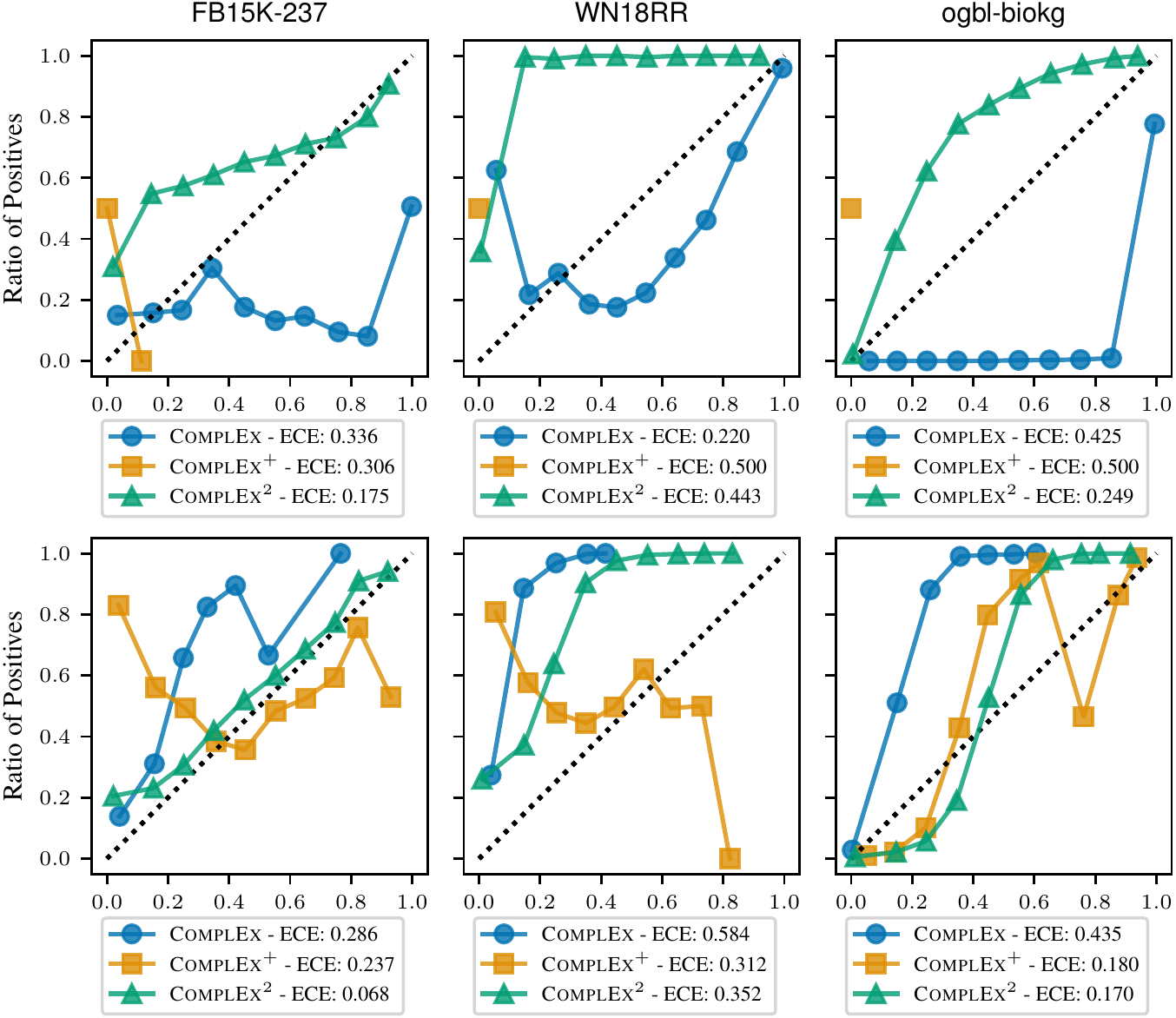}
        \end{minipage}
    \end{subfigure}
    \par
    \vspace{1cm}
    \begin{subfigure}{0.95\textwidth}
        \begin{minipage}{0.15\textwidth}
            \subcaption{}
            \label{fig:calibration-complex-mle}
        \end{minipage}
        \begin{minipage}{0.85\textwidth}
            \includegraphics[scale=0.75]{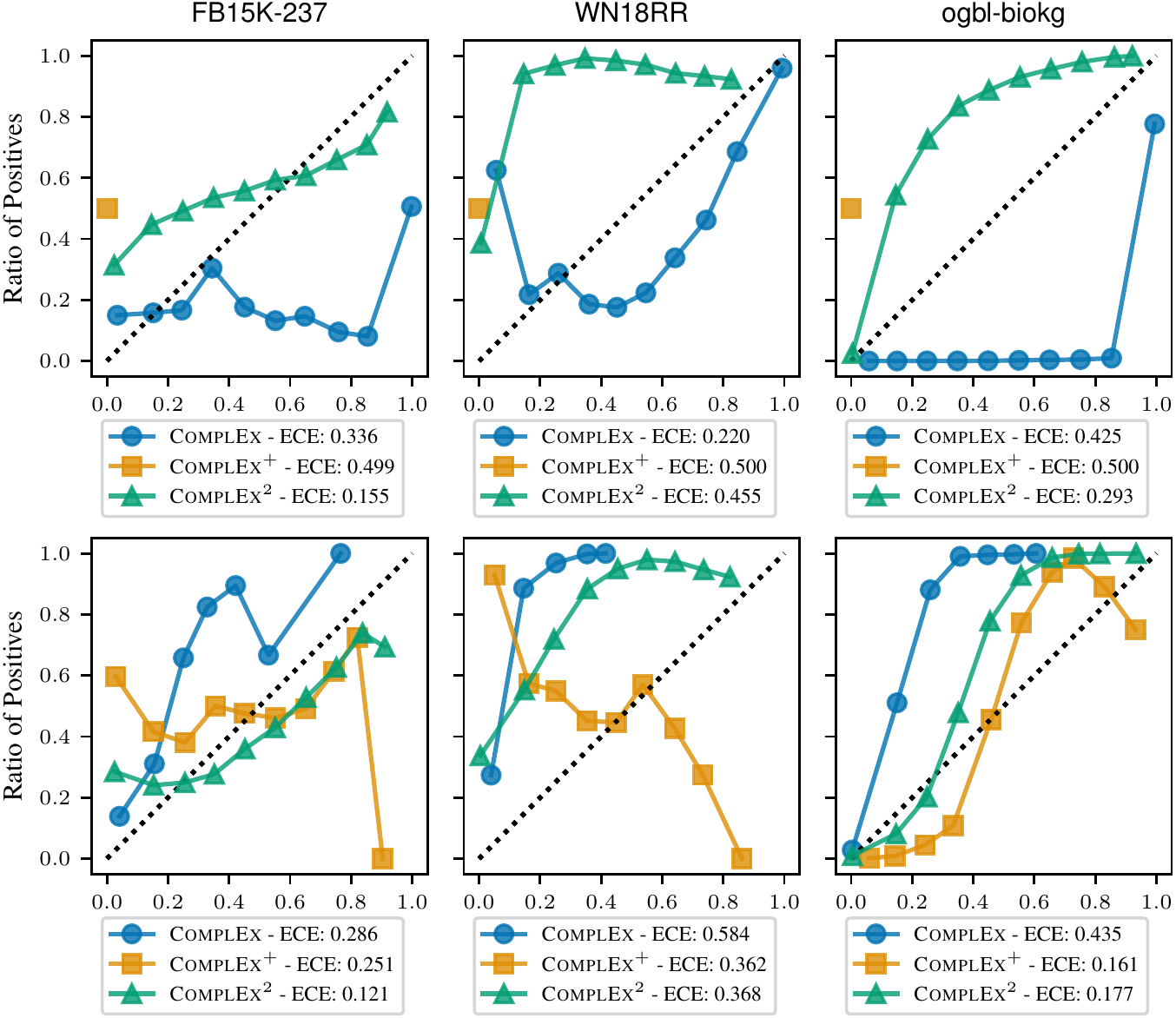}
        \end{minipage}
    \end{subfigure}
    \caption{\textbf{Better calibrated predictions with \SquaredComplEx.} Calibration diagrams of \ComplEx, \NNegComplEx and \SquaredComplEx trained with either the PLL (\cref{fig:calibration-complex-pll}) or MLE (\cref{fig:calibration-complex-mle}) objectives. The probability of triples are obtained via the application of the logistic function (rows above) and min-max normalisation (rows below). See \cref{app:calibration-diagrams} for details. The calibration curves for \NNegComplEx where triple probabilities are obtained with the logistic function do not provide any meaningful information, as the logarithm of their scores are generally distributed over large negative values.
    }
    \label{fig:calibration-complex}
\end{figure}

\end{document}